\documentclass[lettersize,journal]{IEEEtran}
\usepackage{amsmath,amsfonts}
\usepackage{amsthm}
\usepackage{algorithmic}
\usepackage{algorithm}
\usepackage{array}
\usepackage[caption=false,font=normalsize,labelfont=sf,textfont=sf]{subfig}
\usepackage{textcomp}
\usepackage{stfloats}
\usepackage{url}
\usepackage{verbatim}
\usepackage{graphicx}
\usepackage{cite}
\usepackage{color}
\usepackage{multirow}
\begin{document}

\title{How to ensure a safe control strategy? Towards a SRL for urban transit autonomous operation}

\author{Zicong Zhao,~\IEEEmembership{Graduate Student Member,~IEEE}
}



\maketitle

\begin{abstract}
Deep reinforcement learning (DRL) has gradually shown its latent decision-making ability in urban rail transit autonomous operation. However, since reinforcement learning can not neither guarantee safety during learning nor execution, this is still one of the major obstacles to the practical application of reinforcement learning. Given this drawback, reinforcement learning applied in the safety-critical autonomous operation domain remains challenging without generating a safe control command sequence that avoids overspeed operations. Therefore, a SSA-DRL framework is proposed in this paper for safe intelligent control of urban rail transit autonomous operation trains. The proposed framework is combined with linear temporal logic, reinforcement learning, Monte Carlo tree search and consists of four mainly module: a post-posed shield, a searching tree module, an additional actor and a DRL framework. Furthermore, the output of the framework can meet speed constraint, schedule constraint and optimize the operation process. Finally, the proposed SSA-DRL framework for decision-making in urban rail transit autonomous operation is evaluated in sixteen different sections, and its effectiveness is demonstrated through the basic simulation and additional experiment.

\end{abstract}


\section{Introduction}
\IEEEPARstart{A}{utonomous} train operation has become a hot research in the world to make urban rail transit smarter, more efficient and greener. Siemens Mobility in 2018 presented the world’s first autonomous tram at InnoTrans. Also, at InnoTrans in 2022, Thales’ rolling laboratory train Lucy can determine its precise location and speed by itself. In 2023, CRRC QINGDAO SIFANG CO., LTD. presented the Train Autonomous Circumambulate System (TACS) at Qingdao Line 6 and Alstom is now trying to bring GoA4 autonomous systems on regional train lines. Besides these companies, rail developed countries are also working hard on the research of autonomous train operation system. In 2022, Europe’s Rail Flagship Project R2DATO project is spending €180 million to develop the Next Generation Autonomous Train Control system. All in all, different companies or countries will have different technique features for autonomous train operation, but one important thing in common is that the train should have the ability of self-decision or self-driving no matter in common situation or an emergence situation. How to achieve the mentioned ability is now becoming a hot research.

In other transportation area such as Unmanned Aerial Vehicle (UAV) path planning or vehicle autonomous driving, the agent also needs to have the ability of self-planning or self-lane keeping \cite{fu2014monocular,horwick2010strategy}. In recent years, reinforcement learning (RL) or deep reinforcement learning (DRL) are widely used to train an agent with self-decision ability \cite{sutton2018reinforcement}. However, unlike those traditional virtual RL environment like OpenAI Gym or DeepMind MuJoCo \cite{1606.01540,todorov2012mujoco}, the above transportation scenarios are all real and safety-critical which means that a wrong or danger action may lead to an unacceptable result for example the UAV collides the infrastructures, the cars have a collision and the train runs over the limit speed. Unluckily, traditional RL algorithms may not guarantee the safety not only in training process but also in execution process which results the main obstacle that prevents RL from being widely used in real world \cite{thananjeyan2021recovery}.

Moreover, unlike the UAV path planning or vehicle autonomous driving, the intermediate process of autonomous train operation is also important. In UAV path planning, researchers always focus on how to find the shortest and safest path to reach the destination and in vehicle autonomous driving, researchers always focus on how to keep or change the lane to avoid collision. But for autonomous train operation, things are not the same. Because the trains must obey the schedule and there are always several optimization objectives such as the energy consumption and passenger comfort. This makes it harder to design a self-decision algorithm to trade-off between the final object (arrive at the destination safe and on time) and the optimization object (minimize the energy consumption and keep the passengers comfortable). And in recent years, with the development of artificial intelligence (AI), the AI community has realized that if the failure can cause or damage to an AI system, a fail-safe mechanism or fallback plan should be an important part of the design of the AI system \cite{kaur2022trustworthy,liu2022trustworthy,li2023trustworthy}. In some AI guidelines published in recent years, this mechanism has become an important requirement of AI \cite{smuha2019eu}.

In this paper, our main purpose is to propose a self-decision algorithm for autonomous train operation with three other abilities, 1) objectives are optimized, 2) safety is ensured and 3) decisions can be explained. Safe reinforcement learning (SRL), a subfield of RL is used in this paper to design the algorithm. The detailed contributions are introduced in SectionI.B.

\subsection{Related Work}
In urban transit optimal control area, most studies can be regarded as an extension of the study of speed profile optimization. In the past few decades, due to the characteristics of speed tracking control, researchers mainly focused on how to optimize the speed profile, so that the control algorithm, especially PID control, can achieve better control effect. Up to now, most studies are based on Pontryagin's Maximum Principle (PMP) to analysis the switching of working conditions such as maximum acceleration (MA), coasting (CO), maximum braking (MB) to optimize the speed profile\cite{howlett1990optimal,howlett2012energy,khmelnitsky2000optimal,albrecht2016k,albrecht2016key,GOVERDE2021353}. Such optimization methods are also called energy-efficient train control (EETC)\cite{2017Review}.

With the development of intelligent control theory and the requirement of self decision-making in autonomous operation, intelligent control methods represented by dynamic programming (DP) and RL are widely studied to improve automation level of traditional ATO system.

As the basis of RL and DP, Bellman optimal equation is used to build a multi-stage decision making problem for train operation control\cite{ko2004application}. Compared with the famous bionic optimization algorithm such as genetic algorithm and ant colony optimization, DP has a better performance under different operation time and inter-section distance\cite{lu2013single}. With the development of train motors and controllers, continuous control commands are more and more used in urban rail transit, which makes the traditional RL or DP method unsuitable. DRL and approximate dynamic programming (ADP) are then used to handle this problem. For the optimal control of heavy haul train with uncertain environment conditions, the maximum utility of regenerative braking energy and the optimization of speed profile with parametric uncertainty, ADP all has a good performance\cite{2017Optimal,liu2018optimal,2020Train}. As for the specific use of DRL, researches on one hand use it directly to output control command for train\cite{zhu2017communication,LIU2021103249,liu2021dqn}, and on the other hand combine it with other framework such as the expert system or history experience to correct the given control command\cite{yin2014intelligent,zhou2020smart,shang2021deep}.

In RL research area, with the development of computer science and graphics process unit (GPU), several famous algorithms represented by the Q-learning, actor-critic (AC), deep deterministic policy gradient 
(DDPG) and soft actor-critic (SAC) have achieved tremendous achievement in two-player game and computer game\cite{watkins1992q,konda1999actor,lillicrap2015continuous,haarnoja2018soft}. The two most famous researches are AlghaGo and AlphaZero in Chess and AlphaStar in StarCraft for they have almost beaten every human player without pressure\cite{silver2016mastering,silver2017mastering,vinyals2019grandmaster}.

Though DRL has shown great potential in decision-making area, researchers have also found that RL algorithms do not necessarily guarantee safety during learning or execution phases, which leads to an unavoidable drawback for safety-critical application in real world such as robot control\cite{alshiekh2018safe,gu2022review}. To get over this issue, researchers have studied to ensure reasonable system performance and respect safety constraints during the learning and deployment processes, such researches are so-called SRL\cite{garcia2015comprehensive}. Considering there are many different ways to classify SRL algorithms, application area will be used in this paper to make the literature review.

In the car autonomous driving area, many methods have been proposed for autonomous driving based on modern, advanced techniques. Traditional
motion planning and RL methods are combined to perform better than pure RL or traditional methods\cite{gu2022constrained}. Different with \cite{gu2022constrained}, a third layer called risk neural network is added to AC algorithm to realize safe autonomous driving\cite{wen2020safe}. Moreover, control barrier functions (CBF) and Monte Carlo Tree Search (MCTS) are also used for safe autonomous driving\cite{cheng2019end,mo2021safe}.

In the robotics area, the safety of robot is not considered as an optimization objective in the past researches, the study of SRL has established a bridge between the simulation and application.
Optlater ensures safe action sets that a robot can only taken from\cite{pham2018optlayer}. Safe exploration derived from risk function is used to construct PI-SRL and PS-SRL algorithm, which makes SRL based robot walking come true\cite{garcia2012safe,garcia2020teaching}.

SRL is also widely used in other areas. In recommender system, SRL is deployed to optimize the healthy recommendation sequences by utilizing a policy gradient method\cite{singh2020building}. In wireless security, an Inter-agent transfer learning based SRL method is proposed\cite{lu2022safe}. In UAV control, a brain-inspired reinforcement learning model (RSNN) is proposed to realize self-organized decision-making and online collision avoidance\cite{zhao2022nature}. In high speed train operation optimization, a Shield SARSA algorithm is proposed to plan an energy-efficient speed profile without overspeed operation\cite{zhao2022safe}. In diabetes treatment, a new mathematical problem formulation framework called Seldonian optimization problem is proposed, and it is used in the optimization of insulin dosing for type 1 diabetes treatment\cite{thomas2019preventing}.

\subsection{Problems and Contributions}
Through the introduction and literature review, two information can be acquired are that SRL is now widely used in safety-critical area to make RL based method more real world realizable and in urban rail transit area there are few researches typically considering how to construct a SRL based intelligent control method for autonomous operation. Table \ref{ref} summarized the safety protection methods for solving train operation optimization or control problems using RL in recent years. It is clear that the widely used method to prevent overspeed operation nowadays are adding a punishment or set the speed equal to the limit speed. However, the effect of a punishment may be influenced by the value of punishment weight and can only be known after several simulations which is obvious unsuitable for the operation in real world. When setting the speed equal to limit speed, this behavior may break the principle that at each time step the agent receives some representation of the environment’s state and on that basis selects an action \cite{sutton2018reinforcement}. Moreover, since this approach ignores the behavior policy, it may not maximize the long term reward.
\begin{table}[htbp]
        \renewcommand{\arraystretch}{1.2}
        \setlength{\abovecaptionskip}{-0.2cm} 
	\caption{Typical Safety Protection Methods }
	\label{ref}
	\begin{center}
		\begin{tabular}{cc}
			\hline
			\textbf{References} & \textbf{Safety Protection Method}  \\
			\hline
		\cite{zhou2020smart}	  & If $v>v_l$, adapt minimum deceleration   \\
	\cite{shang2021deep}   & If $v>v_l$, using reference system to brake     \\
 \cite{2014Intelligent}   & If $v>v_l$, safety index equal to 0     \\
\cite{tang2020reinforcement}    & If $v>v_l$, add an overspeed punishment        \\
		\cite{wang2023cooperative} & If $v>v_l$, regard as infeasible    \\
	\cite{chen2023research}   & If $v>v_l$, reset environment       \\
 \cite{ning2021deep}   & $v\le v_l$ in model but not mentioned in RL algorithm        \\
 \cite{lin2023reinforcement}   & $v\le v_l$ in model but not mentioned in RL algorithm         \\
 \cite{li2023intelligent}   & If $v>v_l$, add an overspeed punishment -1       \\
 \cite{zhang2021intelligent}   & Add an avoidance punishment        \\
 \cite{su2022cooperative}   & If $v>v_l$, add an operspeed punishment      \\
			\hline
		\end{tabular}
	\end{center}
	\vspace{-0.45cm}
\end{table}

Then in this paper, a SRL framework called SSA-DRL is proposed for safe self-decision making of urban rail transit autonomous operation. The framework is consists of a Shield, a Searching tree, an Additional safe actor and a DRL framework. The main contributions of this paper can be summarized as follows:
\begin{itemize}
    \item The proposed SSA-DRL framework enables agent to learn safe control policies and ensure schedule constraints and operation efficiency.
    \item The safe actions are get by a white box searching tree model and an iterative formula that can be explained.
    \item The proposed SSA-DRL framework can effectively reduce the number of times that the protection mechanism works which means that the final agent has self-protection ability.
    \item The proposed SSA-DRL framework has transferability and robustness and is easy to deploy.
\end{itemize}

The remainder is organized as follows. In Section II, the preliminaries are introduced. In Section III, the proposed SSA-DRL framework is elaborated. In Section IV, simulation results are discussed and in Section V the conclusions are given.

\section{Preliminaries}
\subsection{Markov Decision Process}
A finite Markov Decision Process is usually denoted by the 5-tuple $\left( \mathcal{S},s_0,\mathcal{A},p,\mathcal{R} \right) $ with a finite state set $\mathcal{S}=\left\{s_0,...,s_n\right\}$, a unique initial state $s_0 \in \mathcal{S}$, a finite action set $\mathcal{A}=\left\{a_1,...,a_n\right\}$, a dynamic function $p:\mathcal{S}\times \mathcal{R}\times \mathcal{S}\times \mathcal{A}\rightarrow \left[ 0,1 \right] $ and a reward function $r:\mathcal{S}\times\mathcal{A}\rightarrow\mathbb{R}$.

The solving of a RL task is to find a mapping called policy written as $\pi$ from states to probabilities of selecting each possible action to receive a lot of rewards over a long episode. For the optimal policy, it is better than or equal to all other policies. Noted that there may be more than one optimal policy, thus all the optimal policies are denoted by $\pi_*$. All the optimal policies share the same optimal state-value function $v_*$ and optimal action-value function $q_*$ defined as 
\begin{equation}
    \left\{ \begin{array}{l}
	v_*\left( s \right) \doteq \underset{\pi}{\max}v_{\pi}\left( s \right)\\
	q_*\left( s,a \right) \doteq \underset{\pi}{\max}q_{\pi}\left( s,a \right)\\
\end{array} \right. 
\end{equation}

\subsection{State Value and Action Value}
Value function is widely used in RL to evaluate a given state or a state-action pair. Since the return an agent can get almost depend on the chosen action, thus value function is associated with the policy. The function $v_{\pi}\left(s\right)$, which calculates the value of state $s$ under policy $\pi$ is called a state value function and is denoted by \eqref{statevalue}.
\begin{equation}\label{statevalue}
    \begin{aligned}
    v_{\pi}\left( s \right) 
	&\doteq \mathbb{E}_{\pi}\left[ G_t\mid S_t=s \right]\\
	&=\mathbb{E}_{\pi}\left[ \sum_{k=0}^{\infty}{\gamma}^kR_{t+k+1}\mid S_t=s \right] ,\,\,\text{for\,\,all\,\,}s\in \mathcal{S}
\end{aligned}
\end{equation}
Similarly, the value function $q_{\pi}$ calculates the value of taking an action $a$ at state $s$ under policy $\pi$ and is denoted by \eqref{actionvalue}.
\begin{equation}\label{actionvalue}.
    \begin{aligned}
        q_\pi(s, a) &\doteq \mathbb{E}_\pi\left[G_t \mid S_t=s, A_t=a\right] \\
        &=\mathbb{E}_\pi\left[\sum_{k=0}^{\infty} \gamma^k R_{t+k+1} \mid S_t=s, A_t=a\right]
    \end{aligned}
\end{equation}
$\mathbb{E}\left[\cdot\right]$ denotes the expected value of a random variable.

\subsection{Off Policy DRL}
In off-policy RL, the agent uses a behavior policy for action selection during the learning process, and a target policy for updating the policy. This means that the policies used by the agent while learning are different from those actually executed. The core feature of off-policy RL is to seek the global optimal value. In DRL especially, due to the introduction of replay buffer, off-policy DRL algorithms are more common. The DDPG and SAC algorithms are two examples of off-policy methods based on policy gradient framework AC, which are used as benchmarks in this paper.

DDPG is an deterministic algorithm typically designed for continuous action set which concurrently learns a Q-function $Q\left(s,a\right)$ and a policy. It has two AC structures and uses off-policy data and the Bellman equation to learn the Q-function then uses it to learn the policy. At each state $s$, the optimal action is acquired by solving \eqref{ddpgoptimal}.
\begin{equation}\label{ddpgoptimal}
    a_*\left( s \right) =\text{arg}\underset{a}{\max}Q_*\left( s,a \right) 
\end{equation}

SAC is an algorithm that optimizes a stochastic policy in an off-policy way, forming a bridge between stochastic policy optimization and DDPG-style approaches \cite{SpinningUp2018}. Different from DDPG, SAC is suitable for both continuous and discrete action set. The core feature of SAC is entropy regularization, which means the algorithm is designed to search a trade-off between expected return and entropy. Unlike the traditional DRL algorithm, SAC finds the optimal policy by solving \eqref{sacentropy}.
\begin{equation}\label{sacentropy}
        \pi ^*=\text{arg}\max_{\pi}\underset{\tau \sim \pi}{\text{E}}\left[ \sum_{t=0}^{\infty}{\gamma ^t}\left( R\left( s_t,a_t,s_{t+1} \right) +\alpha H\left( \pi \left( \cdot |s_t \right) \right) \right) \right]
\end{equation}
\begin{equation}
        H\left( P \right) =\underset{x\sim P}{\text{E}}\left[ -\log P\left( x \right) \right] 
\end{equation}
$H$ is the entropy of $x$ calculated by its distribution $P$.

Although both DDPG and SAC have learned good agents on several benchmark tasks, there is no guarantee of safety in these algorithms, nor any other traditional off-policy RL algorithm. Therefore, another purpose of this paper is to combine DRL agents with other modules to both improve control efficiency and ensure safety.

\subsection{Monte Carlo Tree Search}
MCTS is an algorithm based on tree search and Monte Carlo method for decision-making problems, widely used in the field of games and RL. It simulates multiple possible states of a game or problem and selects the optimal action scheme to find the best decision. MCTS iteratively simulates the subsequent development of a searching tree, updates the nodes in the tree according to the simulated results, and selects one node by a policy as the action for the next step. The widely used policies are upper confidence bounds and $\epsilon -greedy$. The basic steps of MCTS are selection, expansion, simulation and back propagation. In this paper, the process of MCTS is addressed to better suitable for the proposed framework.

\subsection{Linear Temporal Logic}
Linear Temporal Logic (LTL) is a widely used temporal logic method in areas such as formal modeling and model checking \cite{pnueli1977temporal}. It can describe constraints and temporal relationships that need to be satisfied by a system in the past, present, and future using time sequence operators. Therefore, it is particularly suitable for describing reactive systems that generate outputs based on external inputs and the system's current state. LTL can conveniently and accurately describe the properties and constraints that a system needs to meet, and is typically described using linear temporal logic formulas.

A word is typically used to express an atomic proposition (AP) or the negation of an AP. Alphabet $\varSigma$ of AP is denoted as $2^{AP}$, where $2^{AP}$ represents the power set of $AP$. Subsequently, the sets of all finite and infinite sequences of elements from the alphabet $\varSigma$ are denoted as $\varSigma^*$ and $\varSigma^{\omega}$, respectively. Important and widely used properties such as safety, satisfiability, and liveness can be defined through linear temporal logic formulas.


\section{Method Formulation}
In this section, the framework of the proposed SSA-DRL is firstly shown in Fig.~\ref{framework}. It is clearly that SSA-DRL consists of four main modules: a Shield based protective module, a searching tree based module, a DRL module and an additional actor module. Then we will explain how these four modules work in detail.

\begin{figure*}[!t]
\centering
\includegraphics[scale=0.8]{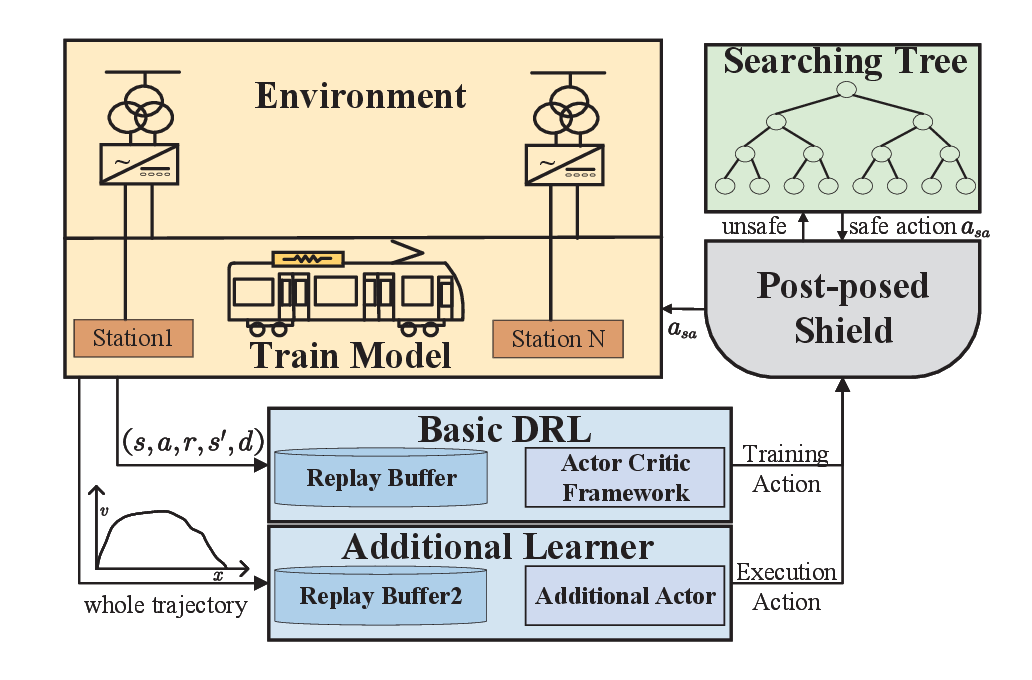}
\caption{Framework of SSA-DRL.}
\label{framework}
\end{figure*}

\subsection{A Post-Posed Shield}
The Shield in this paper comes directly from \cite{alshiekh2018safe}. The Shield consists of finite-state reactive system, safety specification, an observer function, a label and other components. Finite-State Reactive System is usually denoted by a tuple $\mathcal{FS}=\left\{ \mathcal{Q},q_0,\varSigma _I,\varSigma _O,\delta ,\lambda \right\} $, where $\mathcal{Q}$ is the finite set of states, $q_0 \in \mathcal{Q}$ is the initial state, $\varSigma_I=\varSigma_I^1\times\varSigma_I^2,\varSigma_O$ are the input and output alphabets. Then, $\delta:\mathcal{Q}\times\varSigma_I \rightarrow\mathcal{Q},\lambda:\mathcal{Q}\times\varSigma_I^1\rightarrow\varSigma_O$ are the transition and output function respectively. Specification $\psi$ defines the set of all allowed traces, and once a system satisfies all the properties of a specification, we can say a system satisfies $\psi$. Moreover, safety specification is used to construct the Shield. Formally speaking, a safety specification is a specification that if every trace is not in the language represented by the specification with a prefix such that all words starting with the same prefix are also not in the language\cite{alshiekh2018safe}. The above expression may be difficult to understand, readers of this paper can simply recognize a safety specification holds that "bad things will never happen" and a safety automaton can be used to represent a safety specification\cite{2001Model}. Observer function $f:\mathcal{S}\rightarrow L$ is usually a mapping for an MDP $\mathcal{M}=(\mathcal{S},s_0,\mathcal{A},p,\mathcal{R})$ to describe some information at state $s$ and $L$ is a finite set of label. Then, once an RL task can be formulated as $\mathcal{M}=(\mathcal{S},s_0,\mathcal{A},p,\mathcal{R})$ while satisfying a safety specification $\psi^{S}$ with an observer function $f:\mathcal{S}\rightarrow L$, a Shield can be modeled by a reactive system.

\begin{figure}[!t]
\centering
\includegraphics[scale=0.7]{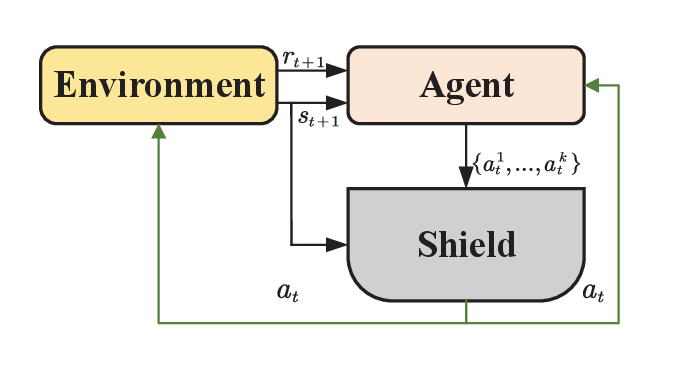}
\caption{Structure of post-posed Shield.}
\label{shield}
\end{figure}

Post-Posed Shield is a specific form of Shield that is set after the learning algorithm as depicted in Fig.~\ref{shield}. It is clear that the actions chosen by the agent are monitored by the Shield and the action violating safety specification $\psi^{\mathcal{S}}$ will be replayed by a safe action $a^{'}$. A new safe action set consists of $a^{'}$ is denoted as $\mathcal{A}_{s}^{'}$. Then the reactive system can be re-written as $\mathcal{FS}=(\mathcal{Q_{\mathcal{S}}},q_{0,\mathcal{S}},\varSigma_{I,\mathcal{S}}^{1} \times \varSigma_{I,\mathcal{S}}^{2},\varSigma_{O,\mathcal{S}},\delta_{\mathcal{S}},\lambda_{\mathcal{S}})$.


A simple example is made here to show how to build a post-posed Shield for the safe control of urban transit. Considering the train is running in a section with only one speed limitation which is 120km/h. The action set is $\mathcal{A}=\{\text{acceleration},\text{coasting},\text{braking}\}$. In the operation process, the speed must hold in the range of 1-119 km/h and the working condition cannot directly changed from acceleration to braking so as the braking to acceleration. Firstly, the safety specifications for the speed controller can be formulated by the temporal logic formula as follow \eqref{safetyspecification}.
\begin{equation}\label{safetyspecification}
\begin{aligned}
    	&\text{G}\left( \text{speed}>1 \right)\\
	\land &\text{G}\left( \text{speed}<119 \right)\\
	\land &\text{G}\left( \text{acceleration}\rightarrow\text{X}\left( \text{coasting} \right) \ \text{U\ braking} \right)\\
	\land &\text{G}\left( \text{braking}\rightarrow\text{X}\left( \text{coasting} \right) \ \text{U\ acceleration} \right)\\
\end{aligned}
\end{equation}
The meaning of LTL formulas \text{G},\text{X()},\text{U} are globally, next time and until respectively, and
the label set can be formulated as $L=\{speed<1,1\le speed \le 119,speed>119\}$. 

Then the component of Shield are discussed. Firstly, the finite state set $Q_{\mathcal{S}}$ can be set as $Q_{\mathcal{S}}=G$, where $G$ is the finite safe set of a safety game satisfies safety specification $\psi^{\mathcal{S}}$\cite{alshiekh2018safe}. Then the initial state is $q_{0,\mathcal{S}}=(q_0,q_{0,\mathcal{M}})$. We make a brief introduction of $q_{0,\mathcal{M}}$ here. As mentioned above, the reactive system to construct Shield should satisfy the safety specification, actually, it should satisfy another specification which is the MDP specification $\psi^{\mathcal{M}}$, and $q_{0,\mathcal{M}}$ is the initial state of $\psi^{\mathcal{M}}$. The input alphabet is $\varSigma_{I,\mathcal{S}}=\varSigma_{I,\mathcal{S}}^{1} \times \varSigma_{I,\mathcal{S}}^{2}=L\times\mathcal{A}=\{\text{acceleration},\text{coasting},\text{braking}\}\times\{speed<1,1\le speed \le 119,speed>119\}$ and the output alphabet is $\varSigma_{O,\mathcal{S}}=\mathcal{A}=\{\text{acceleration},\text{coasting},\text{braking}\}$. The output function can be formulated as \eqref{lambdafun} where $a\in \mathcal{A},a'\in \mathcal{A}_{\mathcal{S}}^{'},g\in G,l\in L$ and $W$ is the set of the winning state of safety game $G$.
\begin{equation}\label{lambdafun}
    \lambda _{\mathcal{S}}\left( g,l,a \right) =\left\{ \begin{array}{l}
	a\ \text{if}\ \delta \left( g,l,a \right) \in W\\
	a'\ \text{if}\ \delta \left( g,l,a \right) \notin W,\text{but}\ \delta \left( g,l,a' \right) \in W\\
\end{array} \right. 
\end{equation}
And the transition function is $\delta_{\mathcal{S}}(g,l,a)=\delta(g,l,\lambda_\mathcal{S}(g,l,a))$.

The above example points out the steps to build a post-posed Shield for urban rail transit control and then a searching tree based module is proposed to better find a safe action $a^{'}$.

\subsection{Safe Action Searching Tree}
In this subsection, a searching tree based module is proposed to output the final safe action. The idea of the searching tree derives from roll out algorithm and is more like a trade-off between MCTS and exhaustive search. Firstly, the Post-Posed Shield provides a safe action set $\mathcal{A}^{'}$ and the searching module using several steps to finally choose the high long-term reward safe action. The framework of the module is depicted in Fig.~\ref{searching tree}.

\begin{figure*}[!t]
\centering
\includegraphics[width=14.5cm]{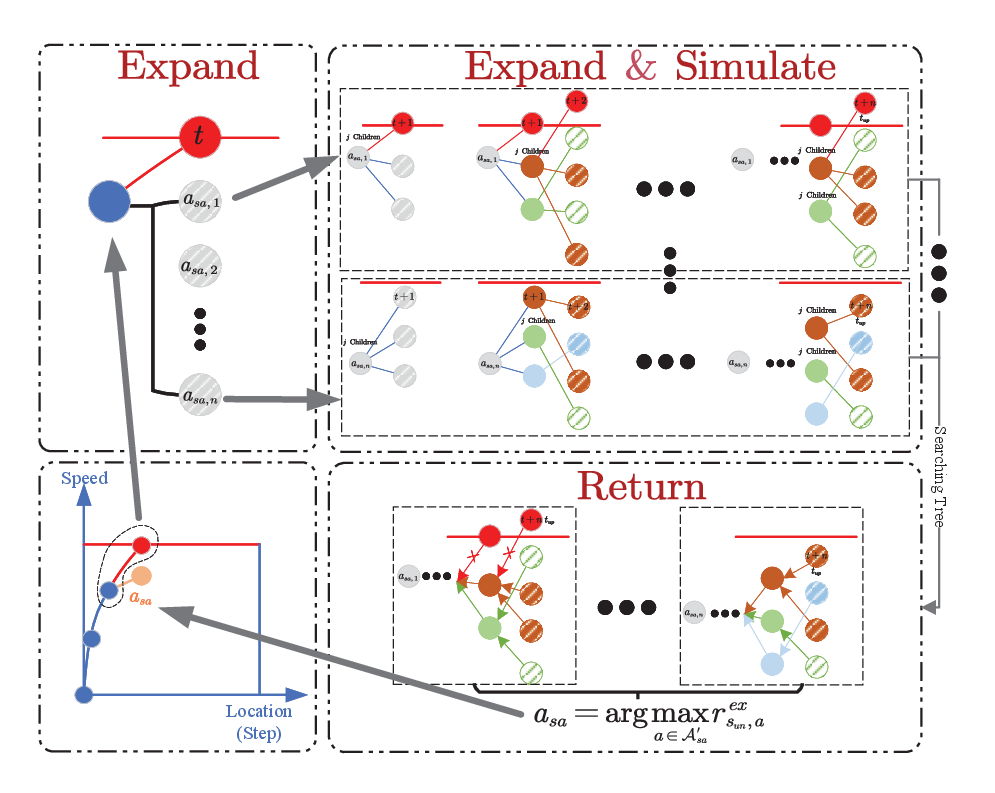}
\caption{Framework of safe action searching tree.}
\label{searching tree}
\end{figure*}

A detailed example is made here to explain how to construct and use the searching tree. Suppose that the initial unsafe state is $s_{un}$ and the safe action set is $\mathcal{A}_{sa}^{'}=(a_{sa,1}^{'},...,a_{sa,n}^{'})$. To output a high long-term reward safe action, each action in $\mathcal{A}_{sa}^{'}$ should be evaluated. $a_{sa,1}$ is chosen firstly then the state will transfer to a new state $s_{sa,1}$, this is actually a roll out or simulation step and $s_{sa,1}$ is the root node. At state $s_{sa,1}$, the DRL agent will output $n_{ex}$ actions, and a simulation step is then executed according to $s_{sa,1}$ and the $n_{ex}$ actions. It is noted here that all these simulation steps are monitored by the Shield then only safe action will be executed and in this step those unsafe actions will not be replaced by a safe one. This means that only safe actions can expand nodes and root node $s_{sa,1}$ will have $n_{ex}^{'}$ children nodes ($n_{ex}^{'}\le n_{ex}$). Then, for these nodes, a roll out step can be executed again. Considering in most DRL algorithms, the policy neural network will be updated in the learning process, thus if the searching tree is a full-depth tree there will be a tricky problem that the action in real training and in expansion at the same state may be different. For example, if the root node is step 3, the update frequency is 5 and the current expansion step is 6. In expansion, the action at step 6 is output by $\mu _{0}^{\theta}$ but the exact action in training at step 6 is output by $\mu _{5}^{\theta}$ where $\mu$ is the parameter of the policy neural network so that the expansion can not deduce a specific future. A trick is used here that the depth of the searching tree will not be fixed but dynamically equal to the remaining step to update the policy net, which means that the step of the leaf nodes will always be the same as the update step. In this case, the depth of the searching tree will not be too large and the each searching tree in training phase can be step-adaptive.

Once the expansion step reaches the update step, the searching tree needs to be pruned. In pruning, all children nodes that are not extended to the update step are deleted. Pruning can help to remove those nodes that will lead to an unsafe state and guarantee that only safe states are returned. After pruning, the searching tree needs to be returned. For the $q_{th}$ node at step $p$, the return $r_{p,q}^{ex}$ is calculated by \eqref{treereturn}
\begin{equation}\label{treereturn}
    \left\{ \begin{array}{l}
	r_{p,q}^{ex}=r_{p,q}^{si}+0.9*\mathbb{E}\left[ r_{p,q,chi}^{ex} \right] ,\text{branch\ node}\\
	r_{p,q}^{ex}=r_{p,q}^{si},\text{leaf\ node}\\
\end{array} \right. 
\end{equation}
where $r_{p,q}^{si}$ is the roll out reward and $\mathbb{E}\left[ r_{p,q,chi}^{ex} \right]$ is the expectation return of all children nodes of $node_{p,q}$. The final safe action $a_{sa}$ of state $s_{un}$ can be chosen by \eqref{safeac}
\begin{equation}\label{safeac}
a_{sa}=\text{arg}\max_{a\in \mathcal{A}_{sa}^{'}}r_{s_{un},a}^{ex}
\end{equation}

The pseudocode of the searching tree is shown in Alg.\ref{alg:searchingtree}.

\begin{algorithm}[htbp]
\caption{Searching Tree Process}\label{alg:searchingtree}
\begin{algorithmic}
\REQUIRE{The unsafe state,$s_{un}^{t}$;The safe action set, $\mathcal{A}_{sa}^{'}$;The policy net;The expansion width, $W$;The current step,$t$;The update frequency,$t_{up}$}
\ENSURE{A safe action, $a_{sa}$}
\WHILE{$m\le|\mathcal{A}_{sa}^{'}|$}
\STATE Get a new state $s_{sa,m}^{t+1}$ by roll out policy with action $a_m^{'}$
\IF{$(t+1)\%t_{up}\ne 0$}
\WHILE{$(t+1)\%t_{up}\ne 0$}
\WHILE{$w\le W$}
\STATE Get action $a_w$ by policy net
\IF{$a_w$ is monitored safe by Shield}
\STATE{Get new state $s_{sa,w}^{t+2}$by roll out policy with $a_m$}
\IF{$(t+2)\%t_{up}=0$}
\STATE Return Searching Tree by \eqref{treereturn}
\ENDIF
\STATE $w=w+1$
\ELSE
\STATE $w=w+1$
\ENDIF
\ENDWHILE
\STATE$t=t+1$
\ENDWHILE
\ELSE
\STATE Return Searching Tree by \eqref{treereturn}
\ENDIF
\STATE $m=m+1$
\ENDWHILE
\STATE $a_{sa}=\text{arg}\max_{a\in \mathcal{A}_{sa}^{'}}r_{s_{un},a}^{ex}$
\end{algorithmic}
\label{searchingtree}
\end{algorithm}

\begin{algorithm}[htbp]
\caption{SSA-DRL algorithm}\label{alg:SSA}
\begin{algorithmic}
\REQUIRE{Policy neural net parameter,$\theta$;Q-function neural net parameter,$\phi$;Additional policy neural net parameter,$\hat{\theta}$;The maximum training episode,$J$}
\ENSURE{Safe policy parameter, $\hat{\theta}_{sa}$}
\STATE Initialize parameter $\theta,\phi,\hat{\theta}$
\STATE Construct Post-posed Shield $\mathcal{FS}_{\text{sh}}$
\WHILE{ $j\le J$}
\REPEAT
\STATE Observe state $s$ and get an action $a$ by \eqref{OUnoise} or \eqref{Ganoise}
\STATE $\mathcal{FS}_{\text{sh}}$ check action $a$
\IF{$a$ is safe}
\STATE $a_{sa}=a$
\ELSE
\STATE Choose a safe action $a_{sa}$ by Alg.\ref{searchingtree}
\ENDIF
\STATE Execute action $a_{sa}$ and observe next state $s^{'}$, done signal $d$ and reward $r$
\STATE $\mathcal{D}\gets \mathcal{D}\cup \left( s,a,r,s',d \right) $
\IF{$s'$ is the terminal step and $R_j\ge \underset{\text{tr}\in \hat{\mathcal{D}}}{\min\text{\ tr}\left[ r \right]}$}
\STATE $\hat{\mathcal{D}}\gets \hat{\mathcal{D}}\cup \text{tr}_{j} $
\ELSE
\STATE Reset the environment
\ENDIF
\IF{update basic neural network}
\STATE Sample transitions $\mathcal{B}=\{(s,a,r,s^{'},d)\}$ from $\mathcal{D}$
\STATE Compute targets by \eqref{ytarget}, update $\phi$ by \eqref{qtarget}
\STATE Update $\theta$ by
\begin{equation*}
\nabla_{\theta} \frac{1}{|\mathcal{B}|}\sum_{s \in \mathcal{B}}Q_{\phi}(s, \mu_{\theta}(s))
\end{equation*}
\IF{there exists target network}
\STATE Update target networks by soft update
\ENDIF
\ENDIF
\IF{update additional policy neural network}
\STATE Sample whole trajectories $\hat{\mathcal{B}}=\{\text{tr}\}$ from $\hat{\mathcal{D}}$
\FOR{many updates}
\STATE Update $\hat{\theta}$ by \eqref{addup}
\ENDFOR
\ENDIF
\UNTIL Environment is reset
\STATE $j=j+1$
\ENDWHILE
\end{algorithmic}
\label{SSA}
\end{algorithm}
\subsection{DRL based guiding learner}
Though a post-posed Shield has a strong ability to prevent the occurance of unsafe actions, it has two disadvantages \cite{alshiekh2018safe}:
\begin{itemize}
    \item Unsafe actions may be part of the final policy, thus the Shield needs to be active even in after learning phase.
    \item Unsafe actions always will be replaced by safe actions, thus the agent will never learn how to avoid unsafe actions by itself.
\end{itemize}

These two disadvantages are both severe for calculating time-critical tasks like urban rail transit control because if the Shield is always active, the solving time of a safe action may be longer than the control cycle. The second disadvantage will also lead to another problem that the agent does not has the self-protection ability. Aiming at these two disadvantages, the SSA-DRL based on the Shield and searching tree is introduced and simple AC algorithm is used here to illustrate how the learner work.

The SSA-DRL seeks to solve the following optimization problem:
\begin{equation}
\begin{array}{c}
	\pi^{*}=\text{arg} \underset{\pi}{\max} \mathbb{E}[\sum_{t=0}^T{\gamma ^tr\left( s_t,a_{sa,t} \right)}]\\
	\text{s.t.\,\,}a_{sa,t}\in \mathcal{A}^{'}\\
\end{array}
\end{equation}
Then the action-value function $Q^{\pi}(s_t,a_t)$ for policy $\pi$ at step $t$ can be calculated by the given policy net $\mu^{\theta}$ or the searching tree $\mathcal{T}(\cdot)$denoted by \eqref{qsa}.
\begin{equation}\label{qsa}
\begin{aligned}
    Q^{\pi}\left( s_t,a_t \right) 
	&=Q\left( s_t,\left( \mu ^{\theta}\left( s_t \right) ,\mathcal{T}\left( s_t \right) \right) \right)\\
	&=\mathbb{E}\left[ R_{t+1}+\gamma Q^{\pi}\left( s_{t+1},a_{t+1} \right) \right] ,a_t,a_{t+1}\in \mathcal{A}^{'}\\
\end{aligned}
\end{equation}

In the learning process of DRL, random noise is always used to  increase exploration. The noise is added to the policy net and facilitates exploration in the action set for more efficient exploration. Ornstein-Uhlenbeck (OU) noise and Guassian noise are widely used since OU noise is autocorrelation and Gaussian noise is easy to design and realize in real world. Thus, the action chosen by the policy net can be represented as:
\begin{equation}\label{OUnoise}
\left\{ \begin{array}{l}
	\mu ^{\theta}\left( s_t \right) =\mu ^{\theta}\left( s_t|\varTheta _{t}^{\upsilon} \right) +\varepsilon\\
	\varepsilon \sim \mathcal{N}_{\text{OU}}\\
\end{array},\text{OU\,\,noise} \right. 
\end{equation}
\begin{equation}\label{Ganoise}
    \left\{ \begin{array}{l}
	\mu ^{\theta}\left( s_t \right) =\mu ^{\theta}\left( s_t \right) +\lambda \beta\\
	\beta \sim \mathcal{N}_{\text{Ga}}\left( 0,1 \right)\\
\end{array},\text{Ga\,\,noise} \right. 
\end{equation}

The parameter of the action-value function $\phi$ can be learned by minimizing the loss function of the critic net as presented by \eqref{qtarget}.
\begin{equation}\label{qtarget}
\nabla_{\phi} \mathbb{E}_{(s,a,r,s',d) \sim \mathcal{B}} \left[ (Q_{\phi}^{\pi}(s,a) - y(r,s',d))^2 \right]
\end{equation}
Where $\mathcal{B}$ is a sample batch of transitions $(s,a,r,s',d)$ stored in replay buffer $\mathcal{D}$. $y(\cdot)$ is called the target and is usually computed by \eqref{ytarget}.
\begin{equation}\label{ytarget}
     y(r,s',d) = r + \gamma (1-d) Q_{\phi}(s', \mu^{\theta}(s'))
\end{equation}
It is noted here that the memory in replay buffer $\mathcal{D}$ follows first in first out principle and the experience batch is randomly sampled thus the sampled experience may not contain a whole trajectory. Then another replay buffer $\hat{\mathcal{D}}$ and an additional policy neural net $\hat{\mathcal{\mu}}$ are introduced.

The replay buffer $\hat{\mathcal{D}}$ is used to save $\hat{N}$ whole trajectories $(S_{\hat{n}},A_{\hat{n}},R_{\hat{n}})$ with highest reward $\hat{R}_{\hat{n}}$ denoted by \eqref{bestmemory}. The experiences in $\hat{\mathcal{D}}$ follow the best in worst out principle and are ranked by the value of the reward.
\begin{equation}\label{bestmemory}
    \begin{aligned}
        \hat{\mathcal{D}}&=\left[ \text{tr}_1,...,\text{tr}_{\hat{n}} \right],\hat{n}\in[1,\hat{N}]\\
        \text{tr}_{\hat{n}}&=\left[ s_t^{\text{tr}_{\hat{n}}},a_t^{\text{tr}_{\hat{n}}},r_t^{\text{tr}_{\hat{n}}}  \right] ,t\in \left[ 0,T \right] 
    \end{aligned}
\end{equation}

The additional net $\hat{\mu}$ with parameter $\hat{\theta}$ is used to study the self-protection ability. Since there exists $a_t^{\text{tr}_{\hat{n}}}\in \mathcal{A}_{sa}^{'}$, thus if $\forall \hat{n}\in \hat{N},\exists \mathbb{E}[| a_{t}^{\text{tr}_{\hat{n}}}-\hat{\mu}^{\hat{\theta}}\left( s_{t}^{\text{tr}_{\hat{n}}} \right) |^{2}] <\hat{\epsilon}$, the additional policy net can be used as the final policy net. The parameter $\hat{\theta}$ can be updated by \eqref{addup}.
\begin{equation}\label{addup}
    \nabla _{\hat{\theta}}\frac{1}{|\hat{\mathcal{B}}|}\sum_{\left( s_{t}^{\text{tr}_{\hat{n}}},a_{t}^{\text{tr}_{\hat{n}}} \right) \in \hat{\mathcal{B}}}{\left( a_{t}^{\text{tr}_{\hat{n}}}-\hat{\mu}^{\hat{\theta}}\left( s_{t}^{\text{tr}_{\hat{n}}} \right) \right)}^2, |\hat{\mathcal{B}}|\le\hat{N}
\end{equation}
Where $\hat{\mathcal{B}}$ is a sample batch of several whole trajectories stored in $\hat{\mathcal{D}}$. Moreover, the structure of $\hat{\mu}$ may be different from the structure of $\mu$, they only have the same dimension and scale of input and output. This method can improve the generalization ability and deployability of the final policy net. 
Then the whole SSA-DRL algorithm is outlined in Alg.\ref{SSA}. It is noted here that the specific steps for updating different DRL algorithms are not exactly the same, the steps in Alg.\ref{SSA} only involves three core ideas in DRL which are policy evaluation, policy improvement and target network. Once other DRL algorithm is used to implement the SSA-DRL algorithm, only some update steps in Alg.\ref{SSA} need to be addressed but the core idea is unchanged. In simulation section, the effectiveness of SAC based SSA-SAC algorithm is a good example to verify the above idea.

\subsection{Optimality and Convergence Analysis}
\subsubsection{Optimality}
In the learning process, there actually exists two policies to get an action, the policy net and the searching tree. The optimality of the policy net does not need to be proved. Then two aspects of the optimality of searching tree are discussed, the first is the original optimality of $a_{sa}$ itself and the second is the policy $\pi_{sa}$ to get $a_{sa}$ is no less than the policy net $\mu^{\theta}$ to get an action once a state is monitored unsafe.
\newtheorem{lemma}{Lemma}
\begin{lemma}
    For policies to get a safe action, $\pi_{sa}$ to get $a_{sa}$ is better than any other policies to get another $\pi_{sa}^{'}$. Moreover, $\pi_{sa}$ is no less than the original policy $\mu^{\theta}$ to get a safe action.
\end{lemma}
\begin{proof}
    The concept of policy improvement is used here to make the proof. In Alg.\ref{searchingtree} obviously, the final safe action $a_{sa}$ is actually chosen by the greedy strategy, thus $a_{sa}$ is the action with the highest long-term reward  and then $\pi_{sa}$ is better than other $\pi_{sa}^{'}$.
    Then, the idea of policy improvement is used to prove that when choosing a safe action, $\pi_{sa}$ is no less than $\mu^{\theta}$. Since the actions used in searching tree are all generated by $\mu^{\theta}$, thus they are identical equal besides the initial unsafe state $\pi_{sa}(s_{un})\ne \mu^{\theta}(s_{un})$. Then if $q_{\mu_{\theta}}(s_{un},\pi_{sa}(s_{un}))\ge v_{\mu^{\theta}}(s_{un})$, the policy $\pi_{sa}$ is no less than policy $\mu^{\theta}$. Keep expanding $q_{\mu^{\theta}}$ we can get \eqref{optimalproof}. It is obvious $q_{\mu_{\theta}}(s_{un},\pi_{sa}(s_{un}))\ge v_{\mu^{\theta}}(s_{un})$, thus $\pi_{sa}(s_{un})$ is no less than $\mu^{\theta}$. We must admit here this proof is not very rigorous, since the original policy improvement method requires $t+n$ to be infinite, but in this paper a clipped form is used, which is $\left( t+n \right) \%t_{up}\ne 0$. 
\begin{equation}\label{optimalproof}
\begin{array}{l}
	v_{\mu^\theta}\left(s_{un}\right)\leq q_{\mu ^{\theta}}\left( s_{un},\pi _{sa}\left( s_{un} \right) \right) \vspace{1ex}\\
	=\mathbb{E}\left[ R_{t+1}+\gamma v_{\mu ^{\theta}}\left( S_{t+1} \right) \mid S_t=s_{un},A_t=\pi _{sa}\left( s_{un} \right) \right] \vspace{1ex}\\
	=\mathbb{E}_{\pi _{sa}}\left[ R_{t+1}+\gamma v_{_{\mu ^{\theta}}}\left( S_{t+1} \right)\mid S_t=s_{un} \right] \vspace{1ex}\\
	\leq \mathbb{E}_{\pi _{sa}}\left[ R_{t+1}+\gamma q_{_{\mu ^{\theta}}}\left( S_{t+1},\pi _{sa}\left( S_{t+1} \right) \right) \mid S_t=s_{un}\right] \vspace{1ex}\\
	=\mathbb{E}_{\pi _{sa}}\left[ R_{t+1}+\gamma \mathbb{E}\left[ R_{t+2}+\gamma v_{\mu ^{\theta}}\left( S_{t+2} \right) \right] \mid S_t=s_{un}\right] \vspace{1ex}\\
	=\mathbb{E}_{\pi _{sa}}\left[ R_{t+1}+\gamma R_{t+2}+\gamma ^2v_{\mu ^{\theta}}\left( S_{t+2} \right) \mid S_t=s_{un}\right] \vspace{1ex}\\
	\leq \mathbb{E}_{\pi _{sa}}\left[ R_{t+1}+\gamma R_{t+2}+\gamma ^2R_{t+3}+\gamma ^3v_{\mu ^{\theta}}\left( S_{t+3} \right) \mid S_t=s_{un}\right] \vspace{1ex}\\
	\vdots\\
	\leq \mathbb{E}_{\pi _{sa}}\left[ \underset{\text{eq}.\left( 11 \right)}{\underbrace{R_{t+1}+\cdots +\gamma ^{n-1}R_{t+n}}} \mid S_t=s_{un}\right] \vspace{1ex}\\
	=v_{\pi _{sa}}\left( s_{un} \right)\vspace{1ex}\\
\end{array}
\end{equation}
\hfill \qedhere
\end{proof}

    
\subsubsection{Convergence}
The convergence of the post-posed Shield has been verified in \cite{alshiekh2018safe} and the feasibility of using MDP specification and safety specification to protect the operation of train has been verified in \cite{zhao2022safe}, thus the post-posed Shield in this paper satisfies the convergence analysis in \cite{alshiekh2018safe}. The simulation results in Section V also verifies the convergence of SSA-DRL algorithm.
\subsection{Algorithm Implementation}
In this subsection, the state set, action set, reward function and the relationship between action and the acceleration are discussed to complete the SSA-DRL based urban rail transit autonomous operation algorithm.
\subsubsection{State Set}
The location $\textit{loc}$, velocity $\textit{vel}$, running time $\textit{time}$ are used to formulated the state set. Thus the state of the agent at step $t$ can be formulated as \eqref{stateset}.
\begin{equation}\label{stateset}
    s_{t}=(\textit{loc}_{t},\textit{vel}_{t},\textit{time}_{t})
\end{equation}
\subsubsection{Action Set}
The percentage of the traction braking control command output to the motor of train is used as the action. The action set is continuous and ranges from -1 to 1. If the value is less than 0, the command is braking otherwise the command is traction. Then the action at step $t$ can be formulated as \eqref{actionset}.
\begin{equation}\label{actionset}
    a_t\in[-1,1]
\end{equation}
\subsubsection{Reward Function}
Since the operation has been protected and the main purpose is autonomous driving, then operation energy consumption $E$, operation time difference $D^T$ and the comfort of passengers $C$ are used to build the reward function. In each transition step $t$, $E_t,D^T_t,C_t$ are calculated by \eqref{energy},\eqref{time} and \eqref{comfort}. 
\begin{equation}\label{energy}
    E_t=\left\{ \begin{array}{l}
	\alpha _{E^{tr}}*E^{tr},a>0\\
	\alpha _{E^{re}}*E^{re},a\le 0\\
\end{array} \right. 
\end{equation}
\begin{equation}\label{time}
    D^T_t=\left\{ \begin{array}{l}
	\alpha _{D^T}*|T^{total}-T^{sch}|,\text{terminal}\\
	\alpha _{D^T}^{'}*|\bar{v}_t-\bar{v}|,\text{mid\,\,step}\\
\end{array} \right. 
\end{equation}
\begin{equation}\label{comfort}
    C_t=\left\{ \begin{array}{l}
	\kappa ,\Delta acc>\sigma\\
	0,\Delta acc<\sigma\\
\end{array} \right. 
\end{equation}
Where $E^{tr}$ is the traction energy consumption, $E^{re}$ is the recovered regenerative braking energy, $\alpha _{E^{tr}},\alpha _{E^{re}},\alpha_{D^{T}},\alpha_{D^{T}}^{'}$ are the weights of energy reward and time reward, $T^{total},T^{sch}$ are the total operation time and scheduled operation time, $\bar{v_{t}},\bar{v}$ are the average speed in step $t$ and the overall average speed, $\kappa$ is a punishment indicator, $\Delta acc$ is the rate of change of acceleration and $\sigma$ is a threshold. The reward of transition step $t$ is then defined as \eqref{reward}. 
\begin{equation}\label{reward}
    R_t=-E_t-D^{T}_t-C_t
\end{equation}
It is also noted that the value of $E^{re}$ is set negative then the train can learn to reduce total energy consumption.
\subsubsection{Relationship between control command and train operation}
The operation of train is restricted by the traction braking characteristic thus the control command cannot directly represent the movement of train. Suppose a given command $a^{tr}_{t}>0$ at step $t$, the acceleration output by the traction motor is $acc^{motor}_{t}=F^{tr}_{+}(v_{t})*a^{tr}/m^{train}$ where $F^{tr}_{+}(v_{t})$ denotes the max traction force at speed $v_{t}$ is a function of speed and $m^{train}$ is the weight of train. Likely, the acceleration of braking can be computed in the same way. Then the actual acceleration of train can be computed by \eqref{acctrain}.
\begin{equation}\label{acctrain}
    acc^{train}=\frac{\left( acc^{motor}-acc^r+g\left( x \right) \right)}{v}
\end{equation}
Where $acc^{r}$ is the resistance acceleration calculated by the Davis function and $g(x)$ is the gravity acceleration at location $x$. Moreover, there exists no steep slope in this paper thus \eqref{gravitycondition} always holds.
\begin{equation}\label{gravitycondition}
    0\le acc^r-g\left( x \right) \le \max \left( acc^{motor} \right) 
\end{equation}

\section{Simulation Results and Performance Evaluation}
\subsection{Simulation Environment}
The simulation is based on Chengdu urban rail transit line 17, there are a total of sixteen sections in up and down directions. The value of the Davis parameters are $r_1=8.4,r_2=0.1071,r_3=0.00472$ respectively, the weight of the train is 337.8 ton, the max and min acceleration are $\pm 1.2\text{m/s}^2$ and the traction braking characteristic is shown in Fig.~\ref{character}. The main parameters used to construct the SSA-DRL algorithm are shown in Table\ref{table1}.

\begin{figure}[!t]
\centering
\includegraphics[width=3in]{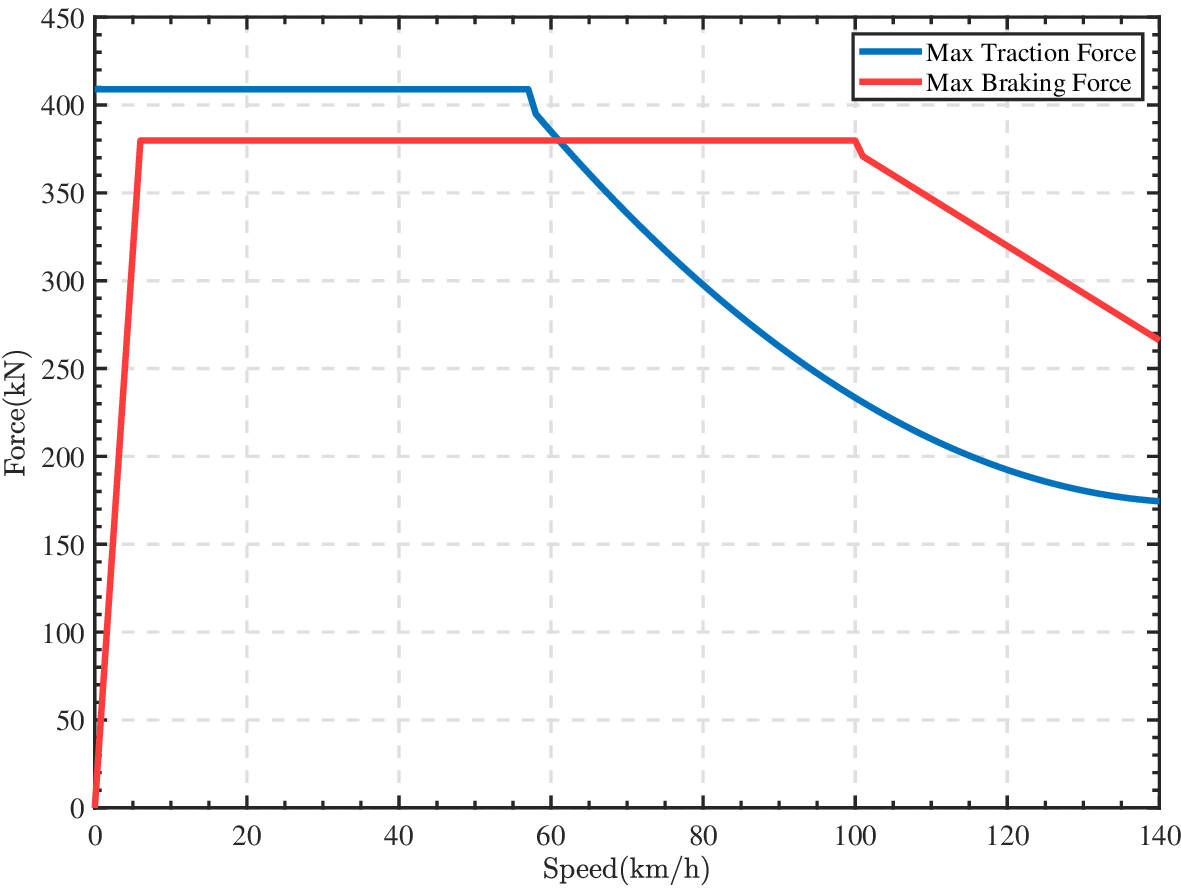}
\caption{Traction and braking characteristic.}
\label{character}
\end{figure}

\begin{table}[htbp]
        \renewcommand{\arraystretch}{1.2}
        \setlength{\abovecaptionskip}{-0.2cm} 
	\caption{The Main Parameters of the Algorithm}
	\label{table1}
	\begin{center}
		\begin{tabular}{cccc}
			\hline
			\textbf{Parameters} & \textbf{Value} & \textbf{Parameters} & \textbf{Value} \\
			\hline
		Traction Weight	$\alpha _{E^{tr}}$  & 3  & Regenerative Weight $\alpha _{E^{re}}$  & 3 \\
	Time Weight	$\alpha _{D^T}$   & 15  & Time Weight $\alpha _{D^T}^{'}$   & 25   \\
		Comfort Punishment	$\kappa$    & 10     & Change Threshold $\sigma$    & 3   \\
		Minibatch $|\mathcal{B}|$ & 256     & Minibatch $|\hat{\mathcal{B}}|$  & 10   \\
	Max Episode $J$   & 500    & Update Frequency $t_{up}$    & 5   \\
			\hline
		\end{tabular}
	\end{center}
	\vspace{-0.45cm}
\end{table}

\begin{table}[htbp]
        \renewcommand{\arraystretch}{1.2}
        \setlength{\abovecaptionskip}{-0.2cm} 
	\caption{The Main Hyperparameters of the Algorithm}
	\label{table2}
	\begin{center}
		\begin{tabular}{cccc}
			\hline
			\textbf{Parameters (DDPG)} & \textbf{Value} & \textbf{Parameters (SAC)} & \textbf{Value} \\
			\hline
		Actor learning rate	 & $1\text{e}^{-5}$  & Policy learning rate   & $1\text{e}^{-5}$ \\
	Critic learning rate	& $1\text{e}^{-3}$  & Value learning rate    & $1\text{e}^{-3}$  \\
		Discount factor & 0.99     & Discount factor  & 0.99   \\
	Soft update factor   & $1\text{e}^{-2}$    & Soft update factor    & $1\text{e}^{-2}$   \\
         /& /     & Soft-Q learning rate    & $3\text{e}^{-5}$   \\
			\hline
		\end{tabular}
	\end{center}
	\vspace{-0.45cm}
\end{table}

The basic DDPG and SAC are used as the baseline to implement autonomous operation in this paper, and the AC networks are both designed through four fully connected hidden layer, the activation functions in hidden layers are Relu, the size of hidden layers is 256 and the final activation function of the actor net is tanh to ensure the output ranges in $\left[ -1,1 \right] $. The optimizers are all Adam. The main hyperparameters of these two algorithms are shown in Table.\ref{table2}. It is noted that in the simulation SSA-DRL algorithm shares the same hyperparameters with the baseline. Moreover, the originally additional actor is designed the same with the actor, but with a five times learning rate. The influence caused by the design of additional actor will be discussed by the ablation experiment. The proposed algorithm is implemented in Matlab and Python on a computer with an AMD Ryzen 7 5800X CPU @ 3.80Ghz and 32GB RAM running Windows 10 x64 Edition. 
\subsection{Basic Simulation} 
The basic simulation aims to verify that the proposed SSA-DRL can control the train complete the operation plan with higher reward and better performance under less protect times in both training and execution process. Here,the protect times is the number of counts the algorithm re-chooses a safe action to correct the original action in training or execution process. Fig.~\ref{fig:speedprofile} shows the speed profiles of the SSA-DRL algorithm in one simulation. Fig.~\ref{fig:uprewardcurve} shows the reward curves of SSA-DRL, Shield-DRL and common DRL algorithms. It can be clearly seen that in most scenarios SSA-DRL can achieve a higher reward than Shield-DRL and the reward of Shield-DRL is also higher than common DRL. Also, SSA-DRL can achieve convergence at a earlier step. It is noted that the reward curves are all smoothed by the moving average and the size of the window is 8. The detailed numerical results are shown in Table.\ref{table4}. The data in time and energy columns are acquired from one operation plan. Since the data of regenerative braking energy in real world can not be acquired, thus it is not considered in energy column but in simulation column. Moreover, the operation time is not fixed and a margin of thirty seconds is allowed. And in the simulation columns, the data are recorded by $ave\pm std$. Rreaders may think that the speed profiles do not match Table.\ref{table4}, it should be made clear that the speed profiles are only results of one simulation but the data in Table.\ref{table4} are numerical results after several simulation times.

Fig.~\ref{fig:protectcountsbox} shows the distribution of the protect times of SSA-DRL algorithms in both training process and execution process. For the SSA-SAC algorithm, the data are acquired by five different seeds with 500 training episodes and 10 execution episodes in each simulation. For the SSA-DDPG algorithm, the only difference is that 400 episodes in each simulation are used because the DDPG algorithm needs to fill the replay buffer. In this case, the data capacity are 2500, 50 for SSA-SAC and 2000, 50 for SSA-DDPG. To better illustrate that the proposed SSA-DRL can effectively reduce the protect times compared with Shield-DRL, we get the same amount of Shield-DRL data and the detailed results are shown in  Table.\ref{tablecompare}. In the Protect Times row, average protect times are shown and in the comparison row, the number is calculated by $|\frac{|\text{SSA-DRL}|-|\text{Sh-DRL}|}{|\text{Sh-DRL}|}|\times 100$. Thus if the number in comparison row is positive, it means that compared with the Shield-DRL algorithm, SSA-DRL algorithm has a decline in $\%$. The bold and underline in each column show the most and lest decline in the same process with the same basic DRL algorithm. It can be acquired from Table.\ref{tablecompare} that 1) the comparison rows are all positive thus SSA-DRL can always reduce the protect times compared with Shield-DRL, 2) except two special situations (Section1 up direction training and Section6 down direction execution), the decline in training and execution are large and the max decline are 80.84$\%$ and 100$\%$ in training and execution respectively, 3) except one special situation (Section8 up direction training), once the process is same, the protect times of Shield-DRL are always larger than the SSA-DRL no matter the basic algorithm is. Fig.~\ref{fig:protectcounts} shows bar graphs of the data in Table.\ref{tablecompare}. It is more clear that compared with the Shield-DRL algorithms, the protect times of SSA-DRL algorithms has greatly reduced.

\begin{table*}[htbp]
\renewcommand{\arraystretch}{1.5}
\setlength{\abovecaptionskip}{-0.2cm} 
\caption{The Original Simulation Results}
\label{table4}
\begin{center}
\begin{tabular}{cccccccccc}
\hline
\multirow{2}{*}{Section} & \multicolumn{1}{c}{\multirow{2}{*}{Direction}} & \multicolumn{1}{c}{\multirow{2}{*}{Time(s)}} & \multicolumn{2}{c}{Simulation Time(s)} & \multicolumn{1}{c}{\multirow{2}{*}{Energy(kw$\cdot$h)}} & \multicolumn{2}{c}{Simulation Energy(kw$\cdot$h)} & \multicolumn{2}{c}{Overspeed Counts} \\ \cline{4-5} \cline{7-10} 
                     & \multicolumn{1}{c}{}                           & \multicolumn{1}{c}{}                      & SSA-SAC           & SSA-DDPG          & \multicolumn{1}{c}{}                        & SSA-SAC            & SSA-DDPG            & SSA-SAC            & SSA-DDPG           \\ \hline
\multirow{2}{*}{1}   & Up                                             & 86                                        & 135.66$\pm$57.90        & 106.17$\pm$11.38               & 52                                          & 29.38$\pm$1.12                & 21.98$\pm$2.02                 & 0                & 0                \\ \cline{2-10}
                     & Down                                           & 92                                 & 107.80$\pm$11.45               & 109.74$\pm$6.07               & 45                                          & 21.61$\pm$1.66                & 24.12$\pm$5.25                 & 0                & 0                \\ \hline
\multirow{2}{*}{2}   & Up                                             & 183                                        & 215.46$\pm$21.24          & 228.30$\pm$27.35            & 116                                          & 53.52$\pm$2.04                &     47.67$\pm$7.34             & 0                & 0                \\ \cline{2-10}
                     & Down                                           & 204                             & 208.24$\pm$6.34               & 206.38$\pm$2.53               & 88                                          & 23.41$\pm$1.25                & 42.37$\pm$9.78                 & 0                & 0                \\ \hline
\multirow{2}{*}{3}   & Up                                             & 276                                        &   298.88$\pm$6.63  & 288.76$\pm$10.32               & 160                                          &   83.85$\pm$11.91              &   55.93$\pm$3.09               & 0                & 0                \\ \cline{2-10}
                     & Down                                           & 259                                        & 267.76$\pm$14.16               & 293.65$\pm$6.95               & 142                                          &  85.10$\pm$15.43             &  31.53$\pm$3.01               & 0                & 0                \\ \hline
\multirow{2}{*}{4}   & Up                                             & 104                                        & 130.45$\pm$15.27               & 137.77$\pm$12.90        & 56                                          & 30.54$\pm$1.87                & 36.73$\pm$6.63                 & 0                & 0                \\ \cline{2-10}
                     & Down                                           & 101                             & 107.75$\pm$1.89               & 104.71$\pm$3.88               & 61                                          & 16.44$\pm$1.32                & 31.01$\pm$6.85                 & 0                & 0                \\ \hline
\multirow{2}{*}{5}   & Up                                             & 104                                        & 121.53$\pm$1.63               & 130.28$\pm$21.16         & 68                                          & 21.80$\pm$0.71                & 33.77$\pm$6.63                 & 0                & 0                \\ \cline{2-10}
                     & Down                                           & 103                             & 107.64$\pm$2.58               & 106.03$\pm$3.49               & 67                                          & 31.60$\pm$2.68                & 39.84$\pm$4.87                 & 0                & 0                \\ \hline
\multirow{2}{*}{6}   & Up                                             & 105                                        & 126.78$\pm$5.87               & 117.11$\pm$3.34       & 63                                          & 42.10$\pm$2.70                & 47.08$\pm$4.69                 & 0                & 0                \\ \cline{2-10}
                     & Down                                           & 106                             & 110.03$\pm$4.31               & 109.46$\pm$3.63               & 48                                          & 20.19$\pm$9.13                & 16.32$\pm$5.46                 & 0                & 0                \\ \hline
\multirow{2}{*}{7}   & Up                                             & 105                                        & 113.42$\pm$5.09      & 114.54$\pm$10.82               & 63                                          & 40.07$\pm$1.40                & 43.86$\pm$3.86                 & 0                & 0                \\ \cline{2-10}
                     & Down                                           & 138                             & 138.82$\pm$4.74               & 134.46$\pm$1.36               & 48                                          & 13.90$\pm$3.59                & 16.31$\pm$4.11                 & 0                & 0                \\ \hline
\multirow{2}{*}{8}   & Up                                             & 172                                        &     170.64$\pm$4.30            & 180.07$\pm$3.70          & 63                                          &  75.15$\pm$6.27               & 60.23$\pm$4.26                & 0                & 0                \\ \cline{2-10}
                     & Down                                           & 171                            & 176.52$\pm$2.28               & 172.50$\pm$2.59               & 48                                          & 39.28$\pm$3.04                & 33.75$\pm$3.70                 & 0                & 0                \\ \hline
\end{tabular}
\end{center}
\vspace{-0.45cm}
\end{table*}

\begin{figure*}
\begin{minipage}{1\linewidth}
\centerline{\includegraphics[width=18cm]{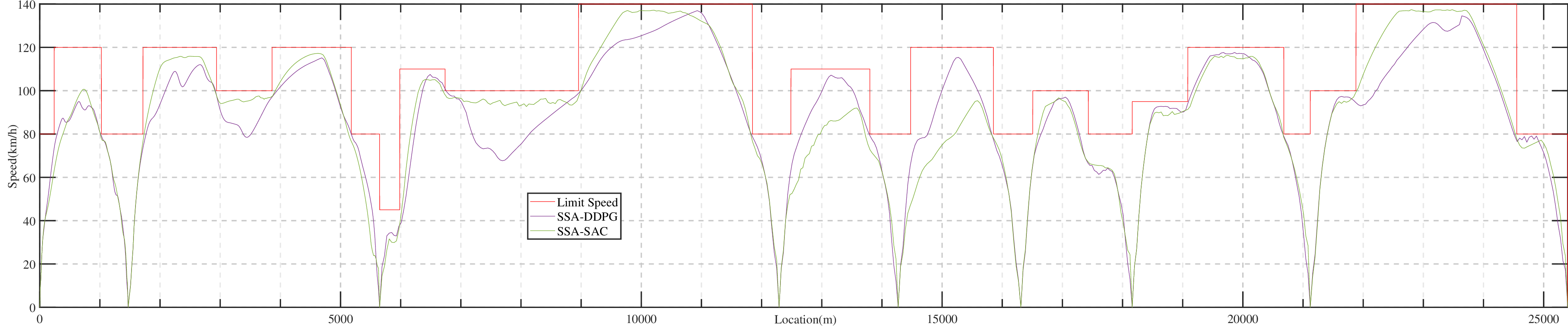}}
\centerline{(a) Speed profiles in up direction}
\end{minipage}
\hfill
\begin{minipage}{1\linewidth}
\centerline{\includegraphics[width=18cm]{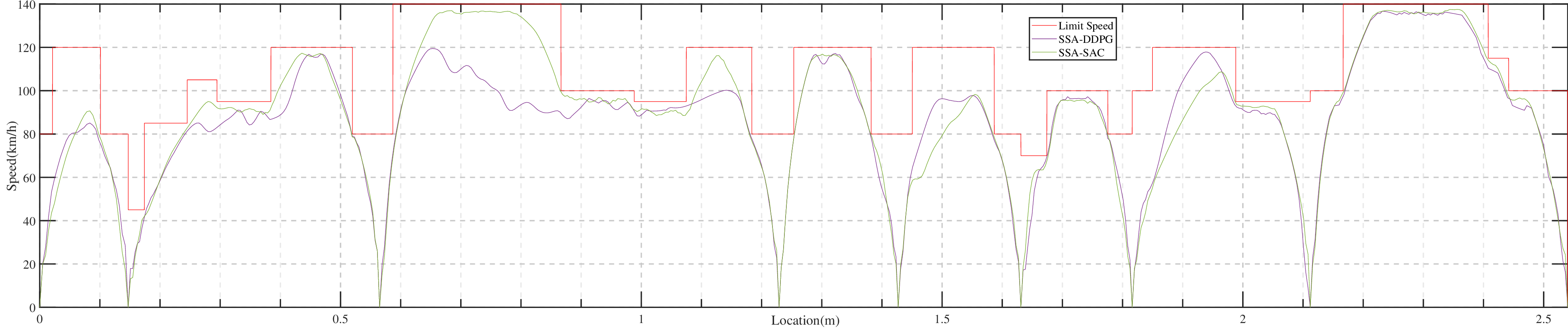}}
\centerline{(b) Speed profiles in down direction}
\end{minipage}
\hfill
\caption{Speed profiles in different sections.}
\label{fig:speedprofile}
\end{figure*}

\begin{table*}[htbp]
\renewcommand{\arraystretch}{1.5}
\setlength{\abovecaptionskip}{-0.2cm} 
\caption{The Protect Times Comparison between SSA-DRL and Shield-DRL}
\label{tablecompare}
\begin{center}
\begin{tabular}{cclllllllll}
\hline
\multirow{2}{*}{Section} & \multirow{2}{*}{Direction} & \multicolumn{1}{c}{\multirow{2}{*}{Content}} & \multicolumn{4}{c}{Train}  & \multicolumn{4}{c}{Execution}            \\ \cline{4-11} 
                         &                            & \multicolumn{1}{c}{}                         & SSA-SAC     & Sh-SAC    & SSA-DDPG    & Sh-DDPG& SSA-SAC     & Sh-SAC    & SSA-DDPG    & Sh-DDPG   \\
\hline
\multirow{4}{*}{1}       & \multirow{2}{*}{Up}        & Protect Times                      &    24.18    &  24.69    &  16.55        &  82.07 & 0.00 & 31.24& 5.40& 92.00   \\ \cline{3-11} 
                         &                            & Comparison($\%$)                                   & \multicolumn{2}{c}{$\underline{2.06}$} & \multicolumn{2}{c}{\textbf{79.83}}& \multicolumn{2}{c}{\textbf{100}} & \multicolumn{2}{c}{94.13} \\ \cline{2-11}
                         & \multirow{2}{*}{Down}      & Protect Times    &  12.77     &  16.66   &      6.60   &   31.90&0.04&29.86&1.80&19.00  \\ \cline{3-11}
                         &                            & Comparison($\%$)                                   & \multicolumn{2}{c}{19.32} & \multicolumn{2}{c}{$\underline{23.31}$}& \multicolumn{2}{c}{99.27} & \multicolumn{2}{c}{90.53} \\
\hline
\multirow{4}{*}{2}       & \multirow{2}{*}{Up}        & Protect Times             &    40.51   &  72.28    &   33.62    &  83.88&0.04&64.96&8.40&58.40   \\ \cline{3-11}
                         &                            & Comparison($\%$)                                   & \multicolumn{2}{c}{43.96} & \multicolumn{2}{c}{59.92}& \multicolumn{2}{c}{99.94} & \multicolumn{2}{c}{85.62} \\ \cline{2-11}
                         & \multirow{2}{*}{Down}      & Protect Times  &   12.02    &  62.75   & 24.33        &   45.00&0.06&19.40&16.00&86.75  \\ \cline{3-11}
                         &                            & Comparison($\%$)                                   & \multicolumn{2}{c}{\textbf{80.84}} & \multicolumn{2}{c}{45.93}& \multicolumn{2}{c}{99.69} & \multicolumn{2}{c}{81.56}\\ 
\hline
\multirow{4}{*}{3}       & \multirow{2}{*}{Up}        & Protect Times         &   36.20    &  62.98    &  29.96   & 79.75&0.04&68.92&24.20&136.00   \\ \cline{3-11}
                         &                            & Comparison($\%$)                                   & \multicolumn{2}{c}{42.52} & \multicolumn{2}{c}{62.43}& \multicolumn{2}{c}{99.94} & \multicolumn{2}{c}{82.21} \\ \cline{2-11}
                         & \multirow{2}{*}{Down}      & Protect Times      &     24.49     &  79.13    &     45.13     &   99.74&0.04&103.65&9.40&126.40  \\ \cline{3-11}
                         &                            & Comparison($\%$)                                   & \multicolumn{2}{c}{69.04} & \multicolumn{2}{c}{54.76}& \multicolumn{2}{c}{99.96} & \multicolumn{2}{c}{92.56}\\ 
\hline
\multirow{4}{*}{4}       & \multirow{2}{*}{Up}        & Protect Times    &      9.67   &   23.70  &   15.69      &  35.54&0.00&24.80&5.80&38.60   \\ \cline{3-11}
                         &                            & Comparison($\%$)                                   & \multicolumn{2}{c}{59.18} & \multicolumn{2}{c}{55.85}& \multicolumn{2}{c}{\textbf{100}} & \multicolumn{2}{c}{84.97} \\ \cline{2-11}
                         & \multirow{2}{*}{Down}      & Protect Times       &   15.85  & 23.88     &       21.79   &   27.78&0.10&18.54&1.80&39.80    \\ \cline{3-11}
                         &                            & Comparison($\%$)                                   & \multicolumn{2}{c}{33.61} & \multicolumn{2}{c}{36.01}& \multicolumn{2}{c}{99.46} & \multicolumn{2}{c}{95.48}\\ 
\hline
\multirow{4}{*}{5}       & \multirow{2}{*}{Up}        & Protect Times      &   10.83       &  25.20   &   22.21   &  29.21&0.08&28.72&0.20&35.00   \\ \cline{3-11}
                         &                            & Comparison($\%$)                                   & \multicolumn{2}{c}{57.04} & \multicolumn{2}{c}{23.98}& \multicolumn{2}{c}{99.72} & \multicolumn{2}{c}{\textbf{99.43}} \\ \cline{2-11}
                         & \multirow{2}{*}{Down}      & Protect Times   &   11.16  &  23.75 &  20.40        &  27.78&0.14&18.28&7.40&34.00  \\ \cline{3-11}
                         &                            & Comparison($\%$)                                   & \multicolumn{2}{c}{53.03} & \multicolumn{2}{c}{26.57}& \multicolumn{2}{c}{99.23} & \multicolumn{2}{c}{72.24}\\ 
\hline
\multirow{4}{*}{6}       & \multirow{2}{*}{Up}        & Protect Times       &    8.54    &   33.74     &   15.44       & 30.86&0.02&33.75&8.20&38.00    \\ \cline{3-11}
                         &                            & Comparison($\%$)                                   & \multicolumn{2}{c}{74.68} & \multicolumn{2}{c}{49.97}& \multicolumn{2}{c}{99.94} & \multicolumn{2}{c}{78.42} \\ \cline{2-11}
                         & \multirow{2}{*}{Down}      & Protect Times   &   12.09      &  25.55    &     18.99    &  26.02&0.12&30.25&18.60&18.80   \\ \cline{3-11}
                         &                            & Comparison($\%$)                                   & \multicolumn{2}{c}{52.66} & \multicolumn{2}{c}{27.01}& \multicolumn{2}{c}{99.60} & \multicolumn{2}{c}{$\underline{1.06}$}\\ 
\hline
\multirow{4}{*}{7}       & \multirow{2}{*}{Up}        & Protect Times     &    18.02     &  52.31    &   21.41    &  31.87&0.02&53.74&27.40&58.80   \\ \cline{3-11}
                         &                            & Comparison($\%$)                                   & \multicolumn{2}{c}{65.56} & \multicolumn{2}{c}{32.81}& \multicolumn{2}{c}{99.96} & \multicolumn{2}{c}{53.40} \\ \cline{2-11}
                         & \multirow{2}{*}{Down}      & Protect Times     &   15.36   &  49.30    &      37.36    &     59.83&0.28&62.65&10.2&66.00   \\ \cline{3-11}
                         &                            & Comparison($\%$)                                   & \multicolumn{2}{c}{68.84} & \multicolumn{2}{c}{37.55}& \multicolumn{2}{c}{99.55} & \multicolumn{2}{c}{54.55}\\ 
\hline
\multirow{4}{*}{8}       & \multirow{2}{*}{Up}        & Protect Times      &     14.02   &  23.61    &  23.20  &  49.19&0.34&29.5&4.7&37.00   \\ \cline{3-11}
                         &                            & Comparison($\%$)                                   & \multicolumn{2}{c}{39.55} & \multicolumn{2}{c}{51.99}& \multicolumn{2}{c}{$\underline{98.84}$} & \multicolumn{2}{c}{81.08} \\ \cline{2-11}
                         & \multirow{2}{*}{Down}      & Protect Times   &  17.27   &  27.25   &  34.81       &  70.41&0.14&32.77&3.4&83.60   \\ \cline{3-11}
                         &                            & Comparison($\%$)                                   & \multicolumn{2}{c}{36.62} & \multicolumn{2}{c}{50.42}& \multicolumn{2}{c}{99.57} & \multicolumn{2}{c}{95.93}\\ 
\hline
\end{tabular}
\end{center}
\vspace{-0.45cm}
\end{table*}

\begin{figure*}
\begin{minipage}{0.2\linewidth}
\centerline{\includegraphics[width=4.5cm]{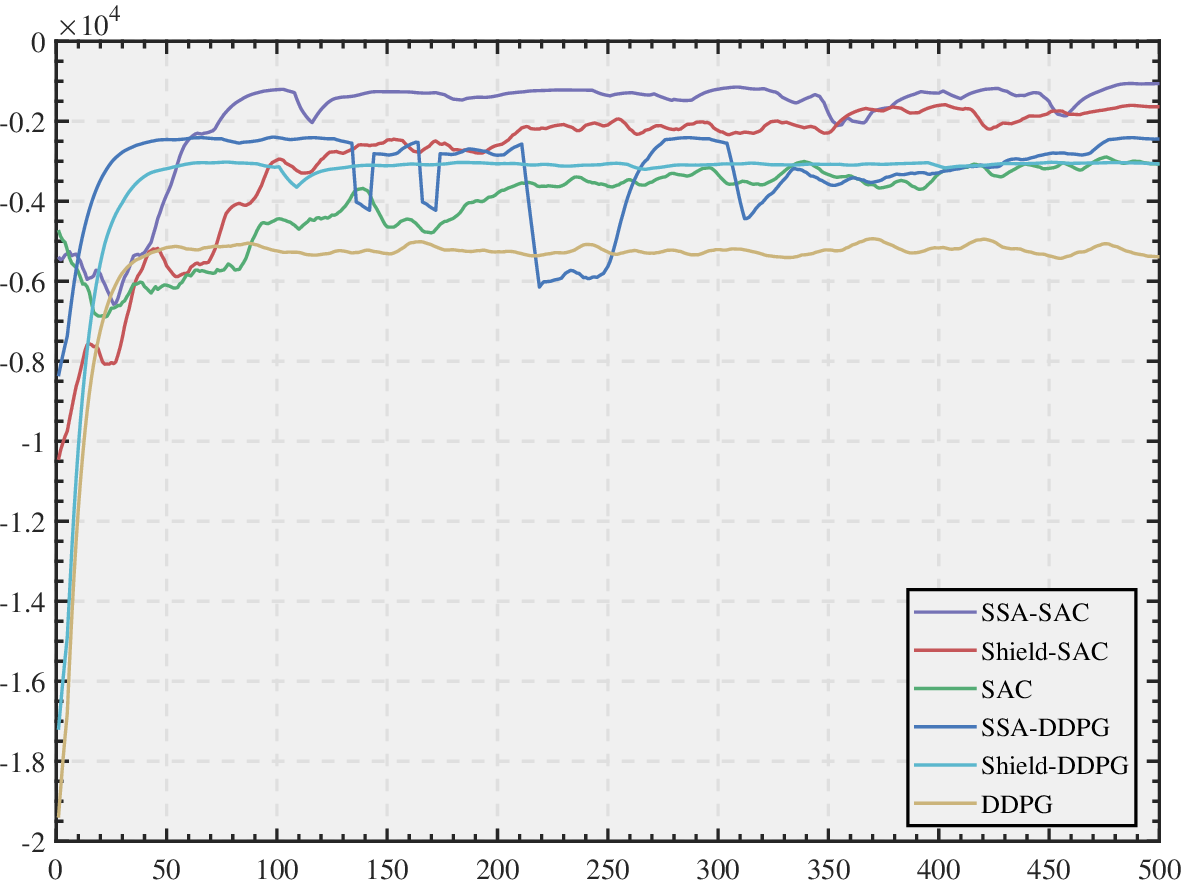}}
\centerline{(a) Sec1}
\end{minipage}
\hfill
\begin{minipage}{0.2\linewidth}
\centerline{\includegraphics[width=4.5cm]{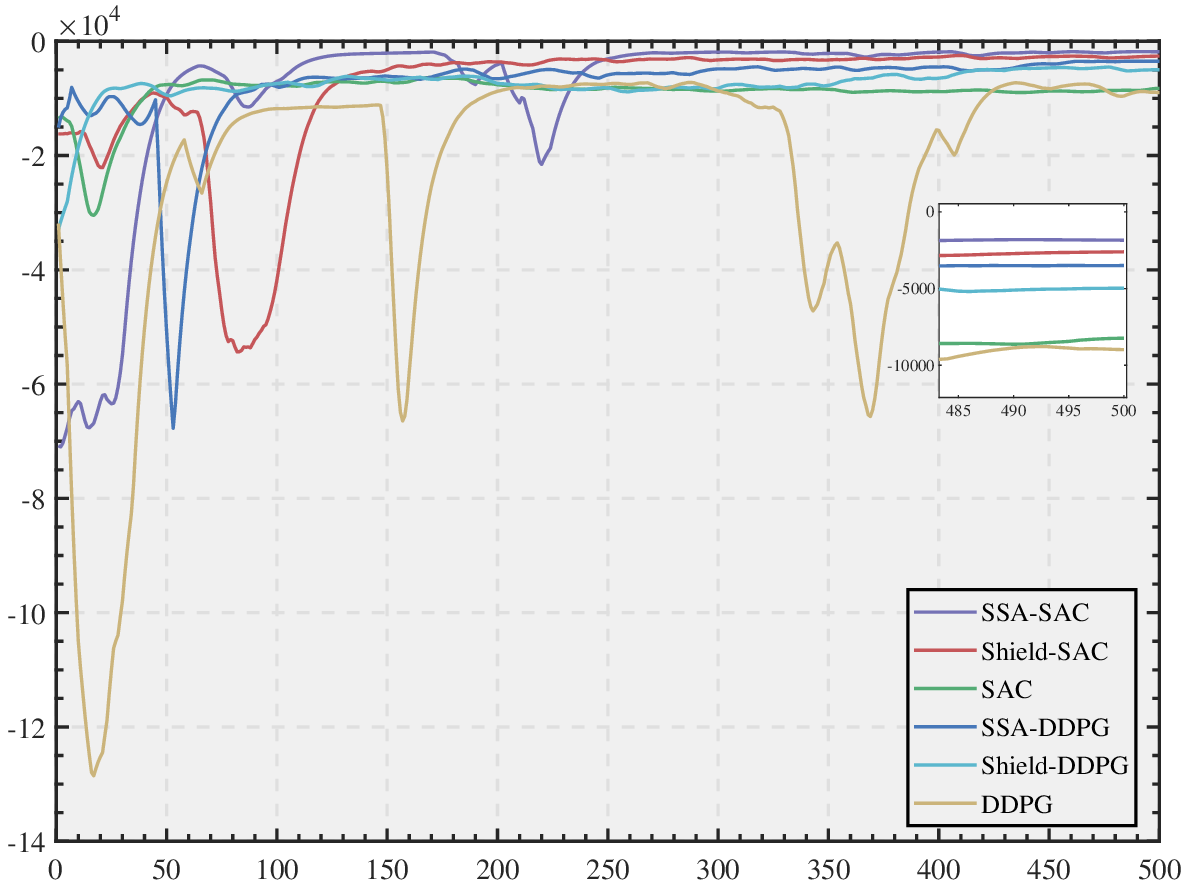}}
\centerline{(b) Sec2}
\end{minipage}
\hfill
\begin{minipage}{0.2\linewidth}
\centerline{\includegraphics[width=4.5cm]{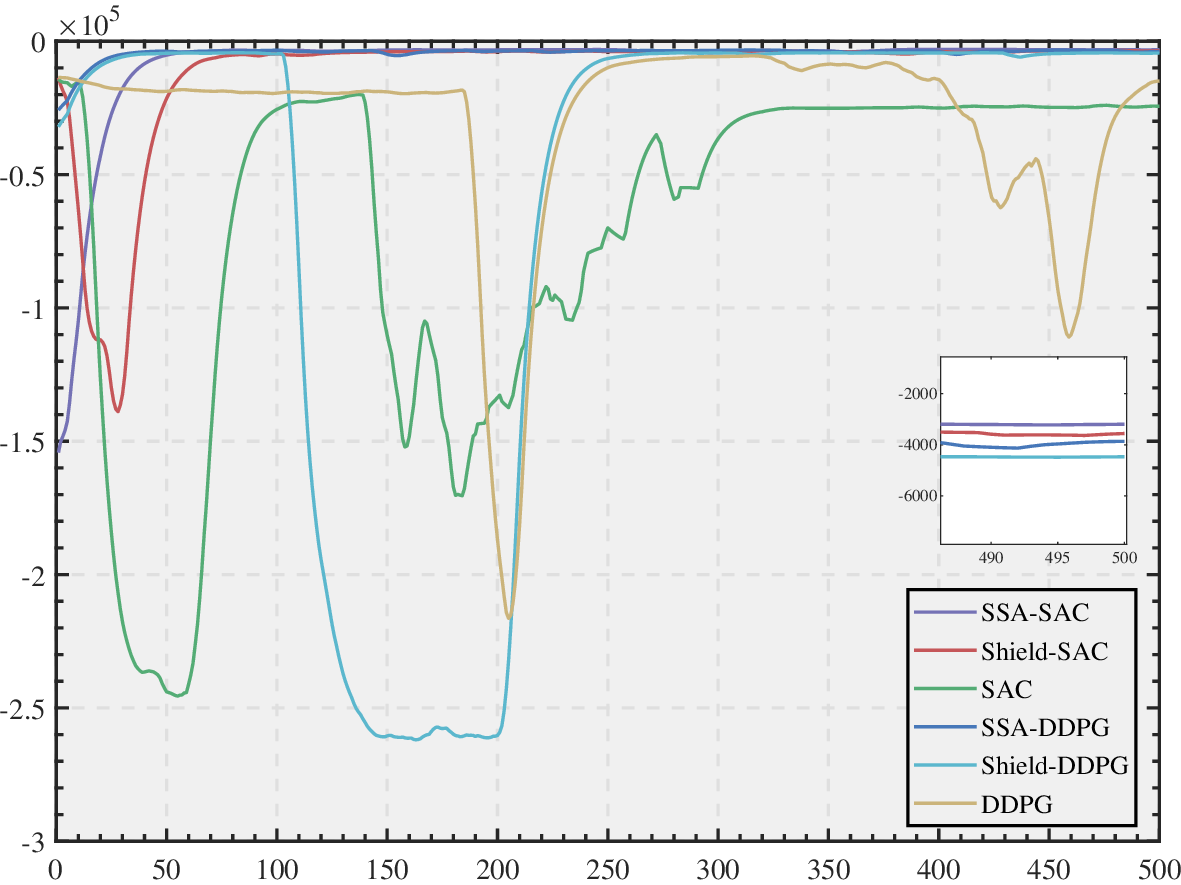}}
\centerline{(c) Sec3}
\end{minipage}
\hfill
\begin{minipage}{0.2\linewidth}
\centerline{\includegraphics[width=4.5cm]{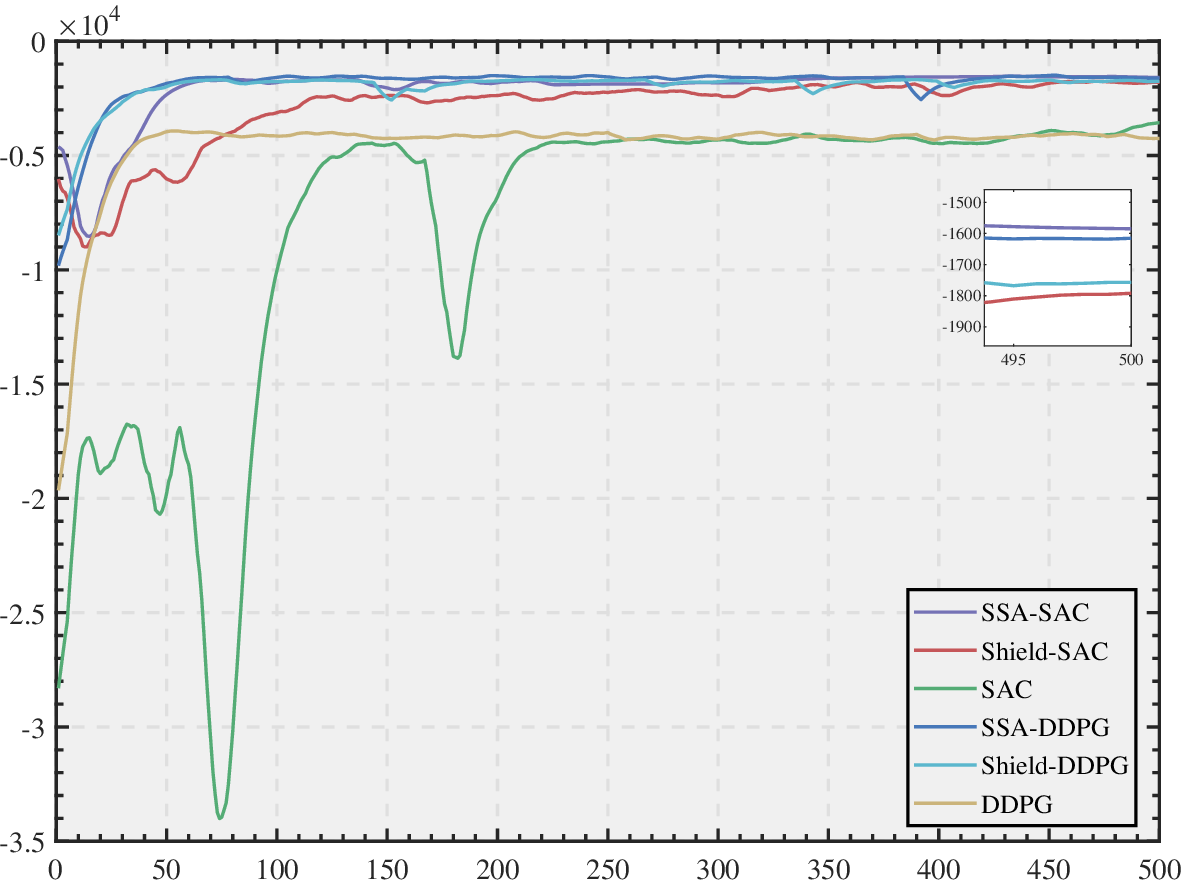}}
\centerline{(d) Sec4}
\end{minipage}
\hfill
\begin{minipage}{0.2\linewidth}
\centerline{\includegraphics[width=4.5cm]{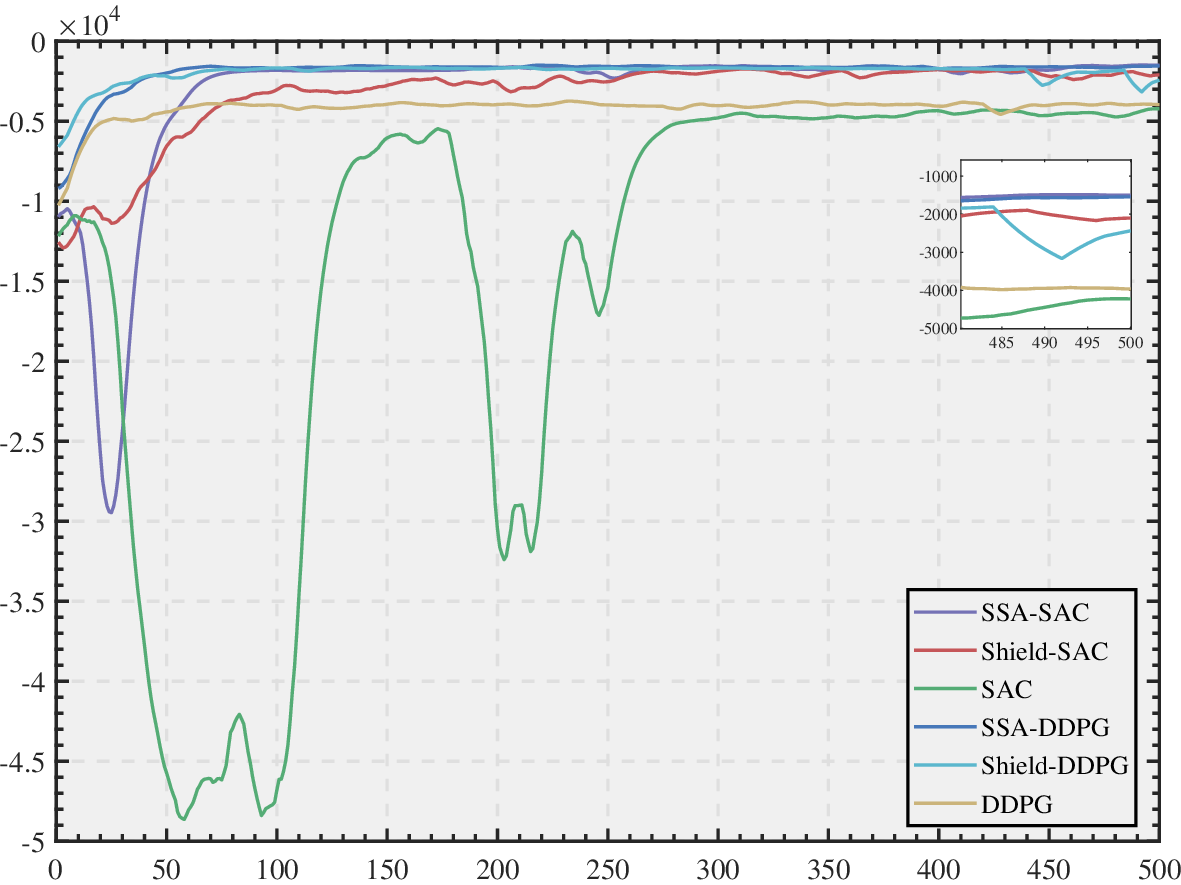}}
\centerline{(e) Sec5}
\end{minipage}
\hfill
\begin{minipage}{0.2\linewidth}
\centerline{\includegraphics[width=4.5cm]{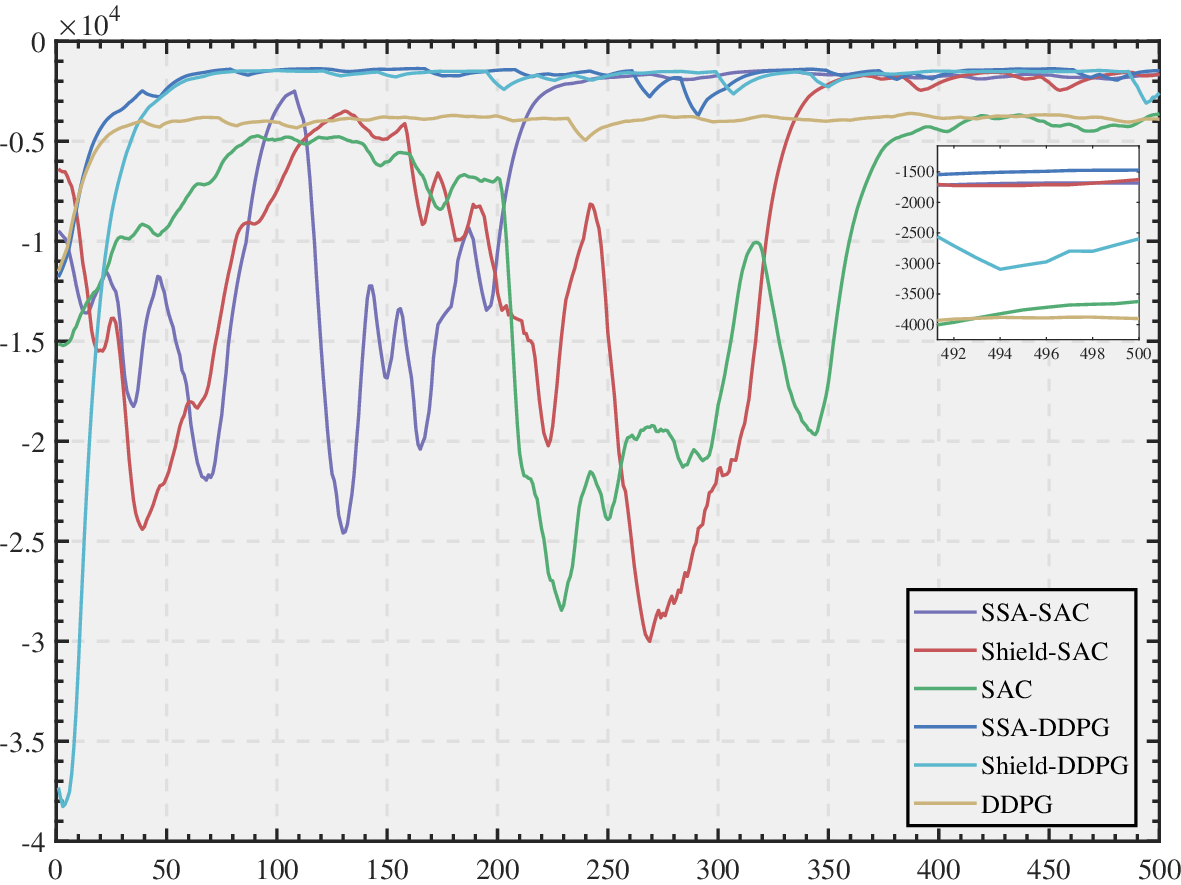}}
\centerline{(f) Sec6}
\end{minipage}
\hfill
\begin{minipage}{0.2\linewidth}
\centerline{\includegraphics[width=4.5cm]{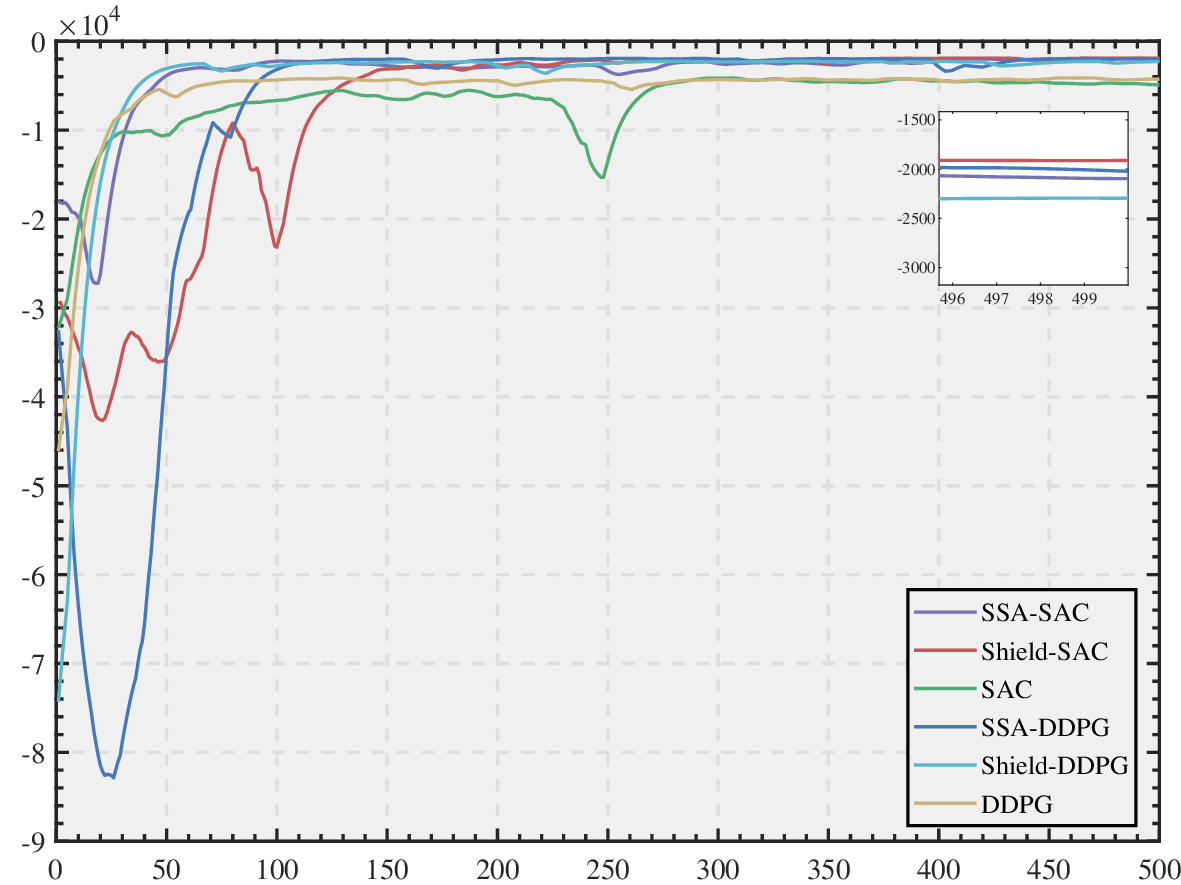}}
\centerline{(g) Sec7}
\end{minipage}
\hfill
\begin{minipage}{0.2\linewidth}
\centerline{\includegraphics[width=4.5cm]{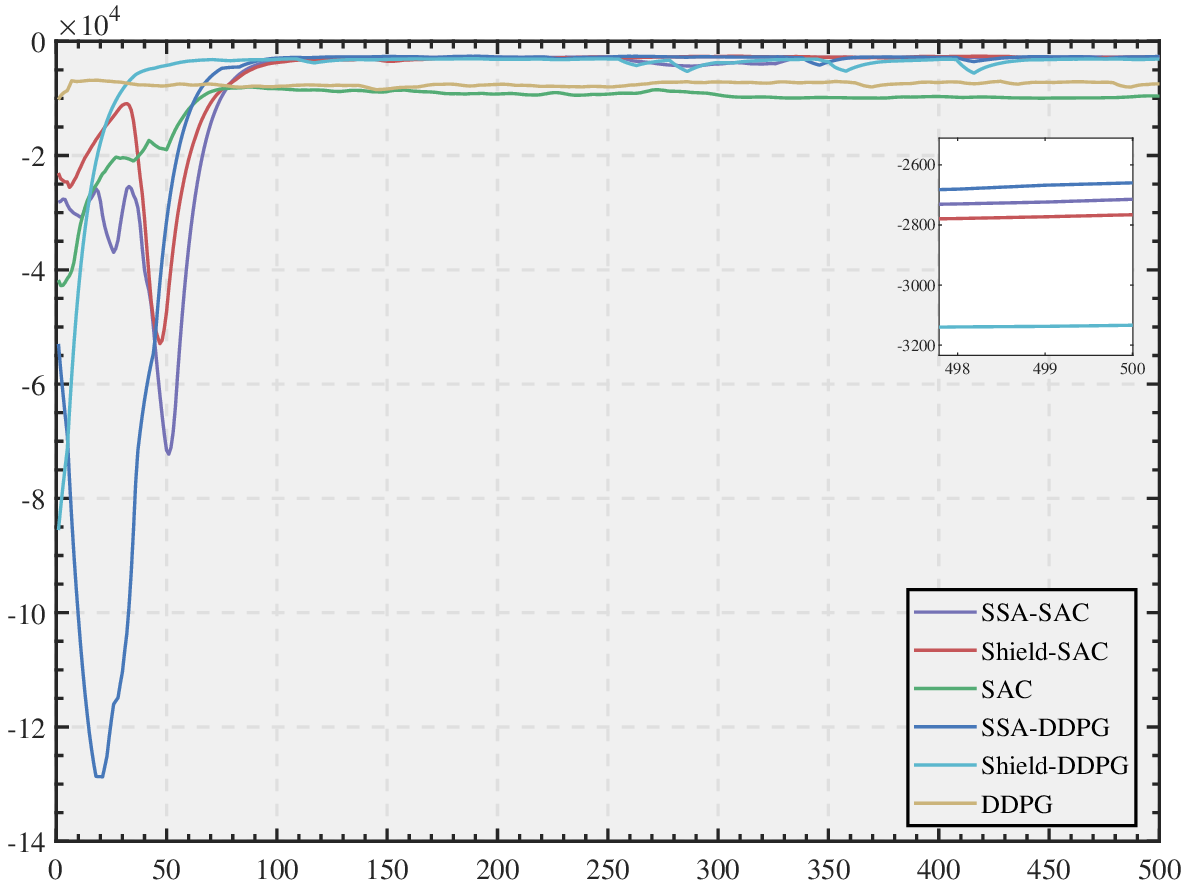}}
\centerline{(h) Sec8}
\end{minipage}
\hfill
\begin{minipage}{0.2\linewidth}
\centerline{\includegraphics[width=4.5cm]{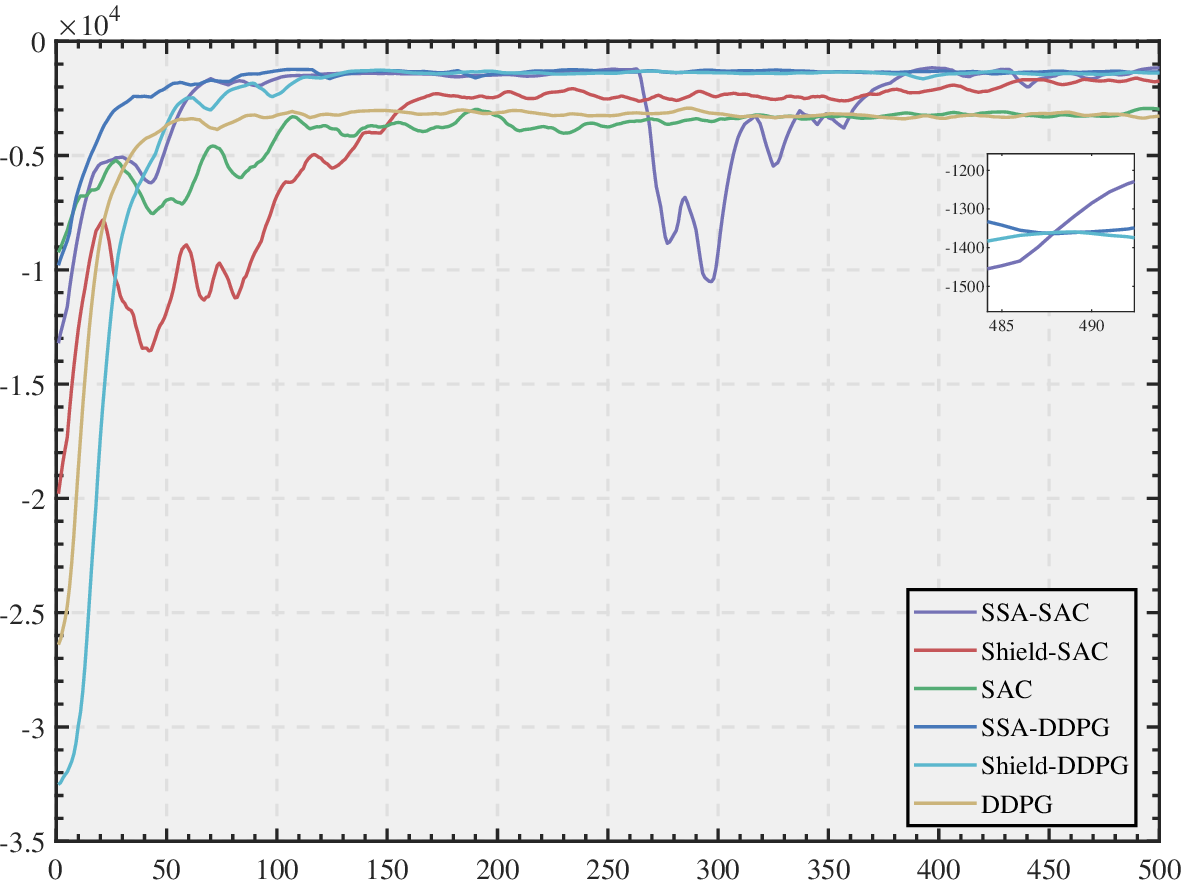}}
\centerline{(i) Sec1}
\end{minipage}
\hfill
\begin{minipage}{0.2\linewidth}
\centerline{\includegraphics[width=4.5cm]{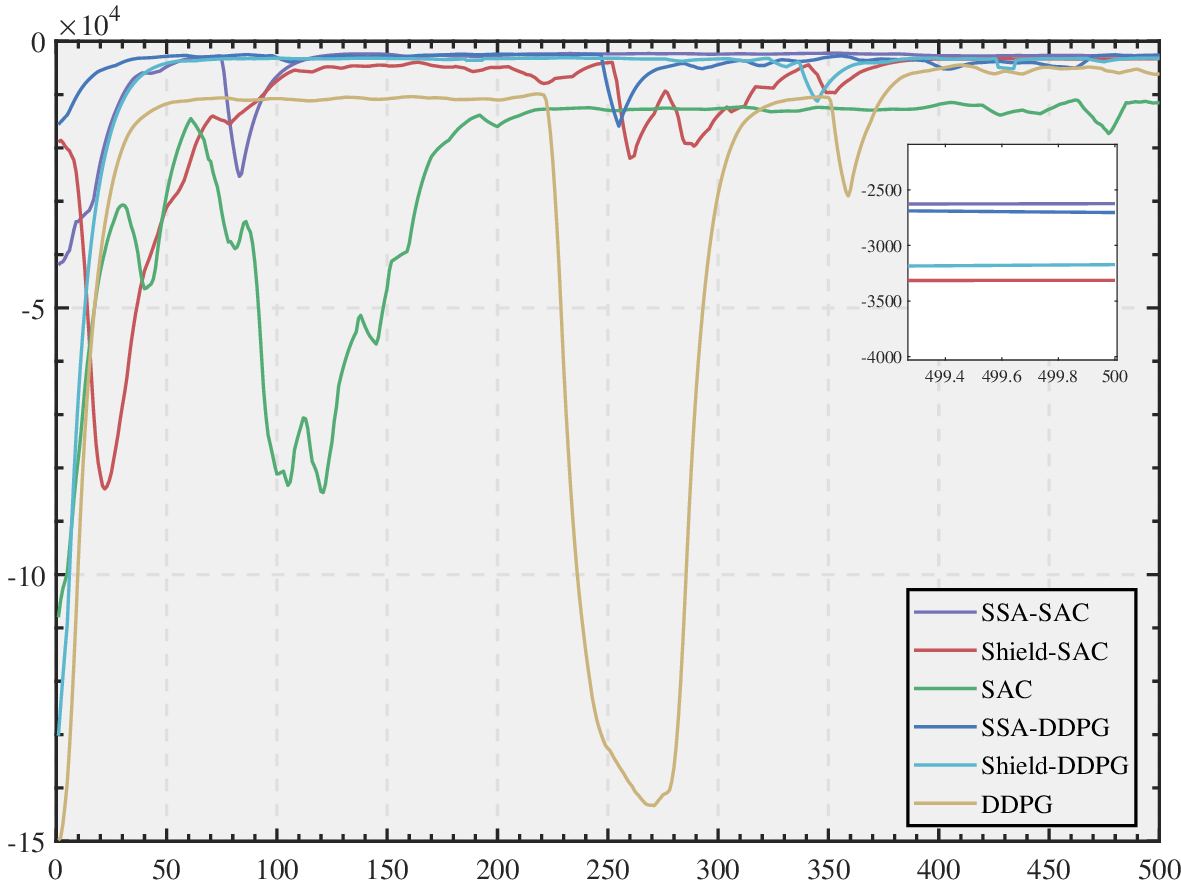}}
\centerline{(j) Sec2}
\end{minipage}
\hfill
\begin{minipage}{0.2\linewidth}
\centerline{\includegraphics[width=4.5cm]{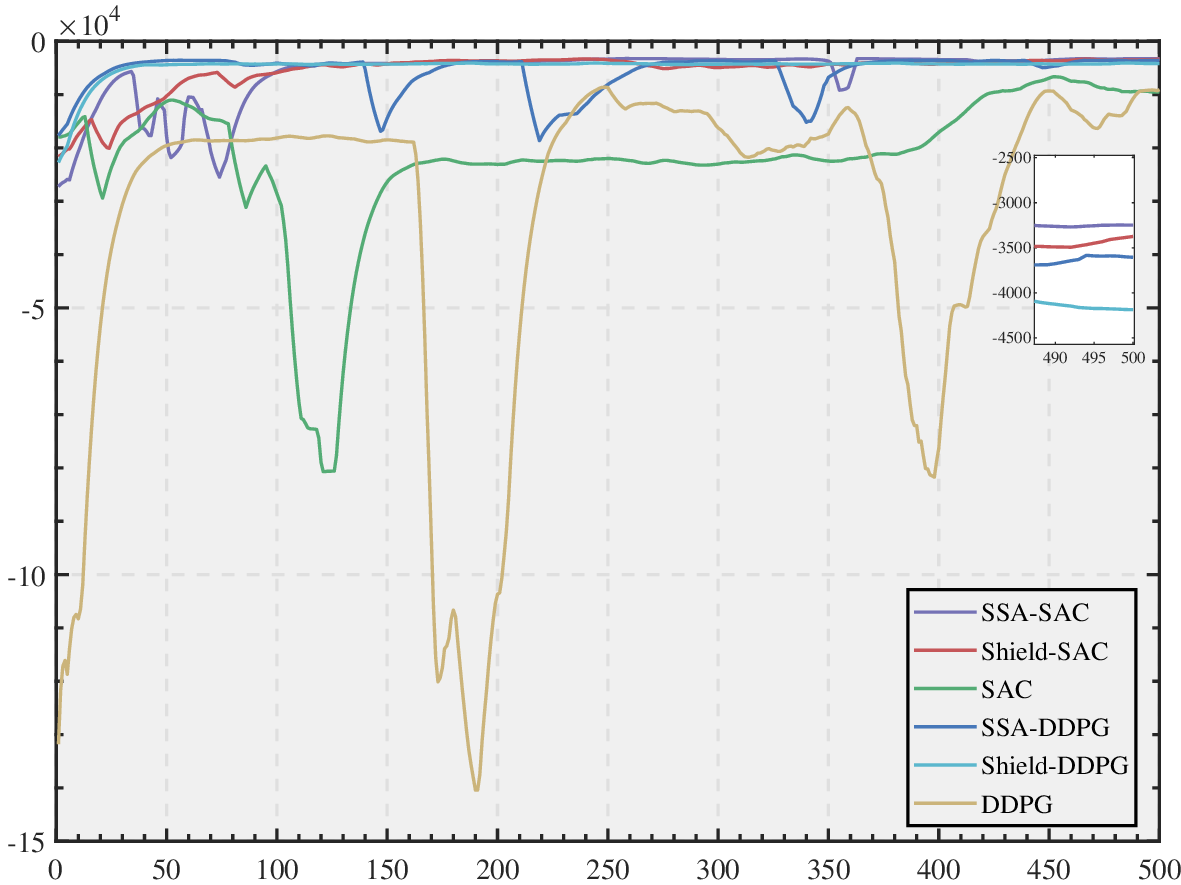}}
\centerline{(k) Sec3}
\end{minipage}
\hfill
\begin{minipage}{0.2\linewidth}
\centerline{\includegraphics[width=4.5cm]{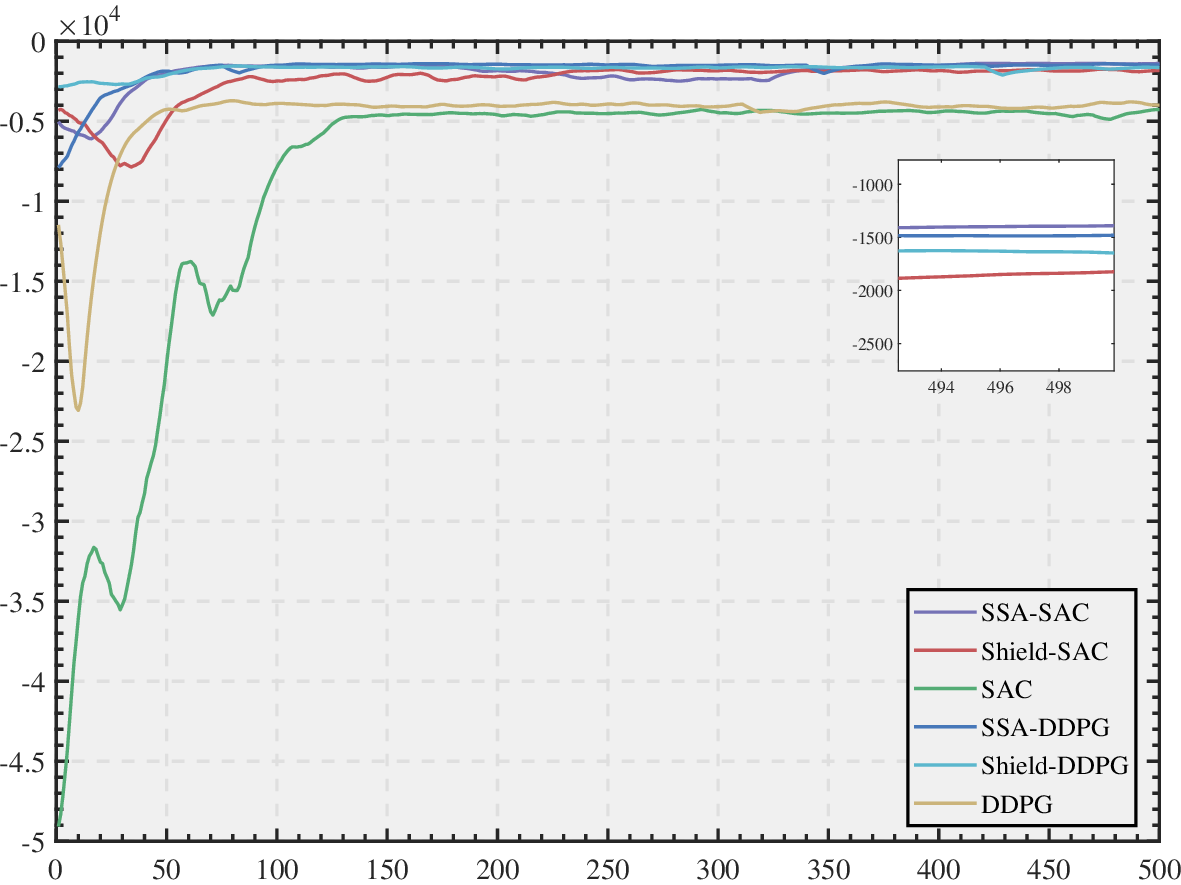}}
\centerline{(l) Sec4}
\end{minipage}
\hfill
\begin{minipage}{0.2\linewidth}
\centerline{\includegraphics[width=4.5cm]{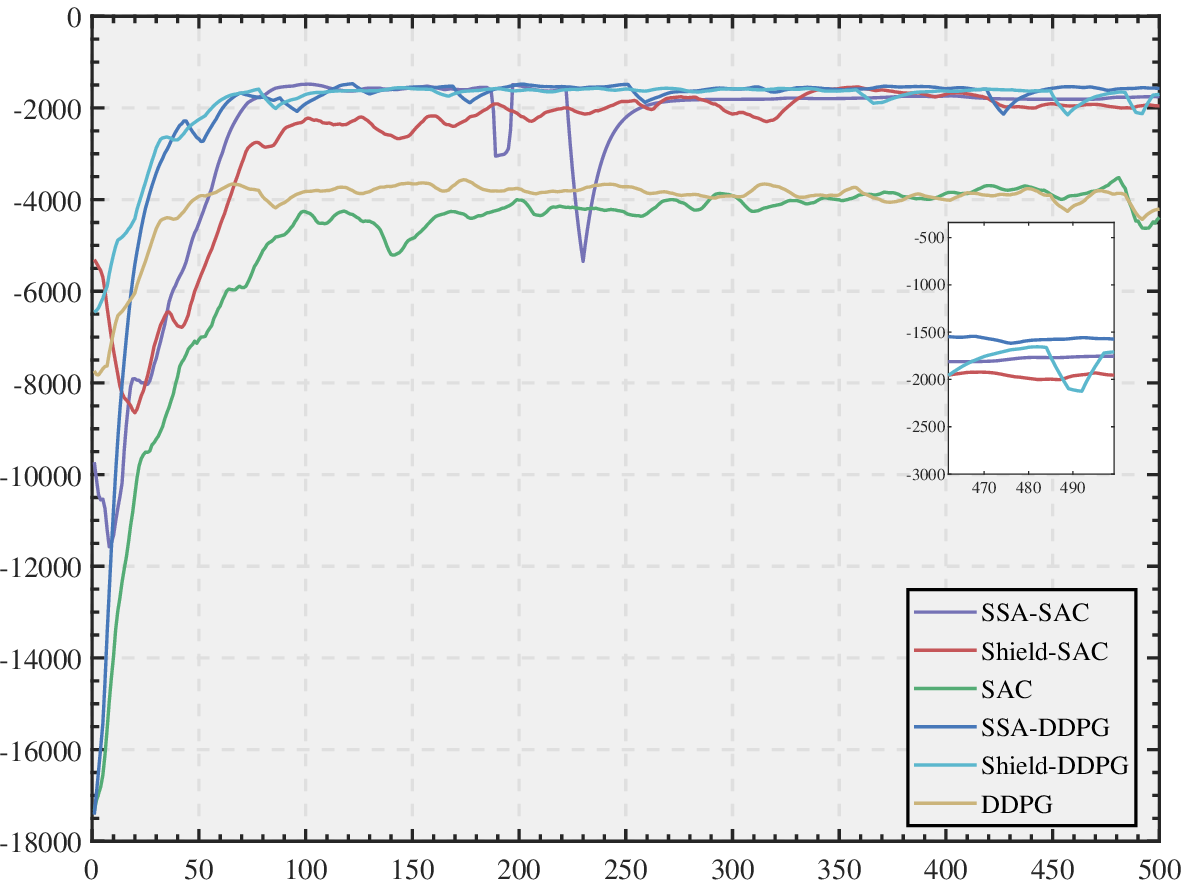}}
\centerline{(m) Sec5}
\end{minipage}
\hfill
\begin{minipage}{0.2\linewidth}
\centerline{\includegraphics[width=4.5cm]{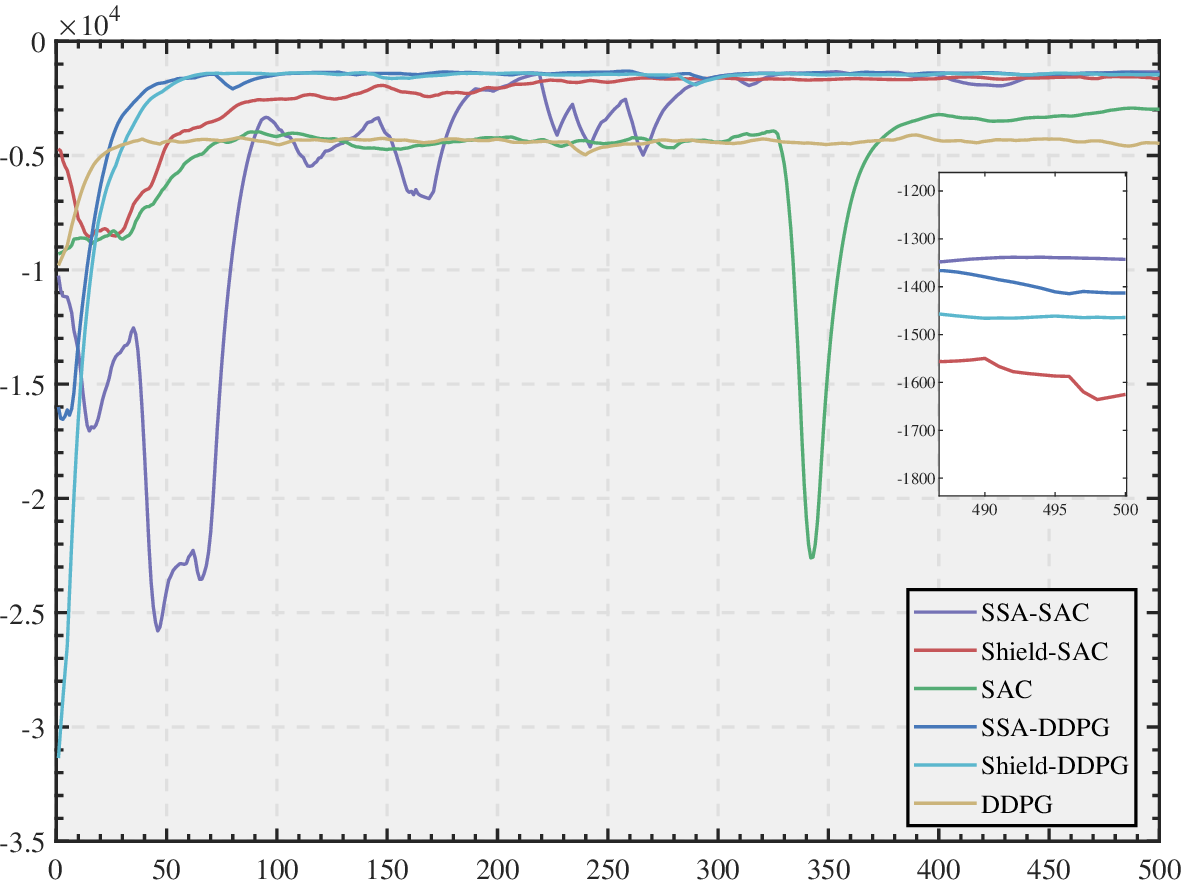}}
\centerline{(n) Sec6}
\end{minipage}
\hfill
\begin{minipage}{0.2\linewidth}
\centerline{\includegraphics[width=4.5cm]{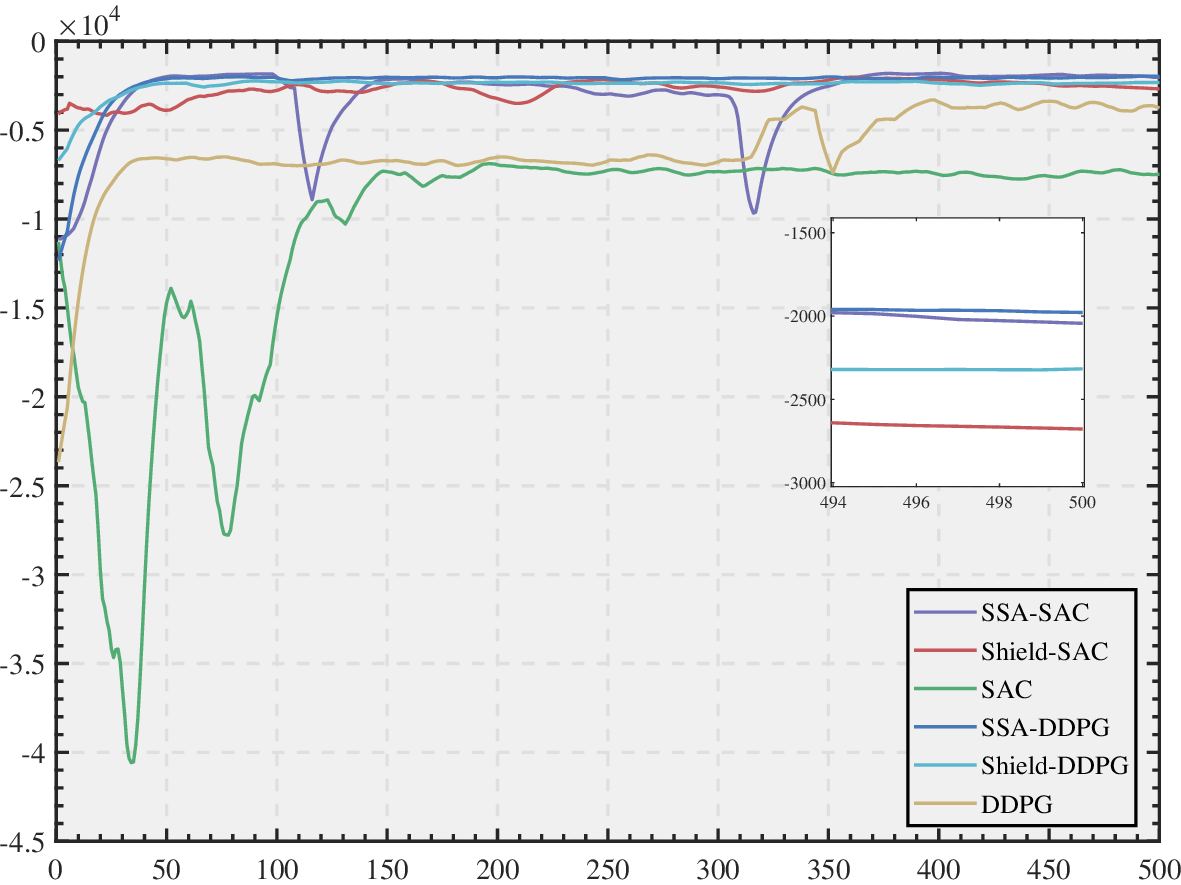}}
\centerline{(o) Sec7}
\end{minipage}
\hfill
\begin{minipage}{0.2\linewidth}
\centerline{\includegraphics[width=4.5cm]{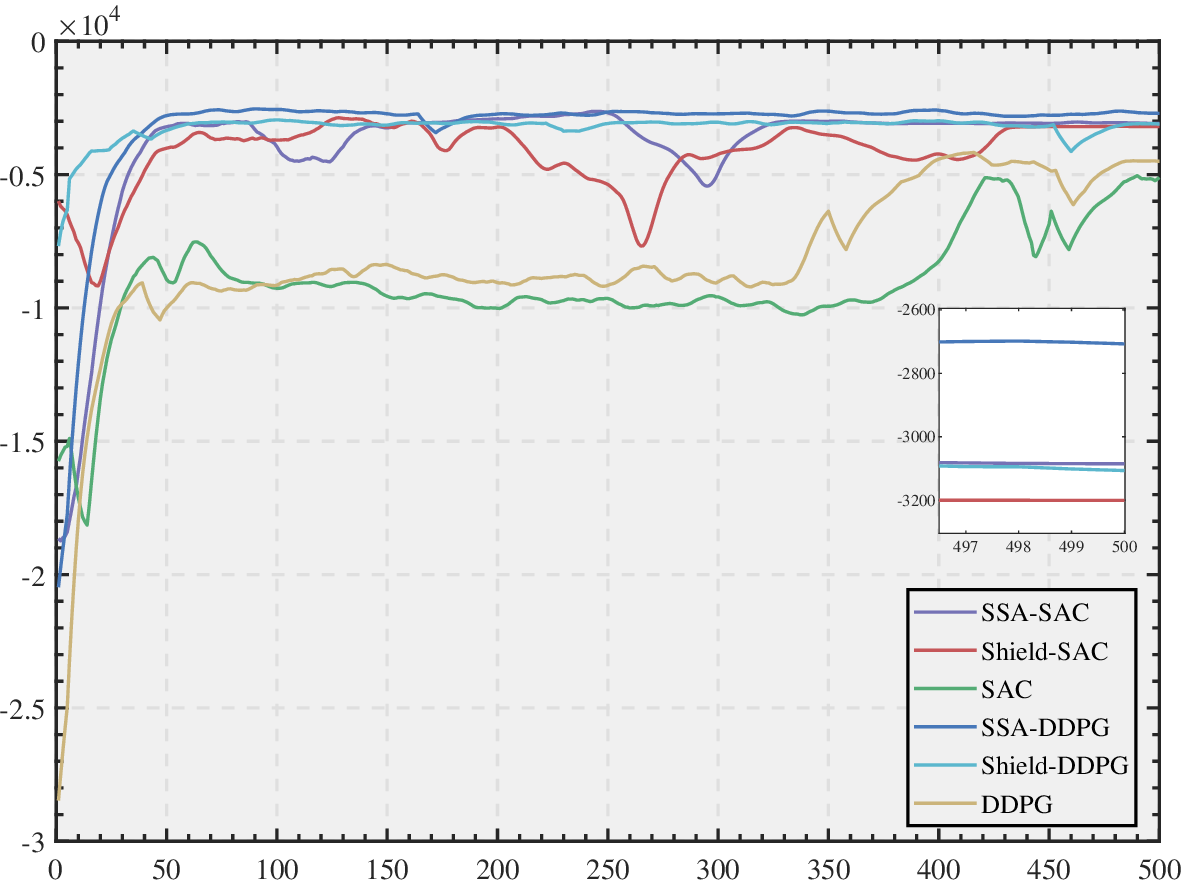}}
\centerline{(p) Sec8}
\end{minipage}
\hfill
\caption{Reward curves. The first tow rows are the reward curves of up direction and the last two rows are of down direction.}
\label{fig:uprewardcurve}
\end{figure*}

\begin{figure*}
\begin{minipage}{0.23\linewidth}
\centerline{\includegraphics[width=4.5cm]{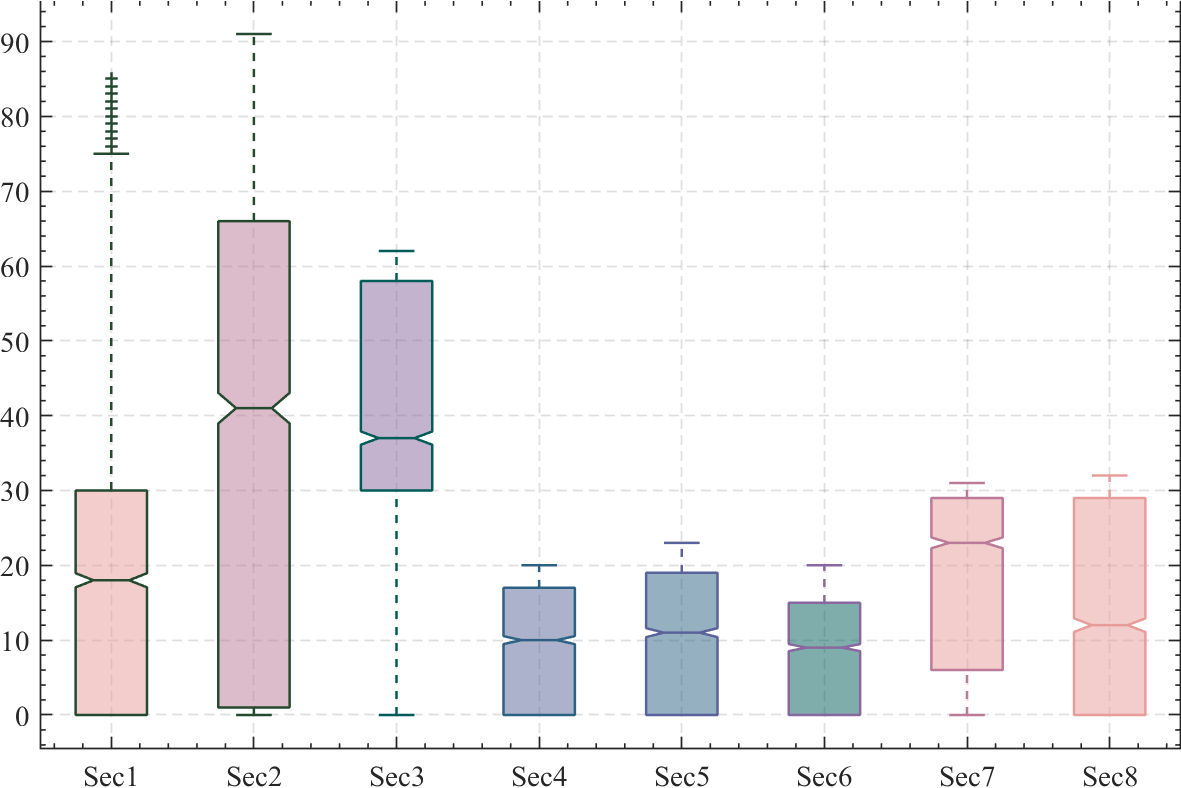}}
\centerline{(a) Up Direction (Training)}
\end{minipage}
\hfill
\begin{minipage}{0.23\linewidth}
\centerline{\includegraphics[width=4.5cm]{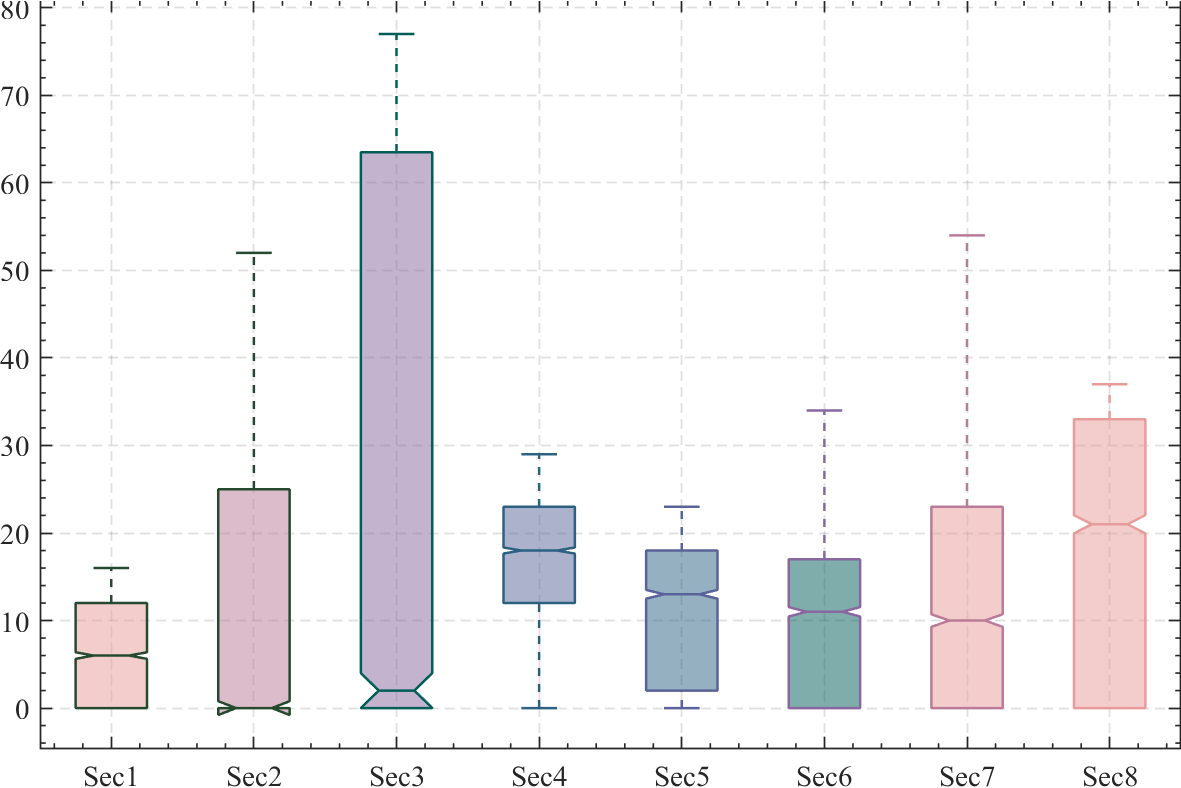}}
\centerline{(b) Down Direction (Training)}
\end{minipage}
\hfill
\begin{minipage}{0.23\linewidth}
\centerline{\includegraphics[width=4.5cm]{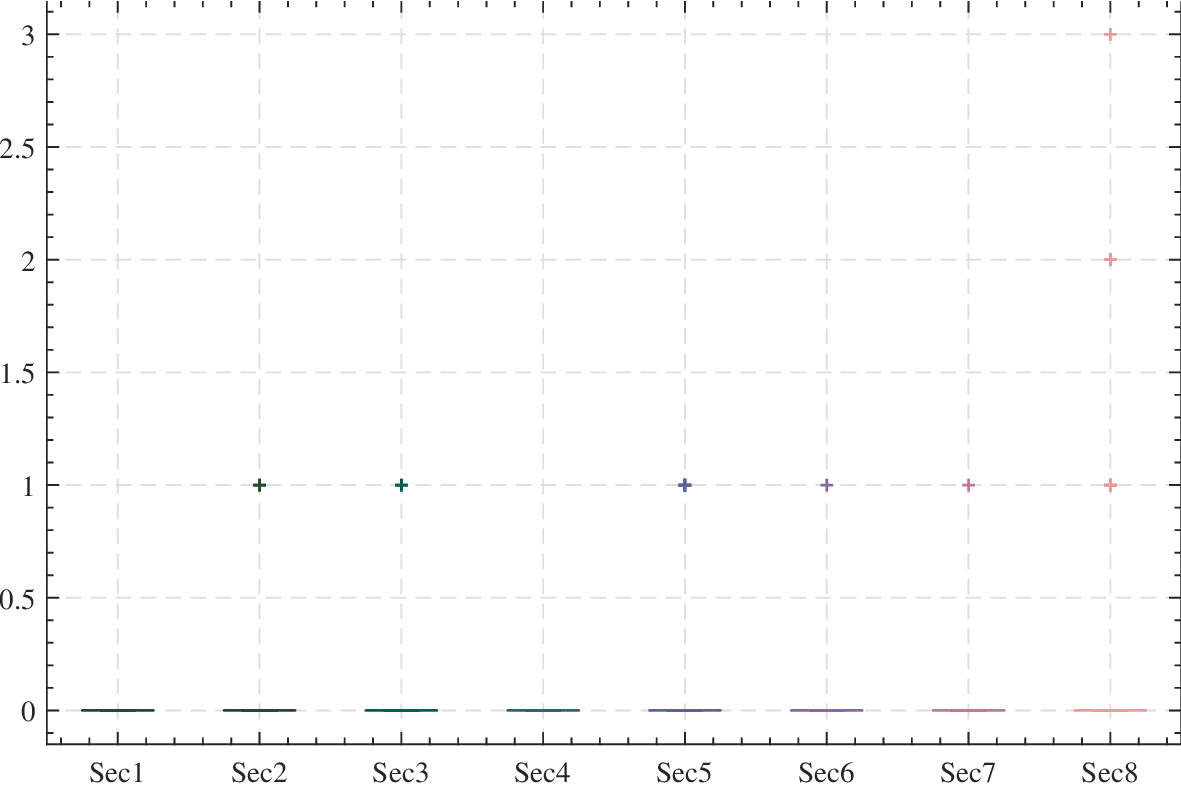}}
\centerline{(c) Up Direction (Execution)}
\end{minipage}
\hfill
\begin{minipage}{0.23\linewidth}
\centerline{\includegraphics[width=4.5cm]{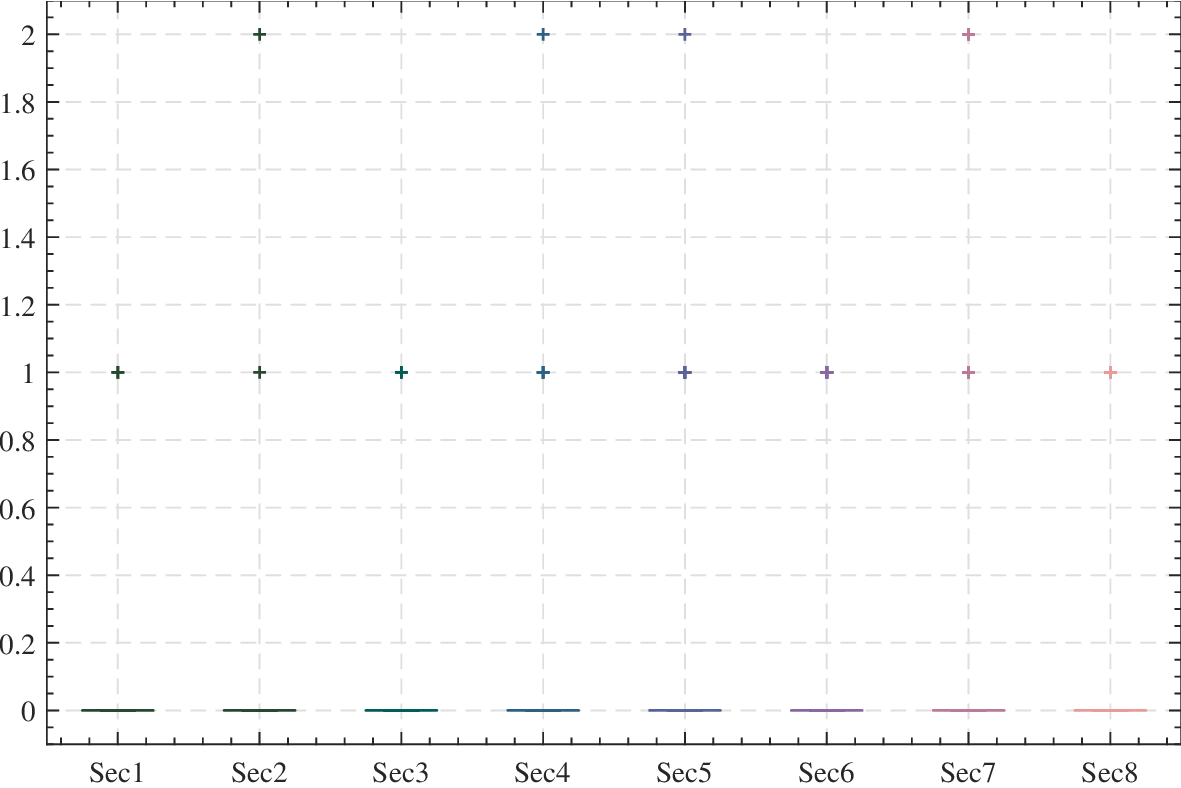}}
\centerline{(d) Down Direction (Execution)}
\end{minipage}
\hfill
\begin{minipage}{0.23\linewidth}
\centerline{\includegraphics[width=4.5cm]{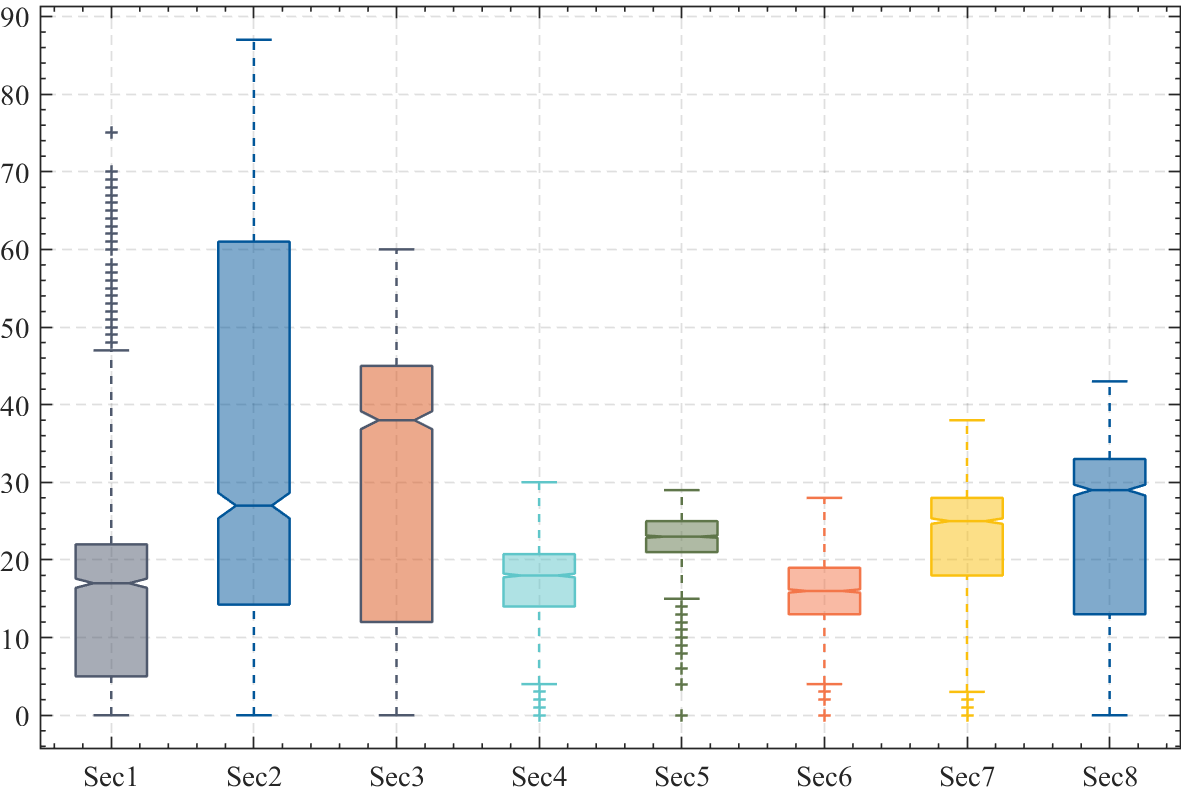}}
\centerline{(e) Up Direction (Training)}
\end{minipage}
\hfill
\begin{minipage}{0.23\linewidth}
\centerline{\includegraphics[width=4.5cm]{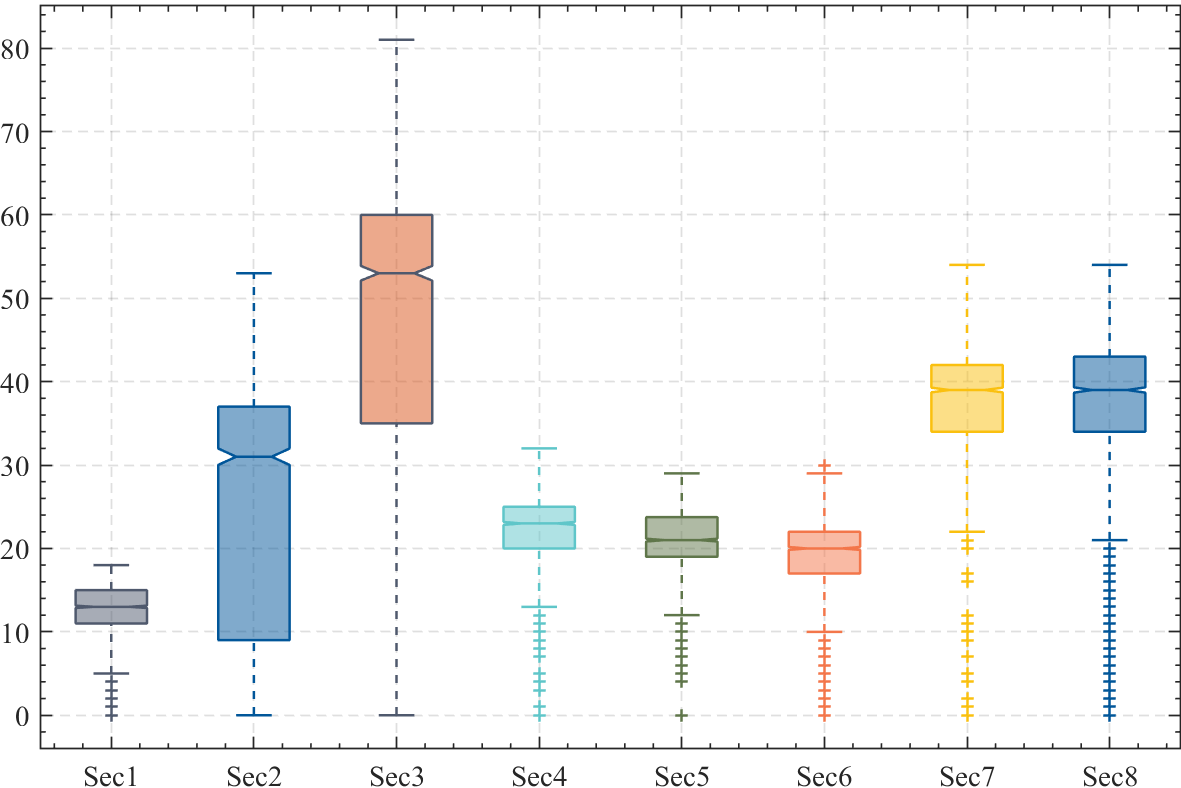}}
\centerline{(f) Down Direction (Training)}
\end{minipage}
\hfill
\begin{minipage}{0.23\linewidth}
\centerline{\includegraphics[width=4.5cm]{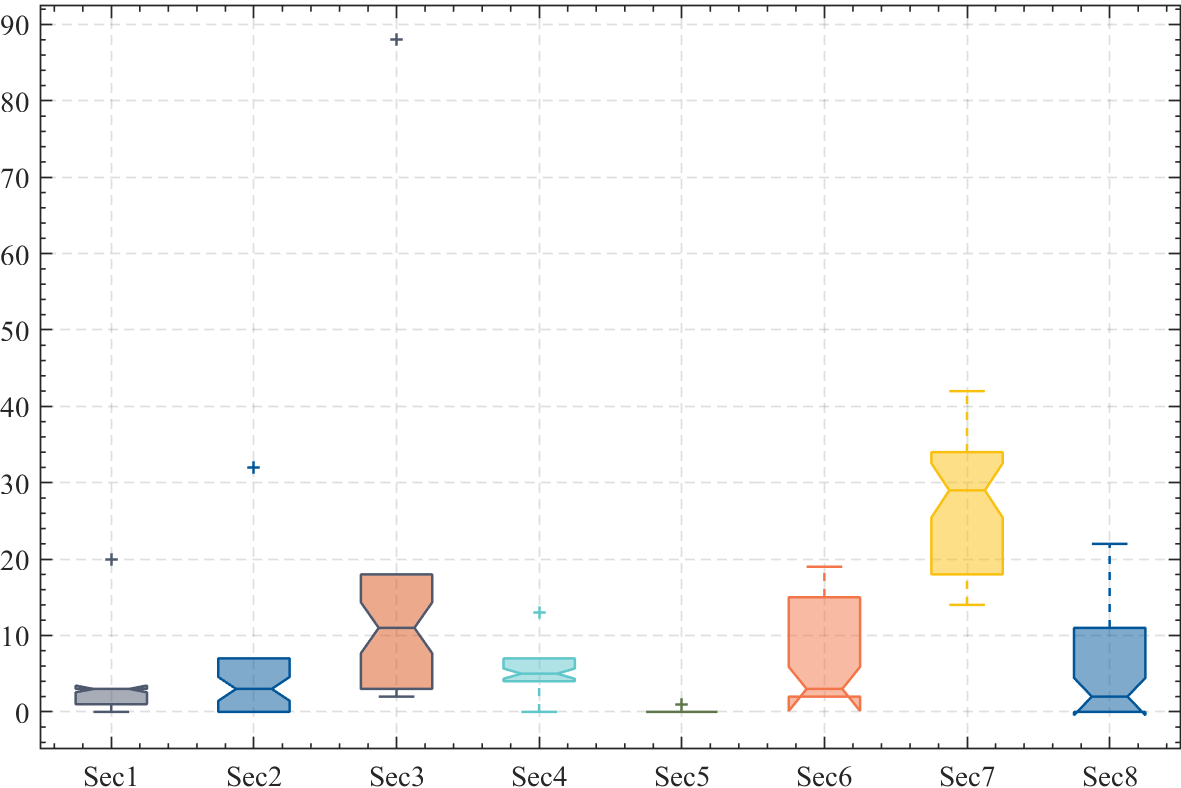}}
\centerline{(g) Up Direction (Execution)}
\end{minipage}
\hfill
\begin{minipage}{0.23\linewidth}
\centerline{\includegraphics[width=4.5cm]{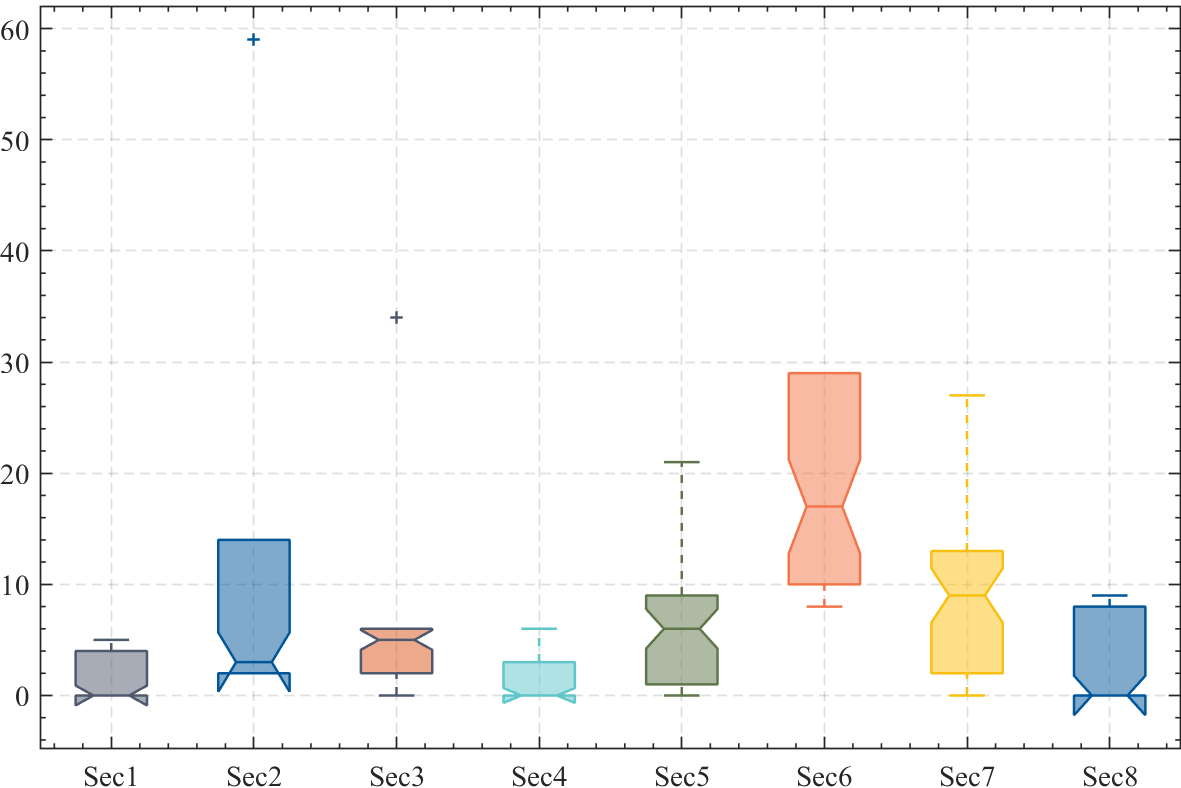}}
\centerline{(h) Down Direction (Execution)}
\end{minipage}
\hfill
\caption{Protect times. The first row is the result of SSA-SAC and the second row is the result of SSA-DDPG.}
\label{fig:protectcountsbox}
\end{figure*}

\begin{figure*}
\begin{minipage}{0.23\linewidth}
\centerline{\includegraphics[width=4.5cm]{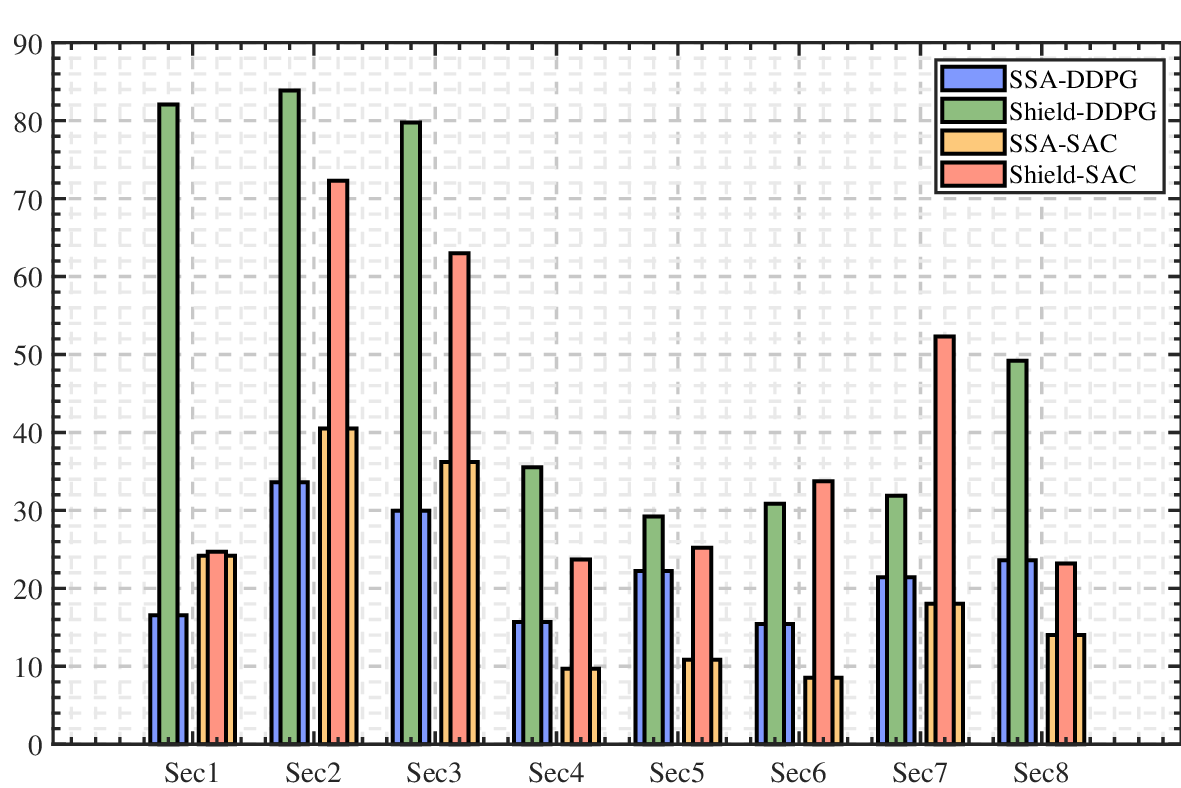}}
\centerline{(a) Up Direction (Training)}
\end{minipage}
\hfill
\begin{minipage}{0.23\linewidth}
\centerline{\includegraphics[width=4.5cm]{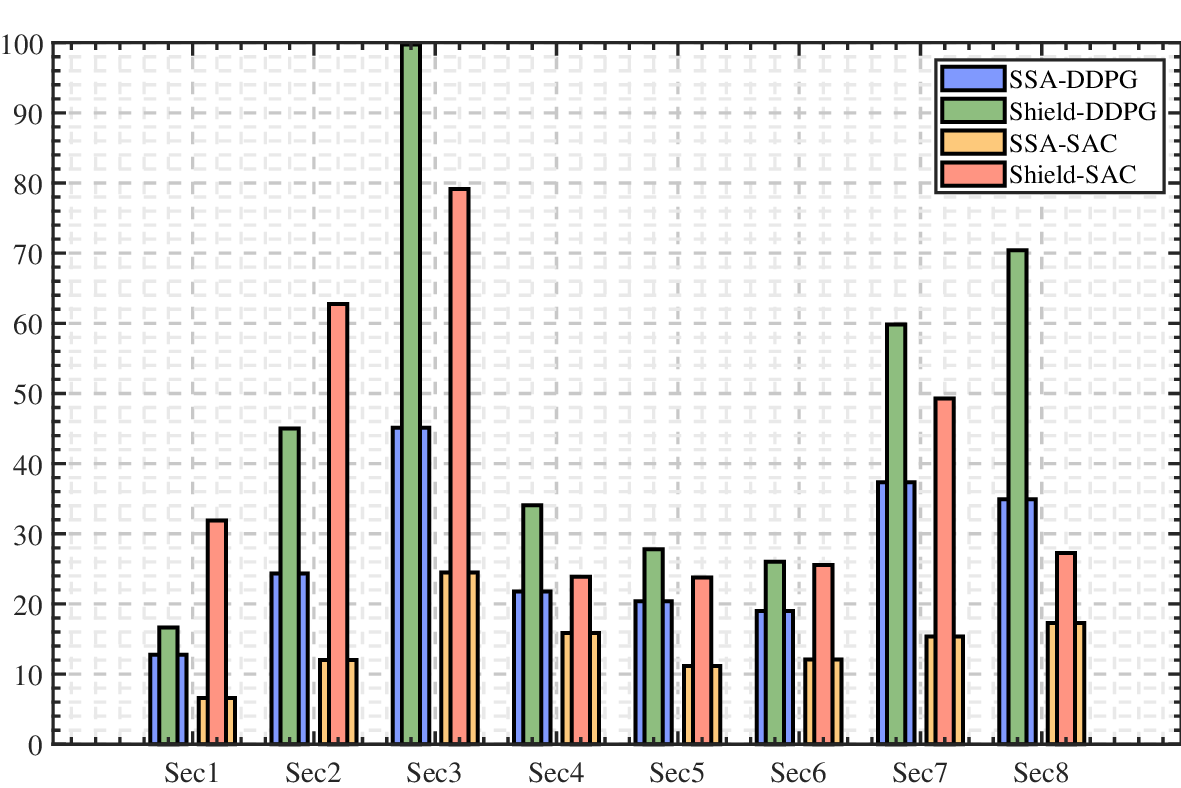}}
\centerline{(b) Down Direction (Training)}
\end{minipage}
\hfill
\begin{minipage}{0.23\linewidth}
\centerline{\includegraphics[width=4.5cm]{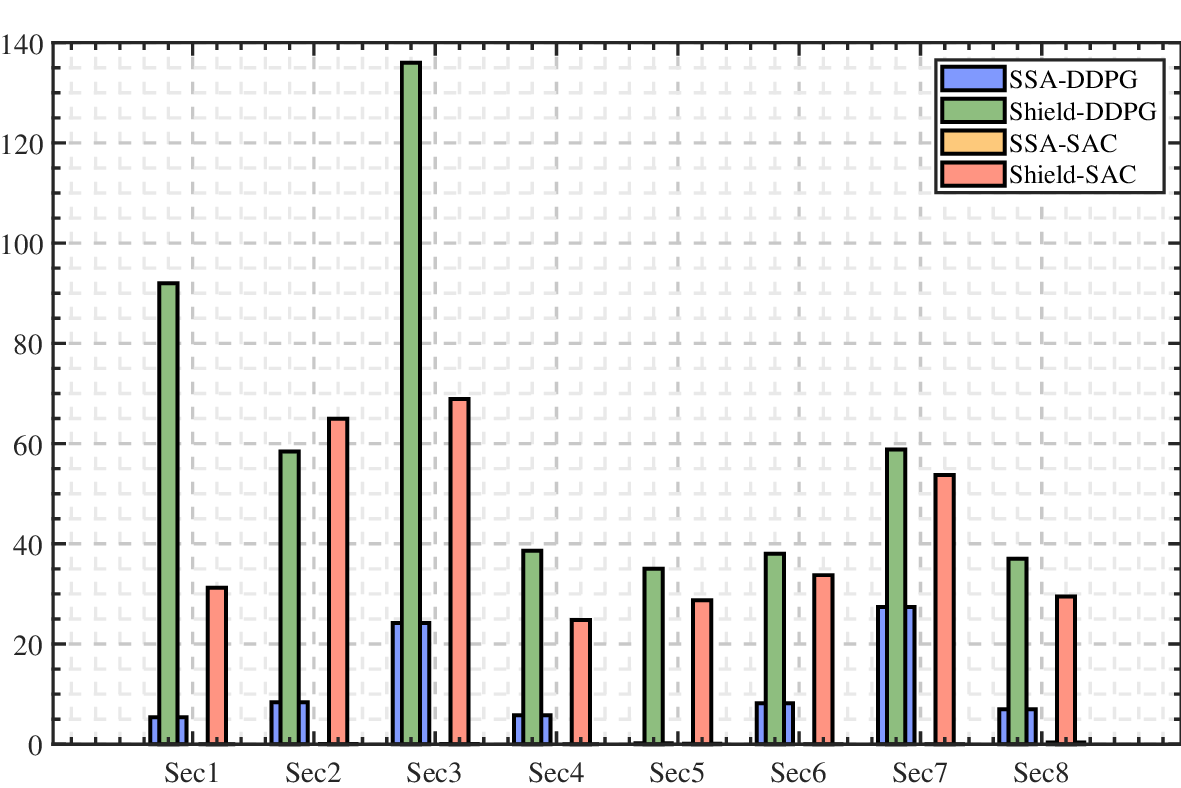}}
\centerline{(c) Up Direction (Execution)}
\end{minipage}
\hfill
\begin{minipage}{0.23\linewidth}
\centerline{\includegraphics[width=4.5cm]{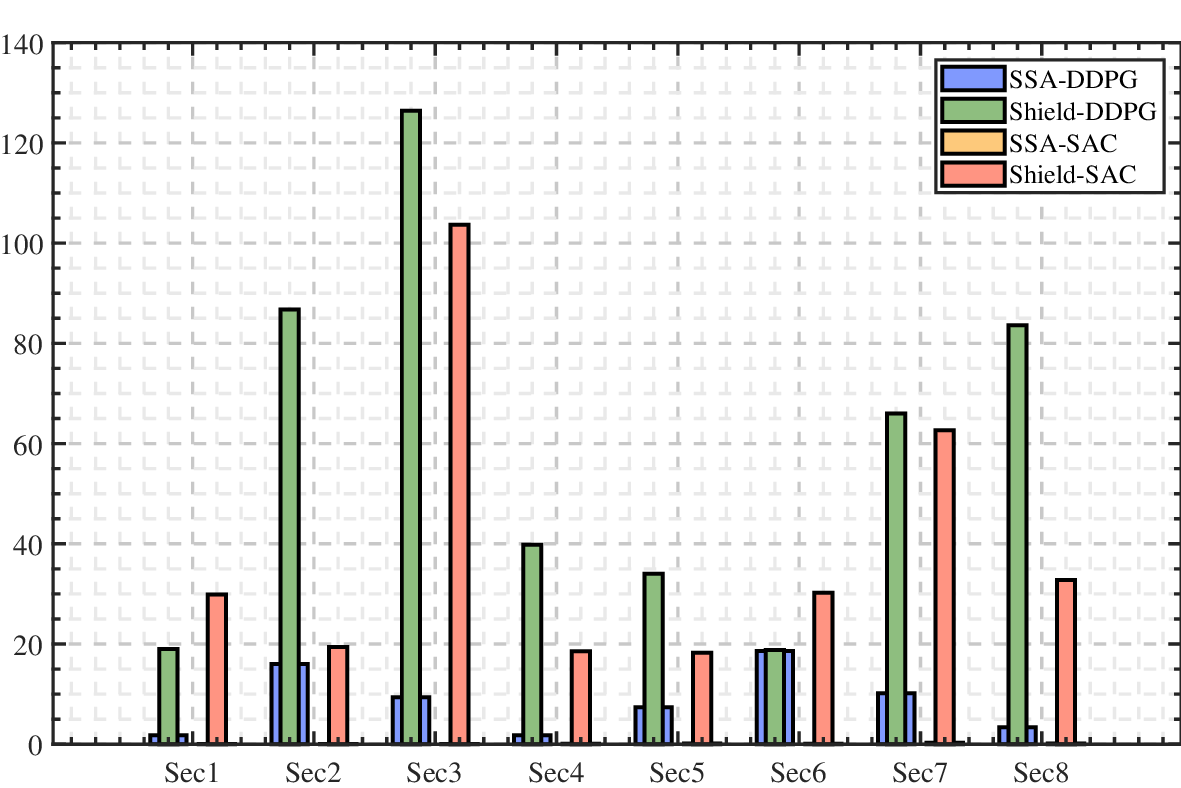}}
\centerline{(d) Down Direction (Execution)}
\end{minipage}
\hfill
\caption{Protect times. Comparison between SSA-DRL and Shield-DRL.}
\label{fig:protectcounts}
\end{figure*}

The distribution of calculation time in execution are shown in Fig.~\ref{fig:time_distribution}. The calculation time consists of getting a action and state transition two parts. If the original action is checked safe by the Shield, getting a action time means the time of the neural network to complete the forward pass. If not, then it is the time of completing a searching process and output a final safe action. The transition time means the time of one step state transition which is simply denoted by $s_t,a_t,r_{t+1},s_{t+1}$. The X-axis in Fig.~\ref{fig:time_distribution} represents the simulation section, the Y-axis represents the range of the calculation time and the Z-axis represents the percentage of the corresponding calculation time range. For example, the bar at the northwest corner in Fig.~\ref{fig:time_distribution}(a) means the percentage of getting an action time in range 0-0.05s of Section1 up direction is 60$\%$-70$\%$. Then from Fig.~\ref{fig:time_distribution}, it can be get that most of the calculation time no matter getting action or state transition is less than 0.02s meanwhile 0-0.005s accounts the biggest proportion that satisfies the time requirement of the control cycle in real world which is usually 200ms.

Then, combine Fig.~\ref{fig:speedprofile}-Fig.~\ref{fig:time_distribution} with the detailed numerical results shown in Table.\ref{table4} and Table.\ref{tablecompare}, a conclusion can be drawn is that SSA-DRL can control the train complete the operation plan and reduce traction energy consumption without overspeed danger which we also say the proposed SSA-DRL ensures a safe control strategy.

\begin{figure*}
\begin{minipage}{0.23\linewidth}
\centerline{\includegraphics[width=4.5cm]{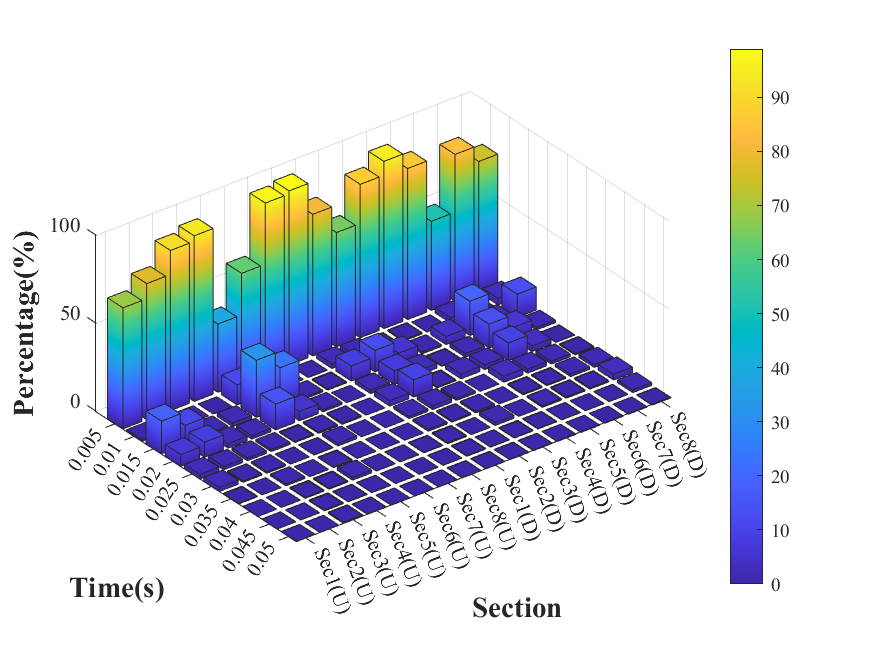}}
\centerline{(a) SSA-DDPG get action}
\end{minipage}
\hfill
\begin{minipage}{0.23\linewidth}
\centerline{\includegraphics[width=4.5cm]{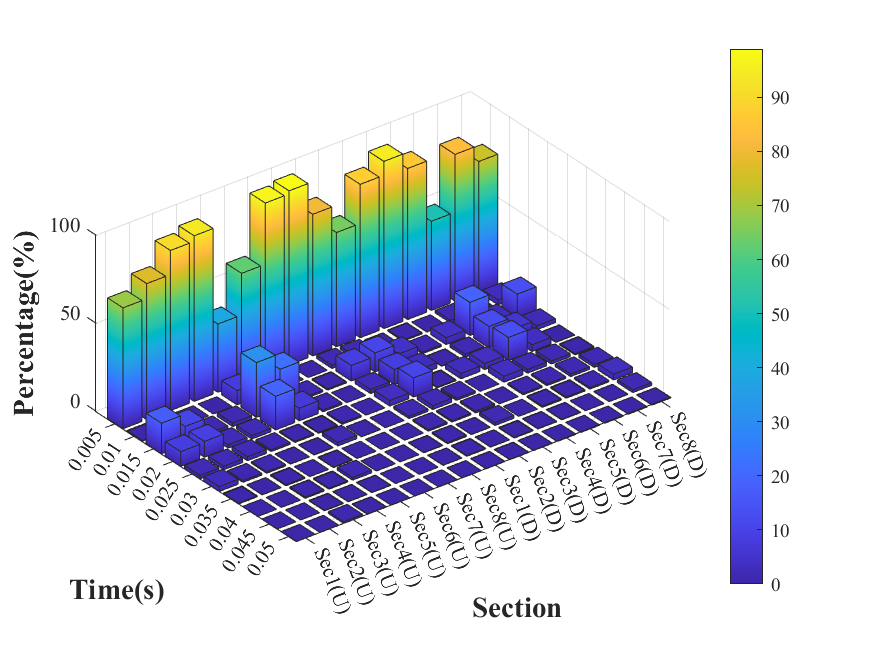}}
\centerline{(c) SSA-DDPG state transition}
\end{minipage}
\hfill
\begin{minipage}{0.23\linewidth}
\centerline{\includegraphics[width=4.5cm]{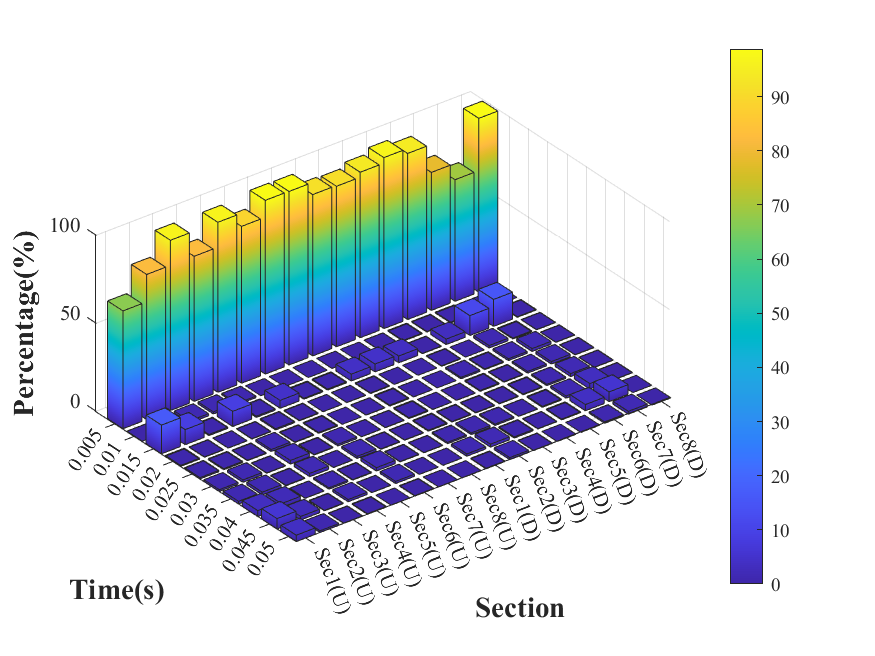}}
\centerline{(e) SSA-SAC get action}
\end{minipage}
\hfill
\begin{minipage}{0.23\linewidth}
\centerline{\includegraphics[width=4.5cm]{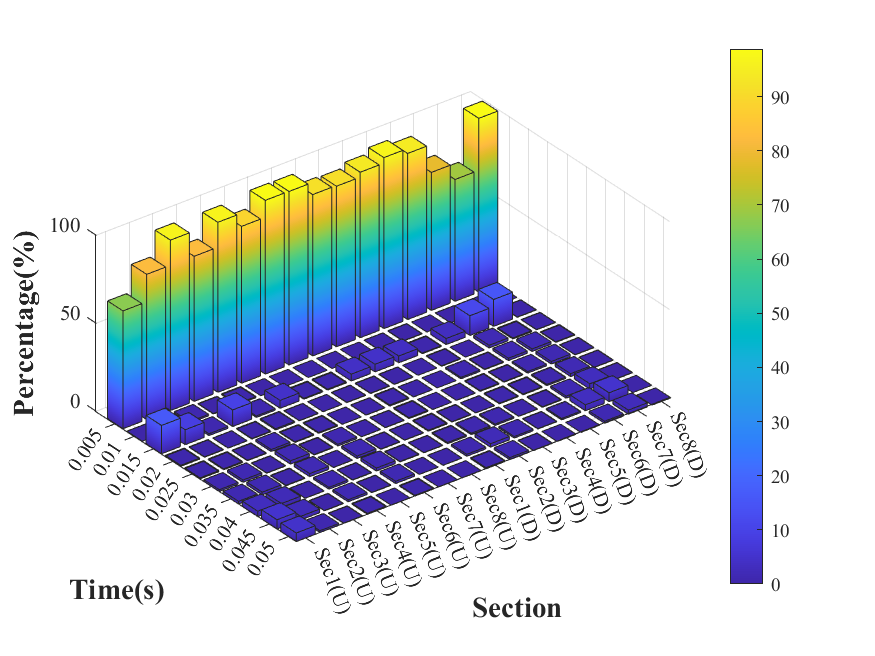}}
\centerline{(g) SSA-SAC state transition}
\end{minipage}
\hfill

\begin{minipage}{0.23\linewidth}
\centerline{\includegraphics[width=4.5cm]{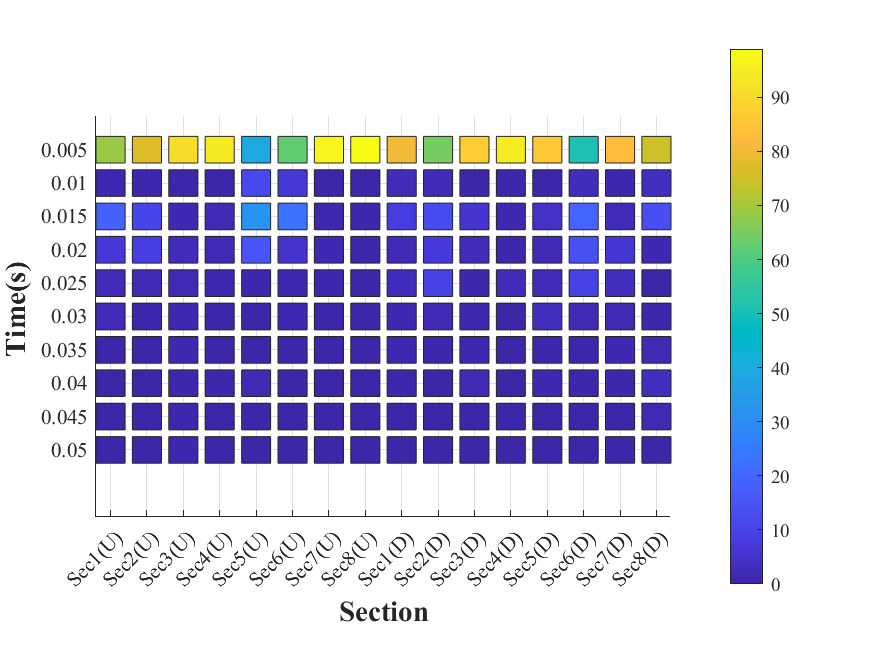}}
\centerline{(b) Top view}
\end{minipage}
\hfill
\begin{minipage}{0.23\linewidth}
\centerline{\includegraphics[width=4.5cm]{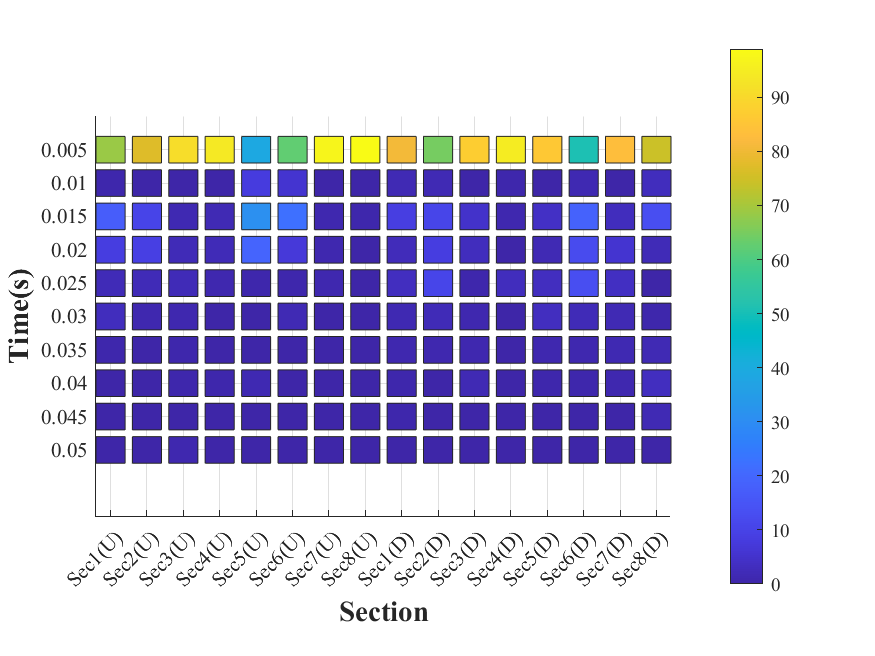}}
\centerline{(d) Top view}
\end{minipage}
\hfill
\begin{minipage}{0.23\linewidth}
\centerline{\includegraphics[width=4.5cm]{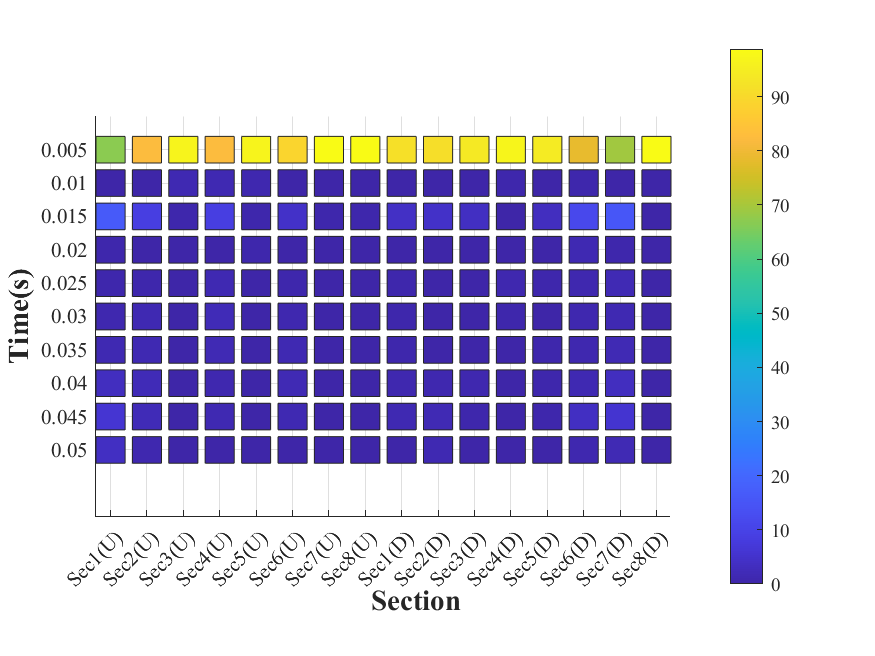}}
\centerline{(f) Top view}
\end{minipage}
\hfill
\begin{minipage}{0.23\linewidth}
\centerline{\includegraphics[width=4.5cm]{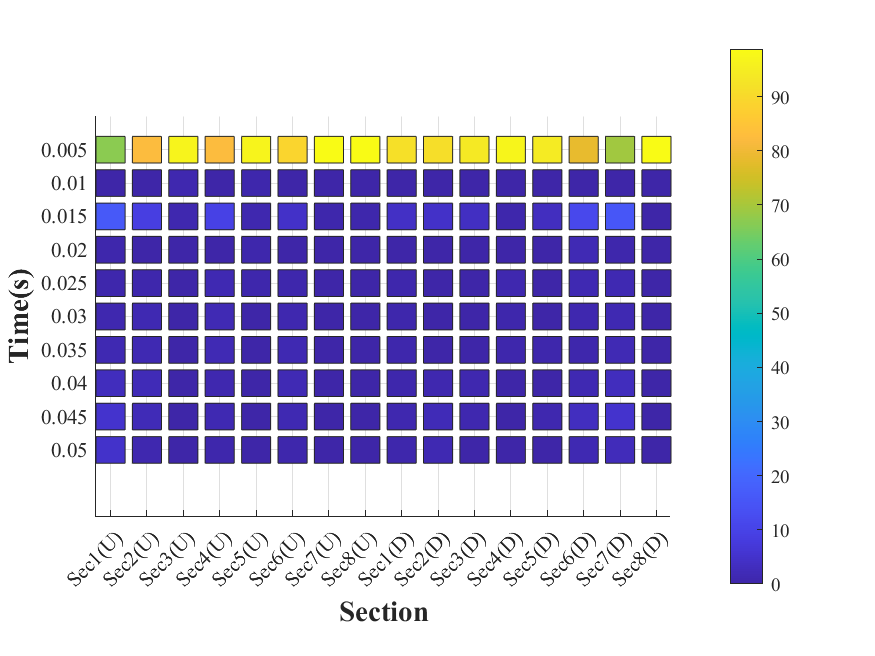}}
\centerline{(h) Top view}
\end{minipage}

\caption{Time distribution of getting action and state transition}
\label{fig:time_distribution}
\end{figure*}

\begin{table*}[htbp]
\renewcommand{\arraystretch}{1.5}
\setlength{\abovecaptionskip}{-0.2cm} 
\caption{The Results of Transferability Experient}
\label{table5}
\begin{center}
\begin{tabular}{cccccccc}
\hline
\multicolumn{1}{l}{Sec} & Direction & Actor              & Time    & Real Energy & Protect Counts & Noise Test & Maximum Step \\
\hline
\multirow{2}{*}{1}      & Up        & Section8(SSA-DDPG) & \textbf{95.20}    & \textbf{22.61}       & 48            & 102         & 146          \\
                        & Down      & Section8(SSA-SAC) & \textbf{97.05 }  & \textbf{25.67}       & 0             & 93         & 146           \\
\hline
\multirow{2}{*}{2}      & Up        & Section8(SSA-DDPG) & \textbf{199.54}  & \textbf{44.92}       & 56            & 163        & 208          \\
                        & Down      & Section8(SSA-SAC) & 391.35  & 85.48       & 0             & 158         & 208          \\
\hline
\multirow{2}{*}{3}      & Up        & Section8(SSA-DDPG) & 1146.95 & 64.46       & 38            & 143        & 165          \\
                        & Down      & Section8(SSA-SAC) & 499.53  & 144.69      & 0             & 133        & 165          \\
\hline
\multirow{2}{*}{4}      & Up        & Section8(SSA-DDPG) & \textbf{120.14}  & \textbf{32.95}       & 18            & 42         & 54           \\
                        & Down      & Section8(SSA-SAC) & 138.85  & 36.15       & 0             & 40         & 54           \\
\hline
\multirow{2}{*}{5}      & Up        & Section8(SSA-DDPG) & \textbf{124.87}  & \textbf{22.8}        & 18            & 35         & 50           \\
                        & Down      & Section8(SSA-SAC) & 134.32  & 42.61       & 0             & 34         & 50           \\
\hline
\multirow{2}{*}{6}      & Up        & Section8(SSA-DDPG) & \textbf{122.53}  & \textbf{29.41}       & 21            & 38         & 49           \\
                        & Down      & Section8(SSA-SAC) & \textbf{117.11}  & \textbf{20.98}       & 0             & 40         & 49           \\
\hline
\multirow{2}{*}{7}      & Up        & Section8(SSA-DDPG) & \textbf{151.16}  & \textbf{40.49}       & 39            & 60         & 79           \\
                        & Down      & Section8(SSA-SAC) & 292.52  & 45.55       & 0             & 66         & 79       \\
\hline
\end{tabular}
\end{center}
\vspace{-0.45cm}
\end{table*}

\subsection{Transferability Experiment}
This experiment aims to test whether the trained neural network of SSA-DRL can be deployed to a new environment. The trained SSA-DDPG and SSA-SAC algorithms in section8 of up and down direction are deployed to section1-section7 of the same direction and the results are shown in Table\ref{table5}. 

The meaning of noise test is briefly explained here. Since in this experiment the trained network is deployed in a new environment, there is a possibility that the network cannot work and the output will always be a noise. Considering the characteristic of tanh function, once the network cannot work, there may exist two special noise action sequences, all -1 or all 1, which indicates that the train will always brake or accelerate. Obviously, the all -1 action sequence cannot complete the operation plan. However, once the action sequence is all 1, it may complete the operation plan because the original action 1 always forces the train to accelerate and when the train is overspeed, the Shield and the searching tree will provide a safe action help the train to slow down. In this case, though the train completes the operation plan, it cannot be regarded as the trained network is transferability. In order to distinguish this case, the noise test is designed. In this test, we directly provide a noise action sequence with all 1 to record the protect times. Once the trained network can complete the operation plan and the protect times is far less than the noise test, the trained network is transferability.

Then from Table\ref{table5}, the SSA-DDPG is transferability in Section1,2,4,5,6,7 and the SSA-SAC is transferability in Section1,6. It is noted that this experiment is not rigorous and the experiment result can only verify the SSA-DRL may be transferability.

\begin{figure*}
\begin{minipage}{0.23\linewidth}
\centerline{\includegraphics[width=4.35cm]{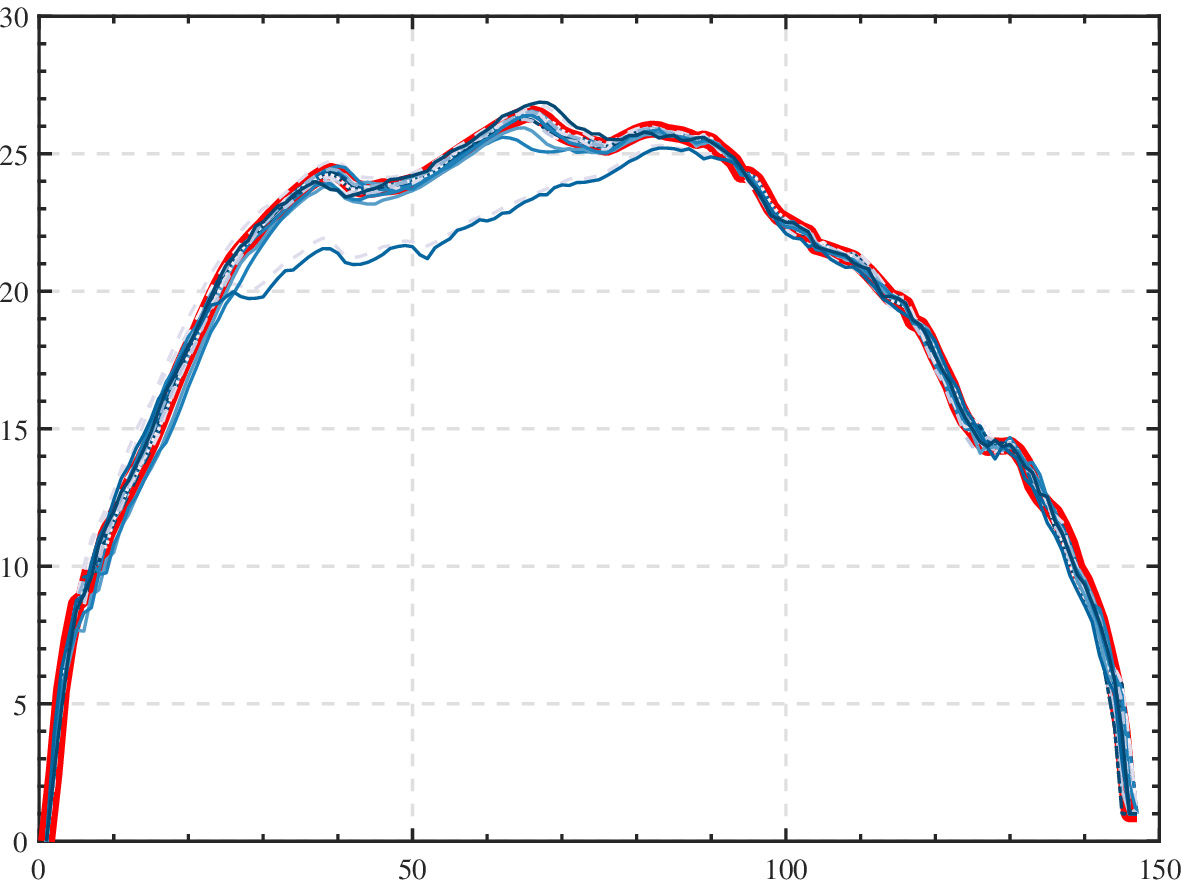}}
\centerline{(a) SSA-DDPG speed}
\end{minipage}
\hfill
\begin{minipage}{0.23\linewidth}
\centerline{\includegraphics[width=4.35cm]{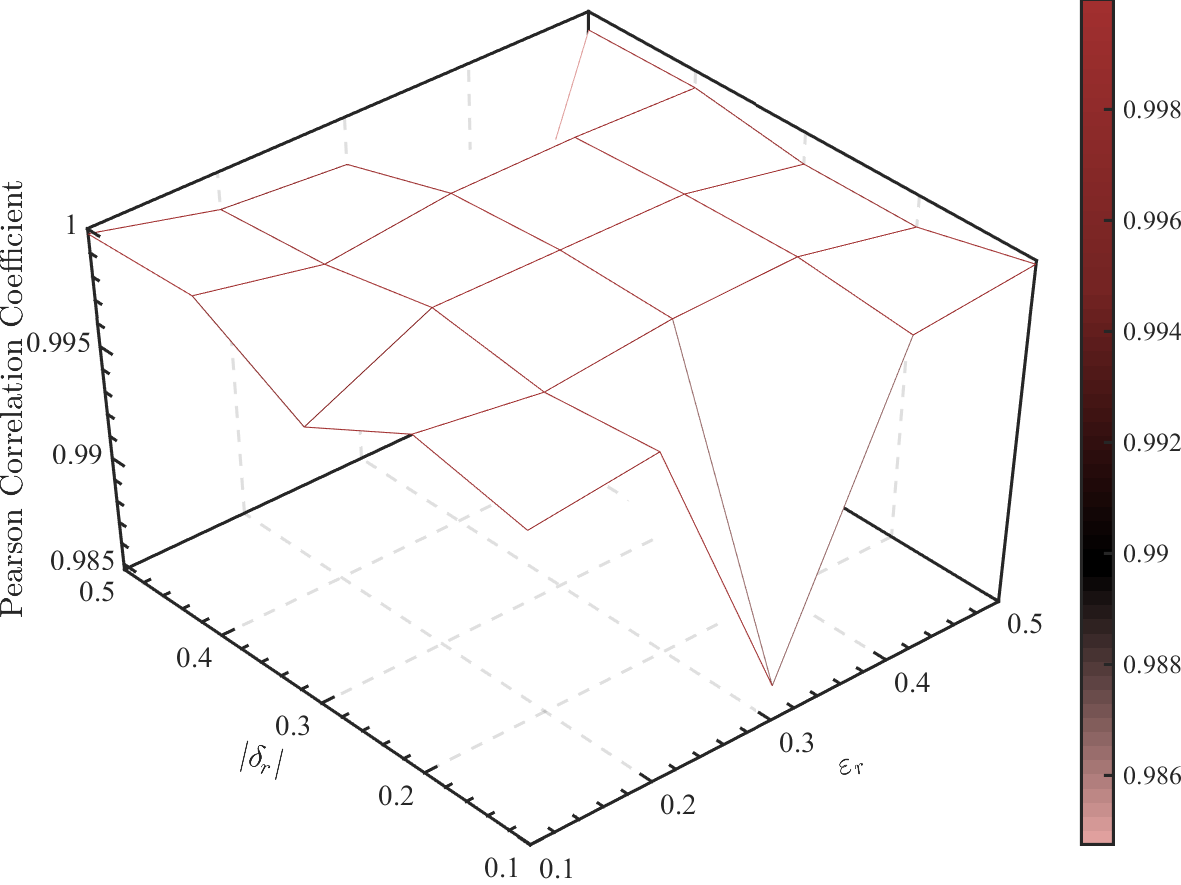}}
\centerline{(b) PCCs of SSA-DDPG speed}
\end{minipage}
\hfill
\begin{minipage}{0.23\linewidth}
\centerline{\includegraphics[width=4.35cm]{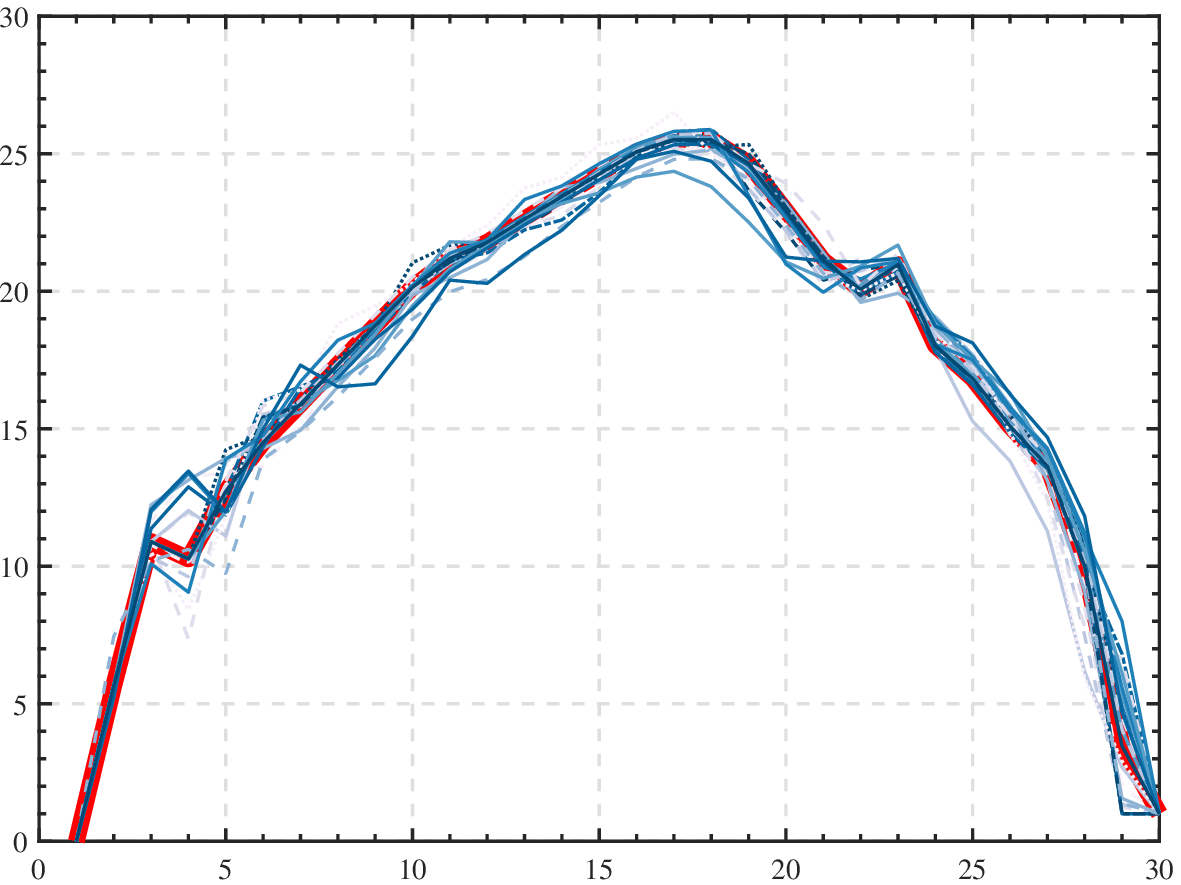}}
\centerline{(c) SSA-SAC speed}
\end{minipage}
\hfill
\begin{minipage}{0.23\linewidth}
\centerline{\includegraphics[width=4.35cm]{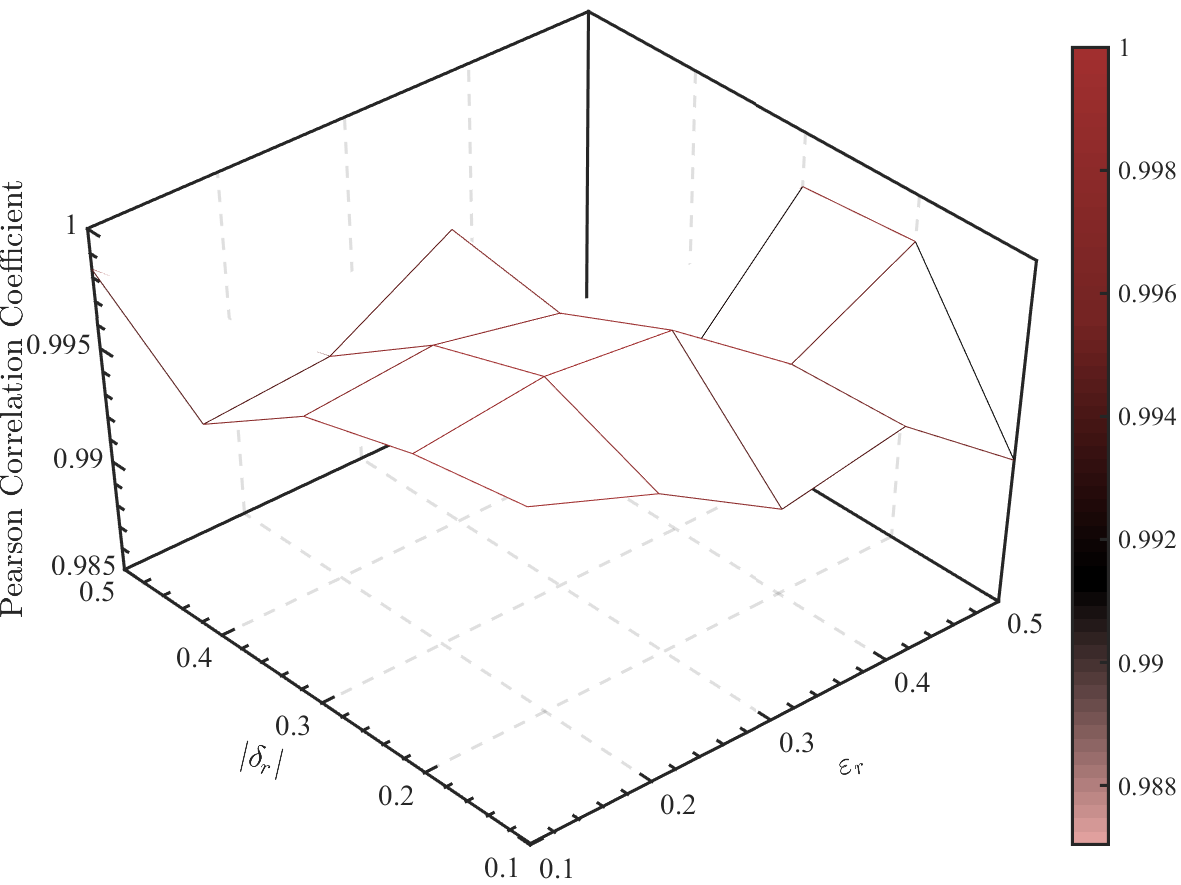}}
\centerline{(d) PCCs of SSA-SAC speed}
\end{minipage}
\hfill
\vspace{1.3cm} 
\begin{minipage}{0.23\linewidth}
\centerline{\includegraphics[width=4.35cm]{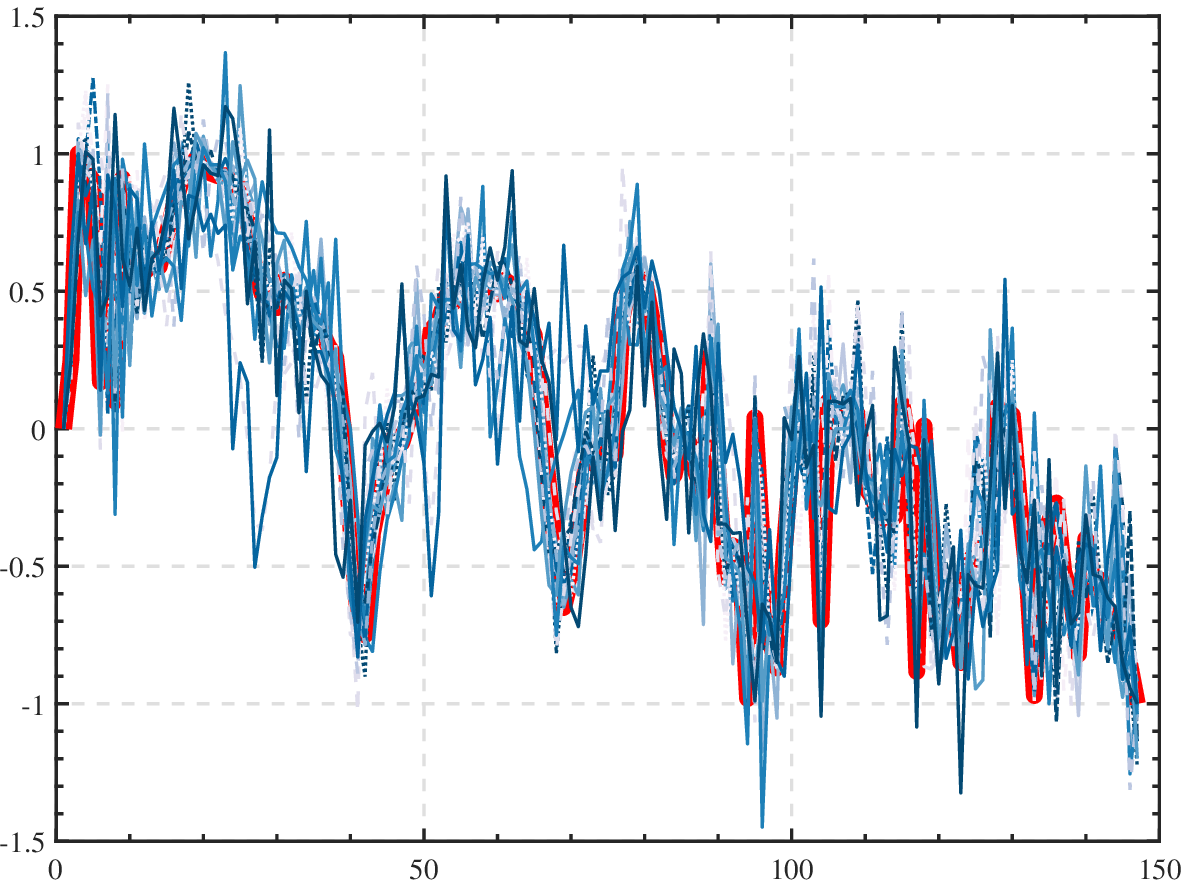}}
\centerline{(e) SSA-DDPG action}
\end{minipage}
\hfill
\begin{minipage}{0.23\linewidth}
\centerline{\includegraphics[width=4.35cm]{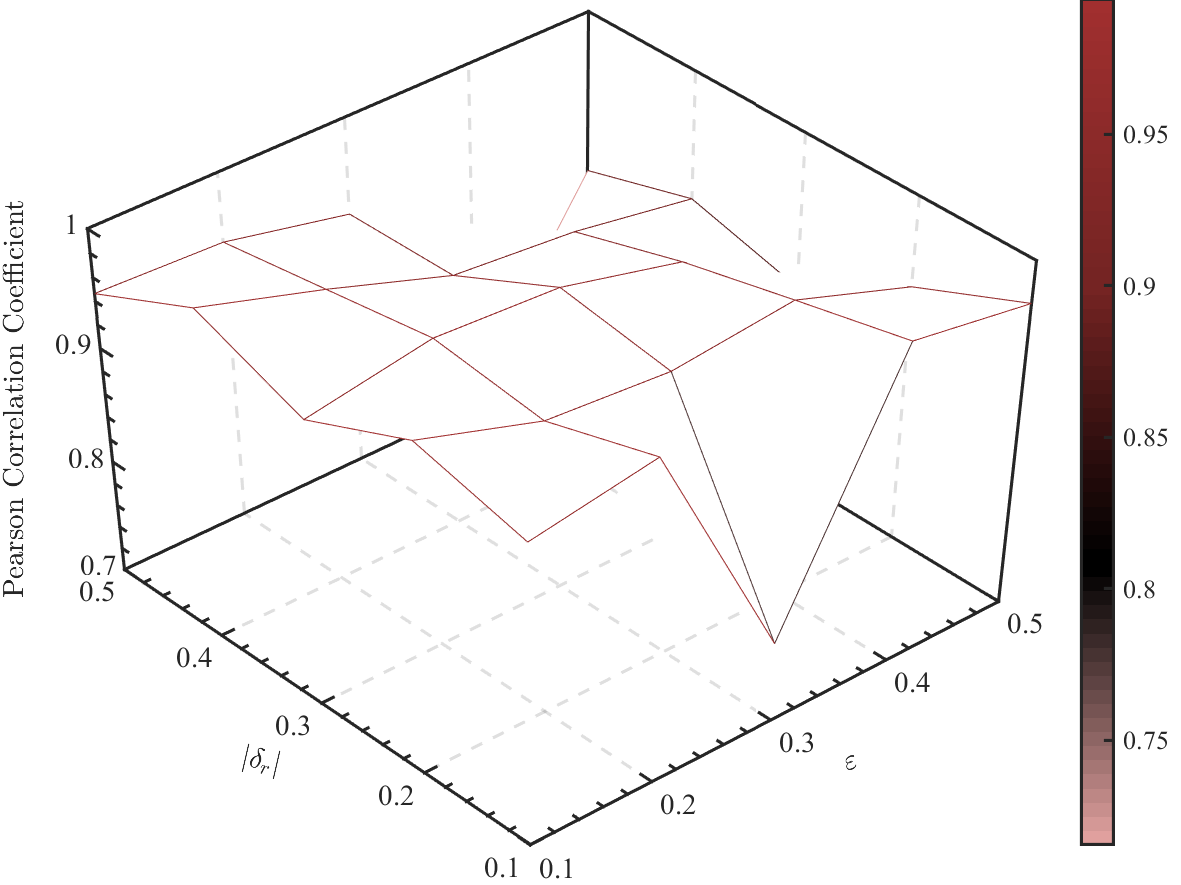}}
\centerline{(f) PCCs of SSA-DDPG action}
\end{minipage}
\hfill
\begin{minipage}{0.23\linewidth}
\centerline{\includegraphics[width=4.35cm]{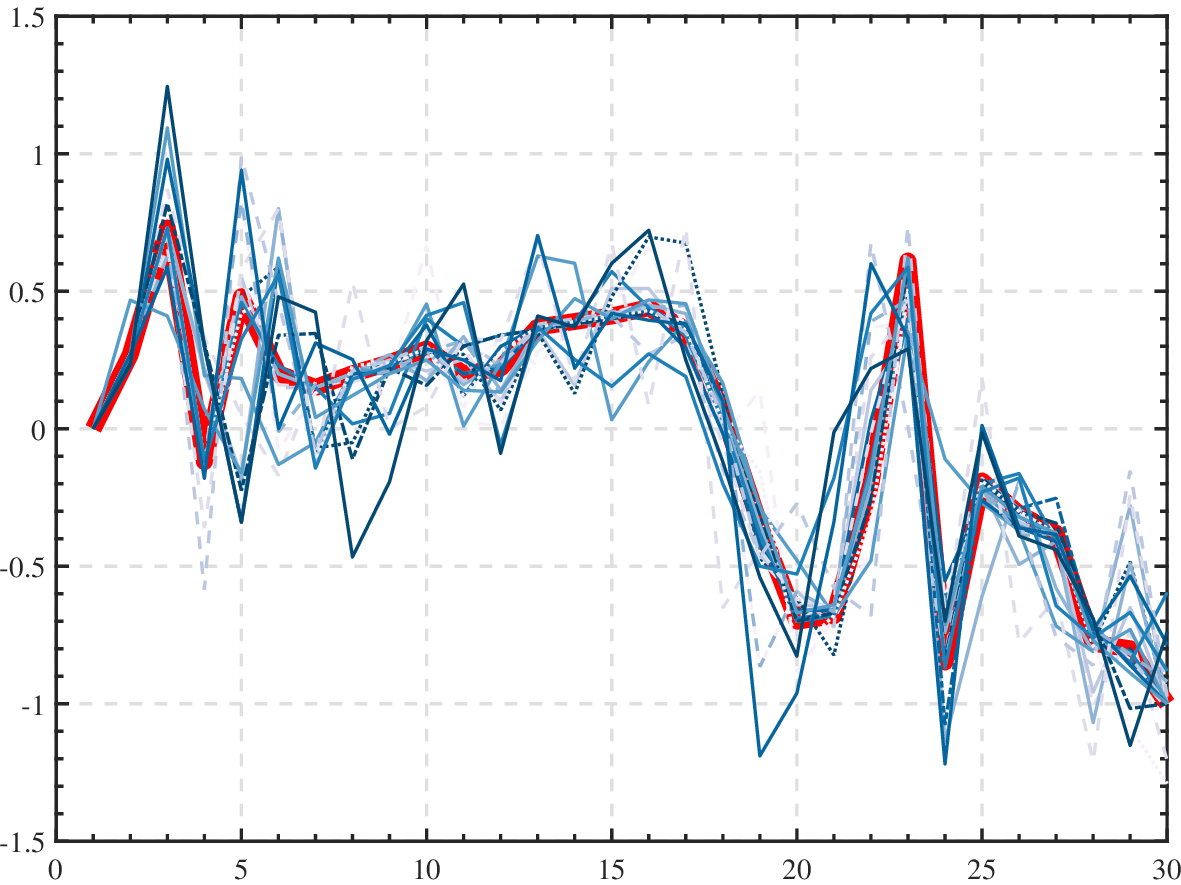}}
\centerline{(g) SSA-SAC action}
\end{minipage}
\hfill
\begin{minipage}{0.23\linewidth}
\centerline{\includegraphics[width=4.35cm]{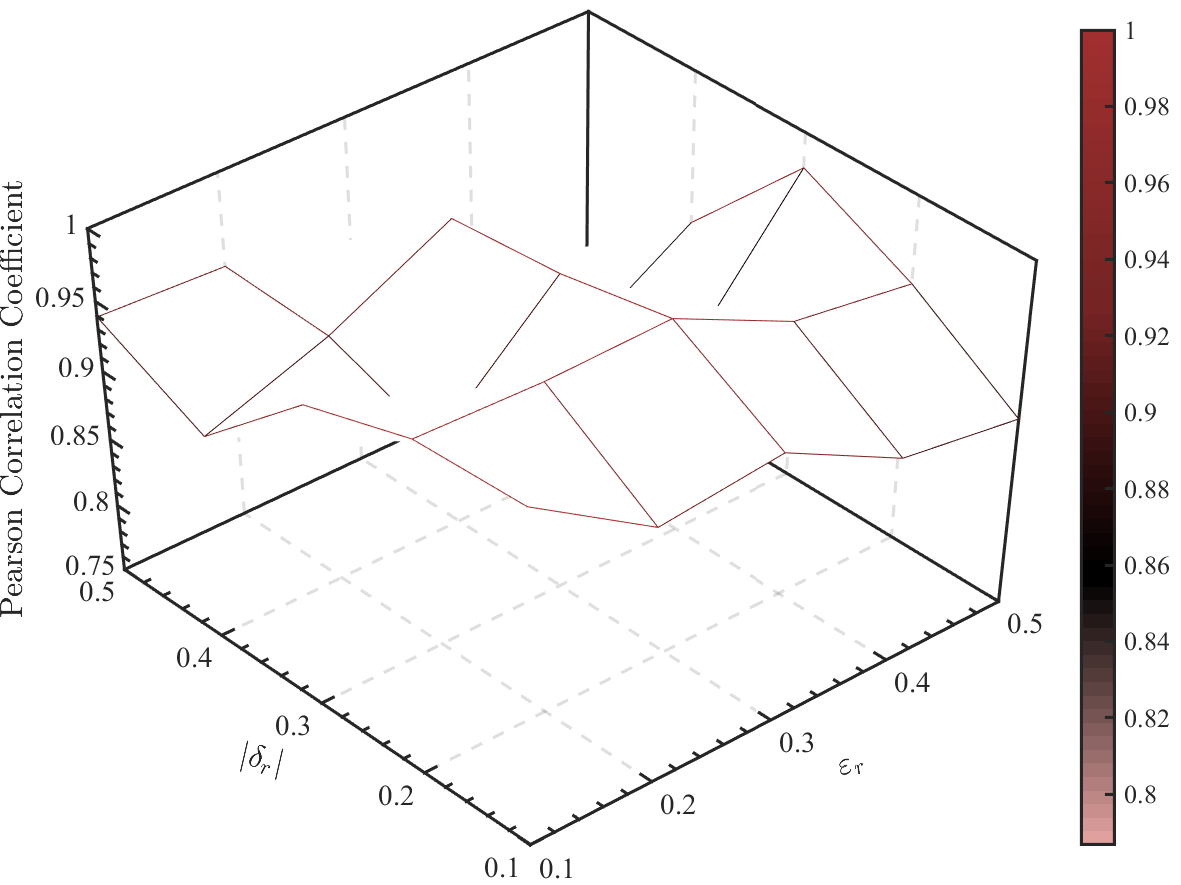}}
\centerline{(h) PCCs of SSA-SAC action}
\end{minipage}
\hfill
\vspace{1.3cm} 
\begin{minipage}{0.23\linewidth}
\centerline{\includegraphics[width=4.35cm]{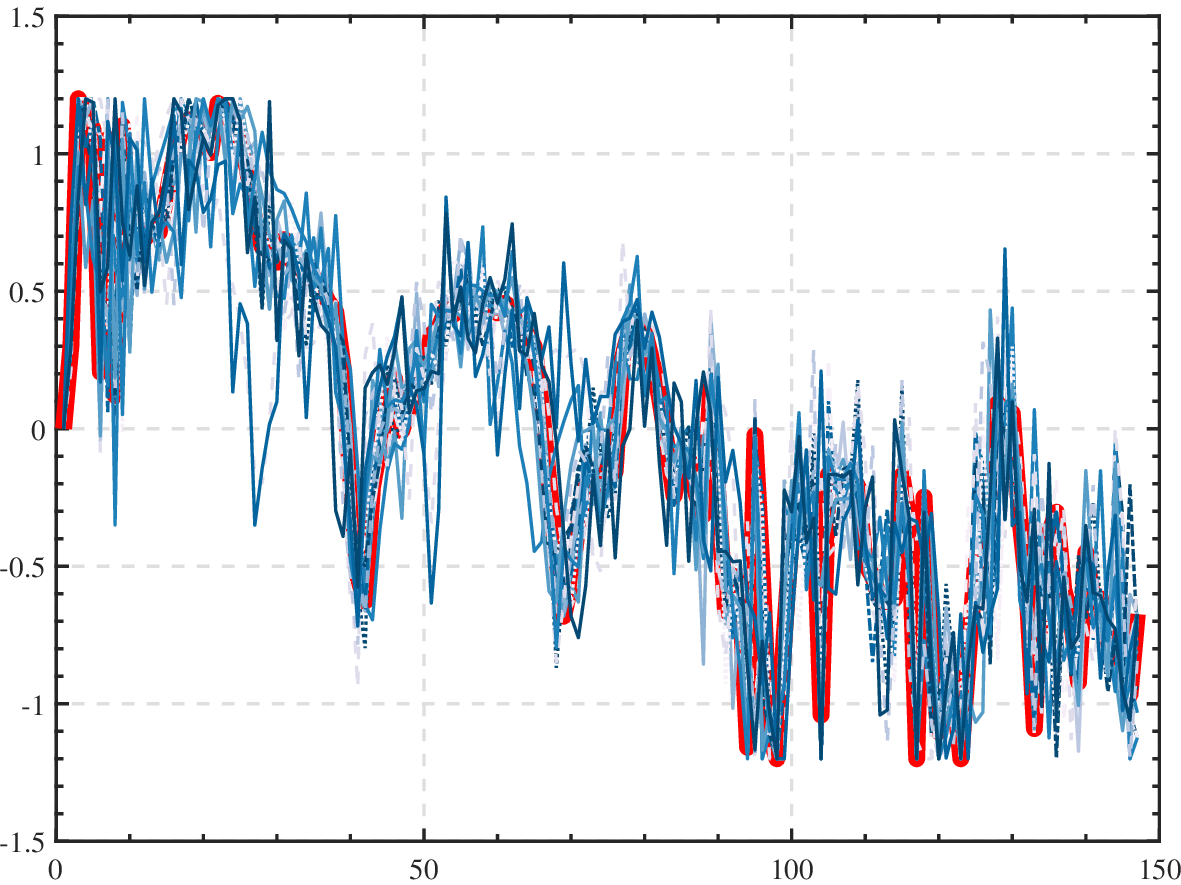}}
\centerline{(i) SSA-DDPG acc}
\end{minipage}
\hfill
\begin{minipage}{0.23\linewidth}
\centerline{\includegraphics[width=4.35cm]{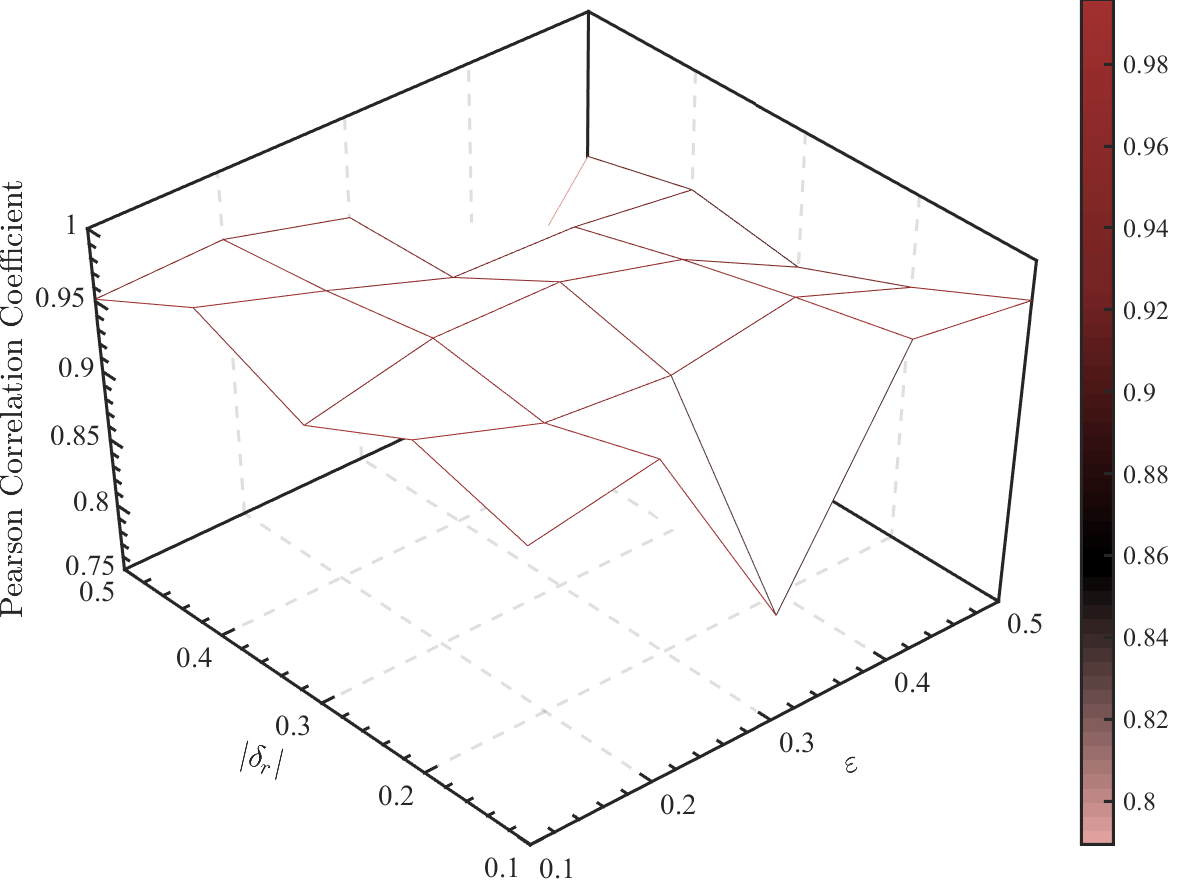}}
\centerline{(j) PCCs of SSA-DDPG acc}
\end{minipage}
\hfill
\begin{minipage}{0.23\linewidth}
\centerline{\includegraphics[width=4.35cm]{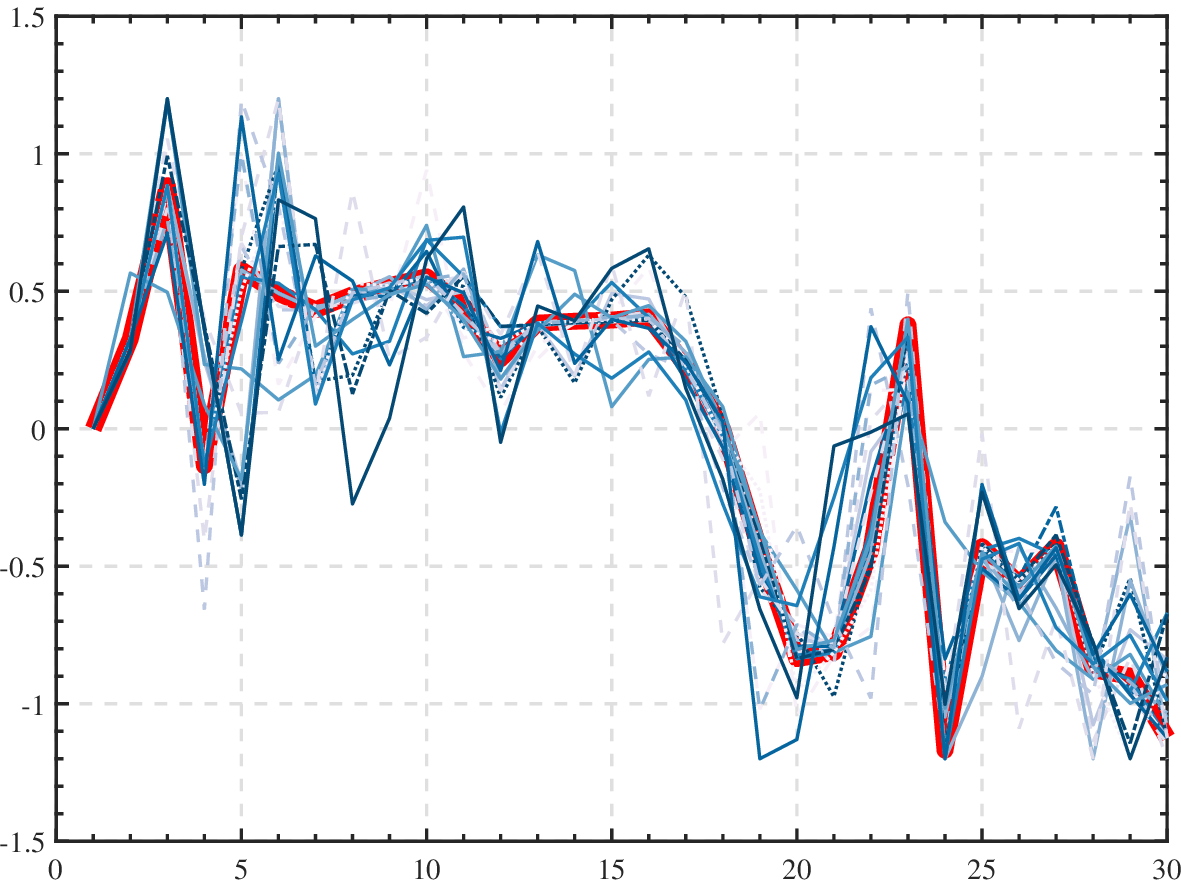}}
\centerline{(k) SSA-SAC acc}
\end{minipage}
\hfill
\begin{minipage}{0.23\linewidth}
\centerline{\includegraphics[width=4.35cm]{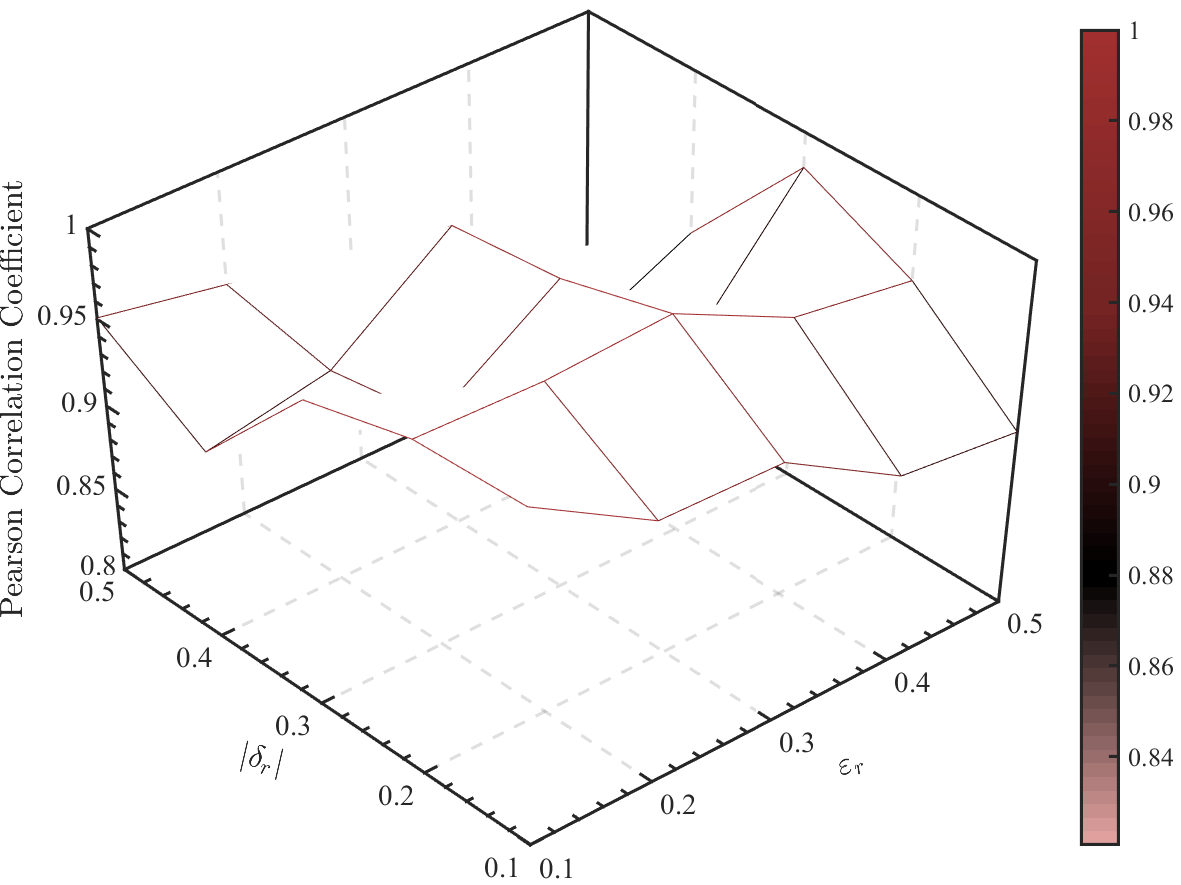}}
\centerline{(l) PCCs of SSA-SAC acc}
\end{minipage}
\hfill
\caption{Results of robustness experiments.}
\label{fig:robustnessv}
\end{figure*}

\begin{figure}
\begin{minipage}{0.95\linewidth}
\centerline{\includegraphics[width=7.5cm]{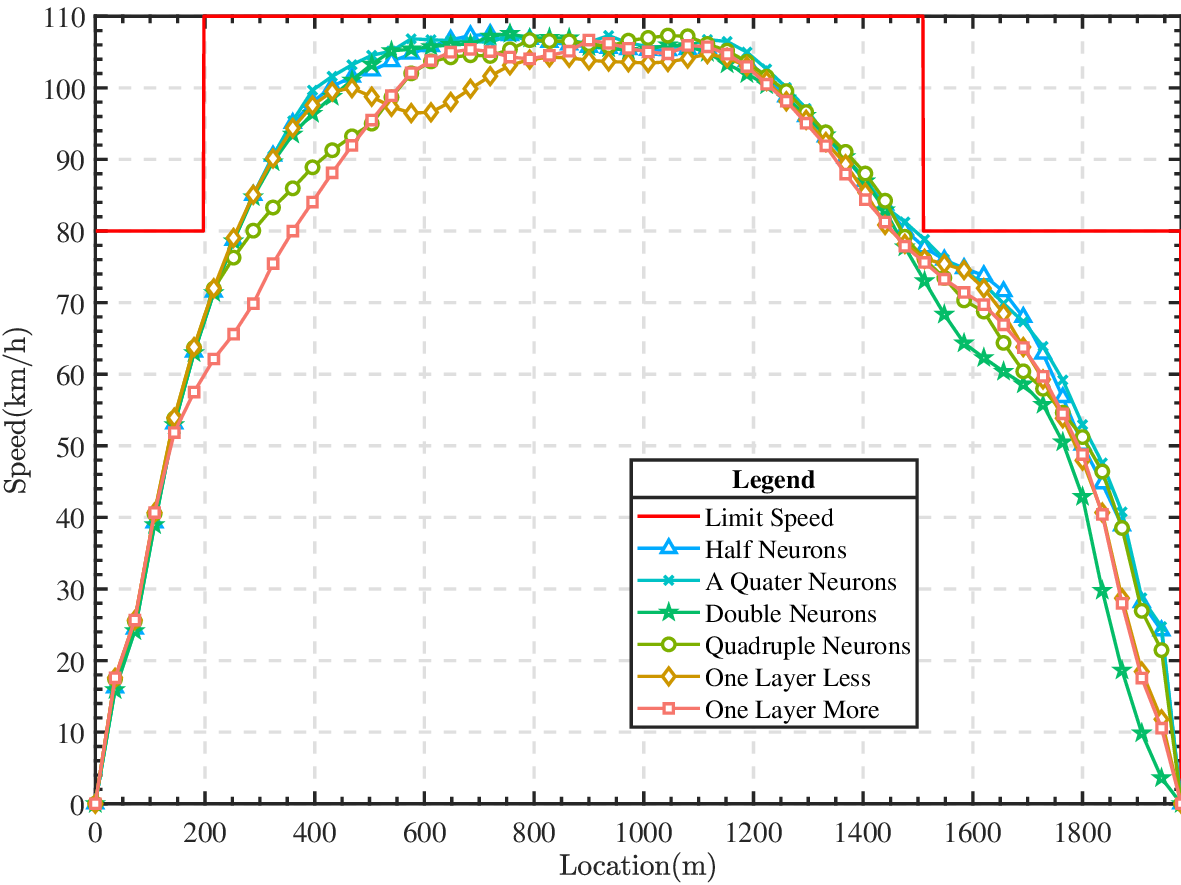}}
\centerline{(a) Up Direction (Section4)}
\end{minipage}
\hfill
\begin{minipage}{0.95\linewidth}
\centerline{\includegraphics[width=7.5cm]{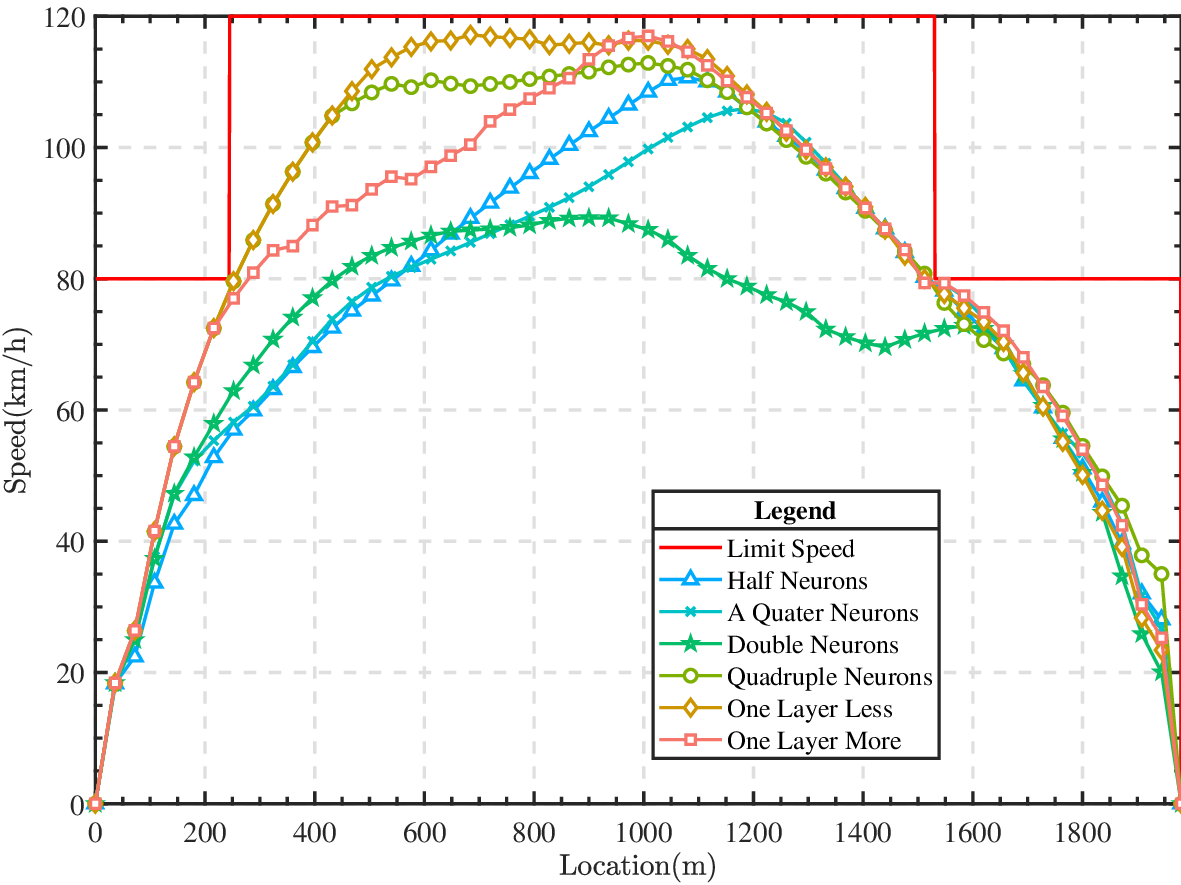}}
\centerline{(b) Down Direction (Section4)}
\end{minipage}
\hfill
\caption{Speed profiles of different structures of additional learner.}
\label{fig:Adexp}
\end{figure}

\subsection{Robustness Experiment}
This experiment aims to verify the robustness of the SSA-DRL, that is, the ability to complete the operation plan under the condition that the action is disturbed. Two parameters $\varepsilon_r,\delta_r$ are introduced in this experiment to control the probability and magnitude of action disturbance. For an action $a_r$ given to the agent, it has $\varepsilon_r\times 100\%$ probability to change to another action $a_r+\mathcal{N}_r,\mathcal{N}_r\in \left[-\delta_r,\delta_r\right]$. The range of $\varepsilon_r,|\delta_r|$ are both $\left[0.1,0.5\right]$ and the changing stepsize is 0.1. The Pearson correlation coefficient (PCC) is used here to measure the degree of correlation between the original sequence and the disturbed sequence. SSA-DDPG  and SSA-SAC are used in Section1 up and down direction respectively. Fig.~\ref{fig:robustnessv} shows the original and disturbed curves of speed profile, action sequence and acceleration sequence and the changing trend of PCC when $\varepsilon_r$ and $\delta_r$ change. In Fig.~\ref{fig:robustnessv}, the original curves are red bold and the disturbed curves are blue dash. It is clear that the disturbed curves are of the same trend with the original curve and there is no completely different curve or one curve is totally changed after a disturbance. The PCCs of speed profile, action sequences and acceleration sequences are all larger than 0.985, 0.7 and 0.75, since the PCC is more close to 1, the two curves are more linear correlation, thus combined with Fig.~\ref{fig:robustnessv}, it can be concluded that the SSA-DRL may have a strong robustness in some scenarios.

\subsection{Design of the Additional Learner}
As mentioned above, the structure of the additional learner can be different from the actor in the used DRL algorithm. In this experiment, six different neural network structures are used to verify that the design of additional learner will not influence the final strategy significantly. The six different structures are half neuron numbers, a quarter neuron numbers, double neuron numbers, quadruple neuron numbers, one hidden layer less (neuron numbers are not changed) and one hidden layer more respectively. In this experiment, the training process is the same with the basic simulation. Also, the active function and the connection mode of neural networks are not changed. SSA-DDPG and SSA-SAC are used in section4 up and down directions respectively and the speed profiles in one execution process are shown in Fig.~\ref{fig:Adexp}.

It can be clearly seen that only three speed profiles (double neurons, a quarter neurons and double neurons) of SSA-SAC have big differences with other speed profiles. In this case, though the structures of the additional learner and the actor in DRL are not the same, it will not have a big influence of the final execution process which makes the users more easily to train an easy to deploy neural network.

\section{Conclusion}
Aiming at the safe control strategy for urban rail transit autonomous operation, an SRL framework called SSA-DRL is proposed in this paper. The SSA-DRL uses a LTL based post-posed Shield to check whether an original action is safe and then uses a searching tree to find a safe action with the highest long term reward to correct the unsafe action. An additional learner consists of a replay buffer and an additional actor is also added to help to reduce the protect times in execution process.

Our framework is verified in simulations under four different aspects with two basic DRL algorithms. The basic experiment shows that the framework can control the train complete the operation plan with a lower energy consumption and protect times. And compared with the basic DRL or Shield-DRL algorithms, the SSA-DRL can get a higher reward and achieve convergence earlier in most simulation scenarios. The transferability and robustness experiment verify that a trained network can transfer to a new environment and can still complete the operation plan under some disturbances. The experiment of the design of the additional learner verifies that SSA-DRL can help users to train a easy to deploy neural network without big loss of performance.

Finally, from our research and other related researches  we summarize the difficulties of design a self-decision algorithm for autonomous train operation into 2\textbf{W}2\textbf{H} as follow.
\begin{itemize}
    \item \textbf{W}hat algorithms can be used to design a self-decision algorithm meanwhile optimize the objective functions?
    \item \textbf{W}hat methods can be used to build a fail-safe mechanism if the designed algorithm is based on AI methods?
    \item \textbf{H}ow do we ensure the designed algorithms always have a safe output which will not lead to an overspeed operation?
    \item \textbf{H}ow do we explain the safe output of the designed algorithm or \textbf{H}ow do we explain why the output is safe?
\end{itemize}

In the future work, we will try to find a more general framework to systematically answer the 2\textbf{W}2\textbf{H} problems.




\bibliographystyle{IEEEtran}
\bibliography{IEEEabrv,main.bib}

\end{document}